\documentclass[12pt,oneside,maitrise]{dms_arxiv}

\usepackage[utf8]{inputenc} 
\usepackage[T1]{fontenc}    %

\usepackage{lmodern}

\anglais
\usepackage{natbib}
\usepackage[english,frenchb]{babel}

\def\sloppy{%
  \tolerance 500
  \emergencystretch 3em%
  \hfuzz .5pt
  \vfuzz\hfuzz}
\usepackage{graphicx,amssymb,subfigure,icomma}

\usepackage{float}
\usepackage{enumitem}
\usepackage{bbm}
\usepackage[title,page]{appendix}
\usepackage{color}
\definecolor{mydarkblue}{rgb}{0,0.08,0.45}

\usepackage[labelfont=bf, width=\linewidth]{caption}

\usepackage{hyperref}  
\hypersetup{colorlinks=true,allcolors=black}
\usepackage{hypcap}   
\usepackage{bookmark} 

\usepackage[capitalize,nameinlink,noabbrev]{cleveref}
\crefname{appsec}{Appendix}{Appendixces}

\numberwithin{figure}{chapter}
\makeatletter
\renewcommand\p@subfigure{\thefigure}
\makeatother

\newcommand{\Prob}{\mathbb{P}}
\newcommand{\E}{\mathbb{E}}
\newcommand{\reals}{\mathbb{R}}
\newcommand{\X}{\mathcal{X}}    
\newcommand{\Y}{\mathcal{Y}}    
\newcommand{\D}{\mathcal{D}}    
\newcommand{\A}{\mathcal{A}}    
\newcommand{\HH}{\mathcal{H}}    
\newcommand{\cR}{\mathcal{R}}
\newcommand{\R}{\mathbb{R}}

\newcommand{\risk}{\mathrm{R}}
\newcommand{\emprisk}{\hat{\risk}}
\newcommand{\gap}{\mathcal{E}_{\text{gap}}}
\newcommand{\errapp}{\mathcal{E}_\text{app}}
\newcommand{\errest}{\mathcal{E}_\text{est}}

\newcommand{\VC}{\mathrm{VC}}
\newcommand{\Rad}{\mathfrak{R}_m}

\newcommand{\Var}{\mathrm{Var}}
\newcommand{\var}{\mathrm{Var}}

\newcommand{\Bias}{\mathrm{Bias}}
\newcommand{\Tr}{\text{Tr}}

\newcommand{\ebias}{\mathcal{E}_\text{bias}}
\newcommand{\evar}{\mathcal{E}_\text{variance}}
\newcommand{\enoise}{\mathcal{E}_\text{noise}}

\DeclareMathOperator{\rank}{rank}
\DeclareMathOperator{\rowspace}{rowspace}
\DeclareMathOperator{\nullspace}{nullspace}

\DeclareMathOperator*{\argmin}{arg\,min}

\newtheorem{assumption}{Assumption}
\newtheorem{definition}{Definition}

\newcommand{\beq}{\begin{equation}}
\newcommand{\eeq}{\end{equation}}
\newcommand{\be}{\begin{equation}}
\newcommand{\ee}{\end{equation}}
\newcommand{\beqa}{\begin{eqnarray}}
\newcommand{\eeqa}{\end{eqnarray}}
\newcommand{\bean}{\begin{eqnarray*}}
\newcommand{\eean}{\end{eqnarray*}}

\newtheorem{cor}{\corollaryname}[section]

\newtheorem{lemma}[cor]{\lemmaname}
\newtheorem{prop}[cor]{Proposition}

\newtheorem{theorem}[cor]{\theoremname}
\theoremstyle{definition}

\numberwithin{equation}{section}
\numberwithin{table}{chapter}
\numberwithin{figure}{chapter}

\usepackage[onehalfspacing]{setspace}

\begin{document}

\version{1}



\title{On the Bias-Variance Tradeoff:\\Textbooks Need an Update}

\titletwo{On the Bias-Variance Tradeoff: Textbooks Need an Update}

\author{Brady Neal}

\copyrightyear{2019}

\department{Department of Computer Science and Operations Research}
\faculty{Faculty of Arts and Sciences}
\degree{Master of Science (M.Sc.)}
\sujet{Computer Science}
\facultyto{Faculty of Graduate and Postdoctoral Studies}

\date{10 December 2019} 


\president{Aaron Courville}

\directeur{Ioannis Mitliagkas}


\membrejury{Gilles Brassard}






\pagenumbering{roman}

\maketitle

\maketitle


\anglais
\chapter*{Abstract}

The main goal of this thesis is to point out that the bias-variance tradeoff is not always true (e.g.\ in neural networks). We advocate for this lack of universality to be acknowledged in textbooks and taught in introductory courses that cover the tradeoff.

We first review the history of the bias-variance tradeoff, its prevalence in textbooks, and some of the main claims made about the bias-variance tradeoff.
Through extensive experiments and analysis, we show a lack of a bias-variance tradeoff in neural networks when increasing network width.
Our findings seem to contradict the claims of the landmark work by \citet{geman}.
Motivated by this contradiction, we revisit the experimental measurements in \citet{geman}.
We discuss that there was never strong evidence for a tradeoff in neural networks when varying the number of parameters.
We observe a similar phenomenon beyond supervised learning, with a set of deep reinforcement learning experiments.

We argue that textbook and lecture revisions are in order to convey this nuanced modern understanding of the bias-variance tradeoff.

\noindent \textbf{Keywords:} bias-variance tradeoff, neural networks, over-parameterization, generalization


\francais

\chapter*{Résumé}

L'objectif principal de cette thèse est de souligner que le compromis biais-variance n'est pas toujours vrai (p.~ex.~dans les réseaux neuronaux). Nous plaidons pour que ce manque d'universalité soit reconnu dans les manuels scolaires et enseigné dans les cours d'introduction qui couvrent le compromis.

Nous passons d'abord en revue l'historique du compromis entre les biais et les variances, sa prévalence dans les manuels scolaires et certaines des principales affirmations faites au sujet du compromis entre les biais et les variances.
Au moyen d'expériences et d'analyses approfondies, nous montrons qu'il n'y a pas de compromis entre la variance et le biais dans les réseaux de neurones lorsque la largeur du réseau augmente.
Nos conclusions semblent contredire les affirmations de l'œuvre historique de \citet{geman}.
Motivés par cette contradiction, nous revisitons les mesures expérimentales dans \citet{geman}.
Nous discutons du fait qu'il n'y a jamais eu de preuves solides d'un compromis dans les réseaux neuronaux lorsque le nombre de paramètres variait.
Nous observons un phénomène similaire au-delà de l'apprentissage supervisé, avec un ensemble d'expériences d'apprentissage de renforcement profond.

Nous soutenons que les révisions des manuels et des cours magistraux ont pour but de transmettre cette compréhension moderne nuancée de l'arbitrage entre les biais et les variances.

\noindent \textbf{Mots clés:} compromis biais-variance, réseaux de neurones, sur-paramétrage, généralisation


\anglais
\cleardoublepage
\pdfbookmark[chapter]{\contentsname}{toc}  
\tableofcontents
\cleardoublepage
\phantomsection  
\cleardoublepage
\phantomsection
\listoffigures


\chapter*{List of Acronyms and Abbreviations}
\begin{twocolumnlist}{.2\textwidth}{.7\textwidth}
    CIFAR10 & dataset from the Canadian Institute For Advanced Research \\
    KNN & K-Nearest Neighbors \\
    LBFGS & Limited-memory Broyden–Fletcher–Goldfarb–Shanno algorithm \\
    MNIST & Modified National Institute of Standards and Technology (dataset) \\
    SGD & Stochastic Gradient Descent \\
    SVHN & Street View House Numbers (dataset) \\
    VC dimension & Vapnik–Chervonenkis dimension \\
\end{twocolumnlist}


\hypersetup{colorlinks=true,allcolors=mydarkblue}

\chapter*{Acknowledgements}

I would like to thank my advisor, Ioannis Mitliagkas, for taking me on as his student, despite the apparent risk that came with that. I greatly appreciate how supportive he has been of me. He has been a fantastic advisor. I would like to thank Yoshua Bengio and Ioannis Mitliagkas for supporting my admission to the department. Without them, I would guess the university would not have accepted me until I completed the final year of my Bachelors degree. 

There are many other people at Mila who have been fantastic to interact with. I would like to thank all of the students I worked with, discussed with, and hung out with. I would like to thank Céline Bégin for greatly helping me navigate all the process-related items at a francophone university.

I would like to thank my girlfriend, Isabelle, who has had an immensely positive impact on me throughout my degree.

 %
 %

\NoChapterPageNumber
\cleardoublepage
\pagenumbering{arabic}


\chapter{Introduction}

\section{Motivation}

An important dogma in machine learning has been that ``the price to pay for achieving low bias is high variance'' \citep{geman}.
This is overwhelmingly the intuition among machine learning practitioners, despite some notable exceptions such as boosting \citep{Schapire1999,buhlmann2003boosting}.
The quantities of interest here are the bias and variance of a learned model's {\em prediction} on an unseen input, where the randomness comes from the sampling of the training data (see \cref{sec:bias-variance} for more detail). The basic idea is that too simple a model will underfit (high bias) while too complex a model will overfit (high variance) and that bias and variance trade off as model complexity is varied. This is commonly known as the \emph{bias-variance tradeoff} (\cref{fig:intro-bv-tradeoff-fortmann} and \cref{sec:bias-variance}).

A key consequence of the bias-variance tradeoff is that it implies that test error will be a U-shaped curve in model complexity (\cref{fig:intro-bv-tradeoff-fortmann}). Statistical learning theory \citep{vapnik1998statistical} also predicts a U-shaped test error curve for a number of classic machine learning models by identifying a notion of model capacity, understood as the main parameter controlling this tradeoff. However, there is a growing amount of empirical evidence that \textit{wider} networks generalize \textit{better} than their smaller counterparts \citep{DBLP:journals/corr/NeyshaburTS14,wide_resnet,novak2018sensitivity, lee2018deep,belkin2018, jamming,fisher-rao_metric,DBLP:journals/corr/CanzianiPC16}. In those cases no U-shaped test error curve is observed. In \cref{fig:neyshabur}, we depict \citet{DBLP:journals/corr/NeyshaburTS14}'s example of this phenomenon.

The lack of a U-shaped test error curve in these prominent cases suggests that there may be something wrong with the bias-variance tradeoff. In this work, we seek to understand if there really is a bias-variance tradeoff in neural networks when varying the network width by explicitly measuring bias and variance. In their landmark work that highlighted the bias-variance tradeoff in neural networks, \citet{geman} claim that bias decreases and variance increases with network size. This is one of the main claims we refute.


\begin{figure}[t]
    \centering
    \subfigure[The bias-variance tradeoff predicts a U-shaped test error curve \citep{fortmann-roe_2012}.]{
        \includegraphics[width=.5\textwidth]{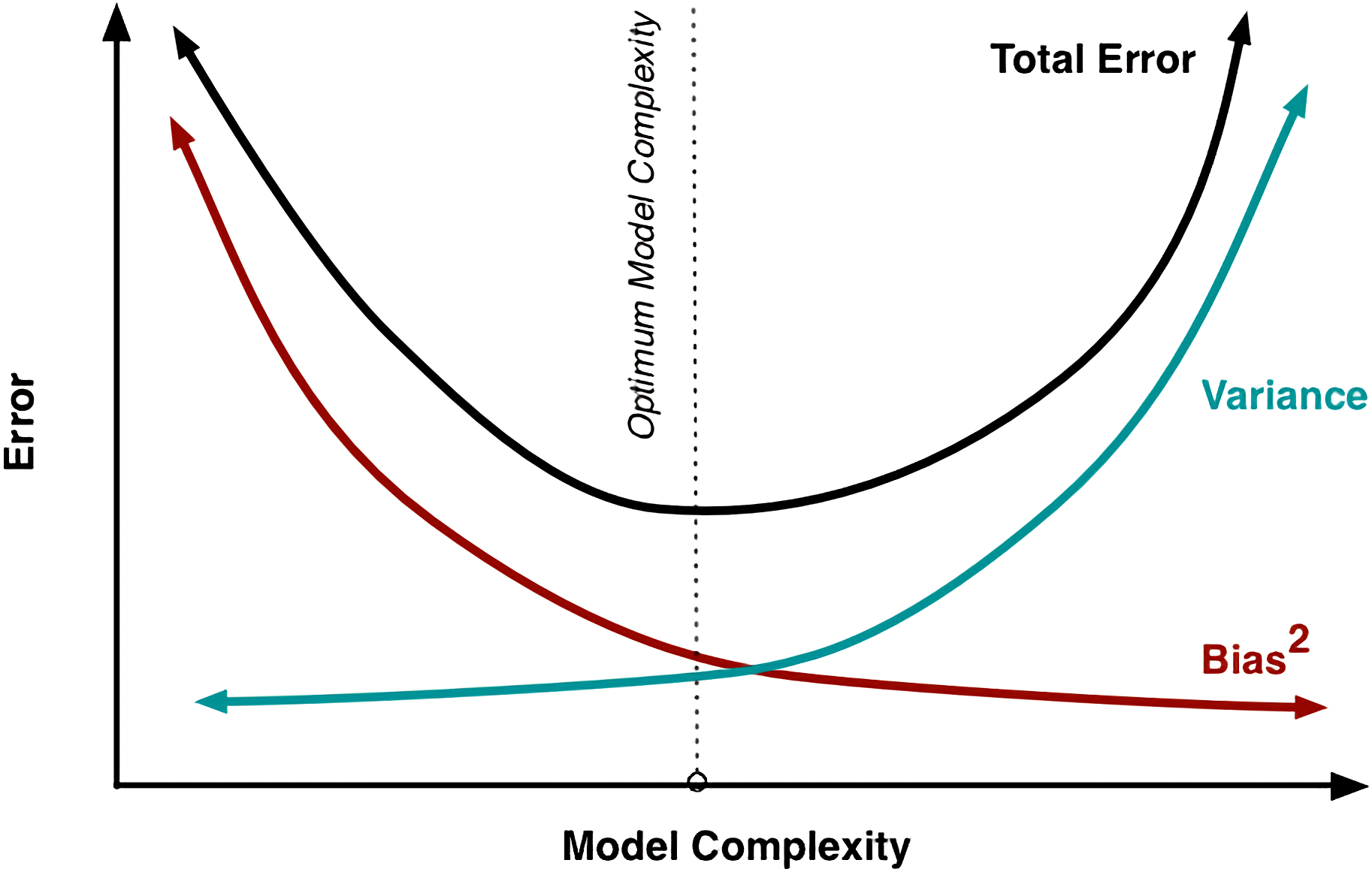}
        \label{fig:intro-bv-tradeoff-fortmann}
    }
    \hfill
    \subfigure[\citet{DBLP:journals/corr/NeyshaburTS14} found that test error actually decreases with neural network width.]{
        \includegraphics[width=.45\textwidth]{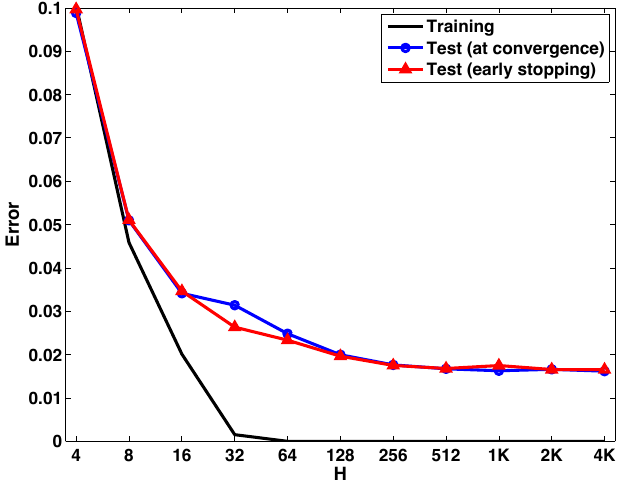}
        \label{fig:neyshabur}
    }
    \caption{Mismatch between test error predicted by bias-variance tradeoff and reality}
    \label{fig:geman2}
\end{figure}

\section{Objective of this Thesis}

The main objective of this thesis is to show that bias-variance tradeoff thinking can be wrong; researchers and practitioners who assume it to always be true may make incorrect predictions related to model selection. Therefore, we recommend that textbooks and machine learning courses are updated to not present the bias-variance tradeoff as universally true (though, it is accurate for some models, which we review in \cref{sec:bv-experimental-evidence}). Similarly, we recommend researchers and practitioners update to not universally assume the bias-variance tradeoff (see \cref{sec:need-to-qualify-claims}). 

Throughout this thesis, we will reference many textbooks, often using their figures and quoting them. This is to help illustrate what is taught in introductory machine learning courses and to ensure that we are not arguing against strawmen.

\section{Novel Contributions}

\begin{enumerate}

    \item We revisit the bias-variance analysis in the modern setting for neural networks and point out that it is not necessarily a tradeoff as \emph{both} bias \emph{and} variance decrease with network width, yielding better generalization (\cref{sec:width}).

    \item
    We perform a more fine-grain study of variance in neural networks by decomposing it into variance due to initialization and variance due to sampling. Variance due to initialization is significant in the under-parameterized regime and monotonically decreases with width in the over-parameterized regime. There, total variance is much lower and dominated by variance due to sampling (\cref{sec:width_var_decoupling}).

    \item We remark that this variance phenomenon is already present in over-parameterized linear models. In a simplified setting, inspired by linear models, we provide theoretical analysis in support of our empirical findings (\cref{sec:theory}).
    
    
\end{enumerate}

\section{Related Work}

\citet{DBLP:journals/corr/NeyshaburTS14} point out that because increasing network width does not lead to a U-shaped test error curve, there must be some form of implicit regularization controlling capacity. Then, the line of questioning becomes ``if not number of parameters, what is the correct measure of model complexity that, when varied, will yield a tradeoff in bias and variance?'' \citet{neyshaburthesis,neyshabur2018the} pursue this direction by studying how test error correlates with different measures of model complexity and by developing models in terms of those new complexity measures.

Our work is consistent with \citet{DBLP:journals/corr/NeyshaburTS14}'s finding, but rather than search for a more appealing measure of model complexity, we study whether it is \emph{necessary} to trade bias for variance. By varying network width (the measure of model complexity that \citet{geman} claimed shows a bias-variance tradeoff), we establish that it is \emph{not necessary} to trade bias for variance when increasing model complexity. To ensure that we are studying networks of increasing capacity, one of the experimental controls we use throughout \cref{article:bias-variance} is to verify that bias is decreasing.




In concurrent work, \citet{jamming, belkin2018} point out that generalization error acts according to conventional wisdom in the under-parameterized setting, that it decreases with capacity in the over-parameterized setting, and that there is a sharp transition between the two settings. While this transition can roughly be seen as the early hump in variance we observe in some of our graphs, we focus on the over-parameterized setting.
\citet{scaling,neyshabur2018the,fisher-rao_metric} work toward understanding why increasing over-parameterization does not lead to a U-shaped test error curve.
Our work is unique in that we explicitly analyze and experimentally measure the quantities of bias and variance. Interestingly, \citet{belkin2018}'s empirical study of test error provides some evidence that our bias-variance finding might not be unique to neural networks and might be found in other models such as decision trees.

\section{Organization}

In \cref{sec:background}, we cover relevant background: the setting in machine learning, the concept of generalization in machine learning, and the concept of model complexity. In \cref{sec:bias-variance}, we cover the bias-variance tradeoff in detail, including topics such as why the bias-variance tradeoff is convincing and its relation to the concepts of generalization and model complexity. Then, we argue the bias-variance tradeoff is applied too broadly in \cref{sec:lack-of-tradeoff} and give specific recommendations for changes in \cref{sec:need-to-qualify-claims}. In \cref{article:bias-variance}, we provide evidence for this in neural networks.


\chapter{Machine Learning Background}
\label{sec:background}

\section{Setting and Notation}

We consider the typical supervised learning task of predicting  an output $y \in \mathcal{Y}$ from an input $x \in \mathcal{X}$, where the pairs $(x, y)$ are drawn from some unknown joint distribution, $\mathcal{D}$. The learning problem consists of learning a function $h_S:  \X \to \Y$ from a finite training dataset $S$ of $m$ i.i.d.\  samples from $\D$. This learned function is also known as a hypothesis $h \in \HH$, which is chosen from a \emph{hypothesis class} $\HH$ of possible functions allowed by the model. Then, the learned function $h_S$ and the learning algorithm $\A : (\X \times \Y)^m \to \HH$ can be formalized as $h_S \leftarrow \A(S)$. Ideally, we would learn $h_S = f$, where $f$ denotes the ``true mapping'' from $\X$ to $\Y$. For some loss function $\ell : \mathcal{Y} \times \mathcal{Y} \to \R$, the quality of a predictor $h$ can quantified by the \emph{risk} (or \emph{expected error}):
\begin{equation*}
    \risk(h) = \E_{(x, y) \sim \mathcal{D}} \, \ell(h(x), y) \, .
\end{equation*}

The goal in supervised learning is to find $\min_{h \in \HH} \risk(h)$. However, we cannot compute $\risk(h)$ because we do not know $\D$. We only have access to the \emph{training error} (a type of \emph{empirical risk}):
\begin{equation*}
    \emprisk(h) = \E_{(x, y) \sim \mathcal{S}} \, \ell(h(x), y) \, .
\end{equation*}
This naturally leads to the concept of \emph{empirical risk minimization}: we learn $h_S$ by attempting to minimize $\emprisk(h_S)$ as a surrogate for $\risk(h_S)$.

\section{Generalization}

We would like that the learned function $h_S$ generalizes well from the training set $S \sim \D^m$ to other unseen data points drawn from $\D$. The name ``generalization'' comes from psychology; for example, if a dog is taught to sit with the verbal cue ``sit'' by its owner, and then told ``sit'' by another person, if the dog sits, it would be generalizing. If the dog were to only sit when it hears the exact same sound (made by its owner) it was trained on, it would be ``overfitting'' and failing to generalize. Overfitting is something to take very seriously in machine learning.

What can go wrong when minimizing the training error $\emprisk$ as a surrogate for minimizing the true risk $\risk$? If the hypothesis class $\HH$ allows for it, $h$ can fit the data sample too closely, leading to a higher true risk than some $h'$ that has a higher training error than $h$. More precisely, $h$ can be worse than $h'$ even when $\emprisk(h) < \emprisk(h')$ because $\risk(h) > \risk(h')$. This can be easily visualized by an example.

In \cref{fig:noisy_sinusoids}, we see data coming from a noisy sinusoid task \citep{wtf_is_bv}. On the left, a linear model is fit to the data. This leads to both high training error and high true risk. In other words, the model is not complex enough. On the right, a much more complex model is fit to the data. This leads to zero training error, as the learned function fits every training point. However, it will also lead to high true risk as it will not generalize well to unseen data. This is because it is fitting the data too closely, fitting the noise in the data, and, hence, overfitting. The linear model was too simple at the highly complex model was too complex. In the middle, we see a model of about the right complexity that learns a function that will generalize the best of the three.

One notion of generalization that is often seen in statistical learning theory is the \emph{generalization gap}. This is simply the difference between the true risk and the training error:
\begin{equation}
    \gap(h) = \risk(h) - \emprisk(h)
\end{equation}

\begin{figure}[t]
    \centering
    \begin{subfigure}
        \centering
        \includegraphics[width=.32\textwidth]{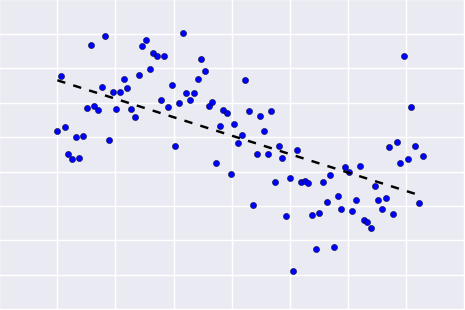}
    \end{subfigure}
    \hfill
    \begin{subfigure}
        \centering
         \includegraphics[width=.32\textwidth]{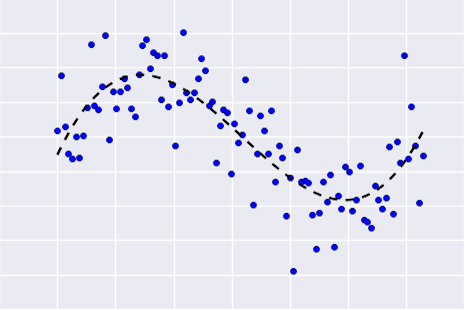}
    \end{subfigure}
    \hfill
    \begin{subfigure}
        \centering
         \includegraphics[width=.32\textwidth]{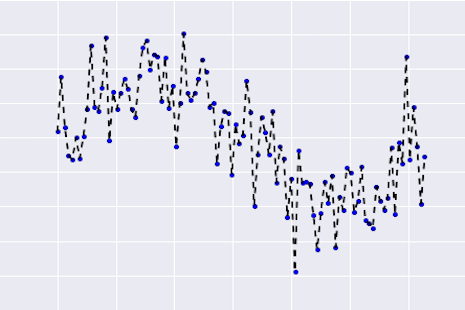}
    \end{subfigure}
    \caption[Increasingly complex models fit to sinusoidal data]{Increasingly complex models fit to sinusoidal data \citep{wtf_is_bv}}
    \label{fig:noisy_sinusoids}
\end{figure}

\section{Model Complexity}
\label{sec:model-complexity}

In \cref{fig:noisy_sinusoids}, the key concept that varies (increases from left to right) is \emph{model complexity} (i.e., the complexity of the hypothesis class $\HH$). Models that are not sufficiently complex will underfit, while models that are too complex will overfit (see, e.g., \cref{fig:noisy_sinusoids}). In terms of hypothesis classes, the larger $|\HH|$ is, the more functions exist that will fit the training data, but they might not perform well on unseen data. It is intuitive that the larger $|\HH|$ is, the more our model will overfit (see \cref{fig:bv-hypothesis-class} in \cref{sec:bv-intuition} for an illustration of this in the bias-variance framework). And indeed, there is theory that supports this intuition \citep[Theorem 2.2]{Mohri:2012}. For any $\delta > 0$, with probability at least $1 - \delta$,
\begin{equation}
    \forall h \in \HH, \qquad \risk(h) \leq \emprisk(h) + \sqrt{\frac{\log | \HH | + \log \frac{2}{\delta}}{2m}} \, .
\end{equation}

The quantity $|\HH|$ is one notion of model complexity. However, for many models (e.g. neural networks), $|\HH|$ is infinite. Therefore, a better notion of model complexity is needed. The VC dimension of $\HH$, $\VC(\HH)$, is a better notion of model complexity, which leads to a finite bound with infinite models classes \citep[Chapter 3.3]{Mohri:2012}. For any $\delta > 0$, with probability at least $1 - \delta$,
\begin{equation}
    \forall h \in \HH, \qquad \risk(h) \leq \emprisk(h) + \sqrt{\frac{8\VC(\HH) \log \frac{2em}{\VC(\HH)} + 8 \log \frac{4}{\delta}}{m}} \, .
\end{equation}

This bounds grows with $\VC(\HH)$. For many models the VC dimension ends up being roughly proportional to the number of parameters in the model. For example, the VC dimension of various kinds of neural networks grows with the number of parameters \citep{Baum, Karpinski, NIPS1998_1515, pmlr-v65-harvey17a}.

Rademacher complexity is another measure of model complexity. Intuitively, it measures the capacity of a model to fit random noise. Generalization bounds in terms of Rademacher complexity are also prevalent \citep[Theorem 3.2]{Mohri:2012}:
\begin{equation}
    \forall h \in \HH, \qquad \risk(h) \leq \emprisk(h) + \Rad(\HH) + \sqrt{\frac{\log \frac{1}{\delta}}{2m}}
\end{equation}
where $\Rad(\HH)$ denotes the Rademacher complexity of $\HH$. Known bounds on Rademacher complexity also grow with the number of parameters \citep{Bartlett:2003:RGC:944919.944944}.

These generalization bounds in terms of model complexity are important because of how they are interpreted. The general idea is that the model must be complex enough to achieve a low $\emprisk(h)$, but not too complex that the complexity measures such as $\VC(\HH)$ and $\Rad(\HH)$ will blow up, leading to high bounds on $\risk(h)$. For example, when interpreting the VC-based generalization bound, \citet[Chapter 2.2]{Abu-Mostafa:2012:LD:2207825} wrote,
``Although the bound is loose, it tends to be equally loose for different learning models, and hence is useful for comparing the generalization performance of these models. [...] In real applications, learning models with lower $\VC(\HH)$ tend to generalize better than those with higher $\VC(\HH)$. Because of this observation, the VC analysis proves
useful in practice [...] the VC bound can be used as a guideline for generalization, relatively if not absolutely.''

\chapter{The Bias-Variance Tradeoff}
\label{sec:bias-variance}

\section{What are Bias and Variance?}

The bias is a measure of how close the central tendency of a learner is to the true function $f$. If, on average (over training sets $S$), the learner learns the true function $f$, then the learner is unbiased. For some $x \sim \D$, the bias is 
\begin{equation*}
    \Bias(h_S) = \E_S[h_S(x)] - f(x) \, .
\end{equation*}

The variance is a measure of fluctuations of a learner around its central tendency, where the fluctuations result from different samplings of the training set. By definition, a learner that generalizes well does not learn dramatically different functions, depending on sampling of the training set. Let $\Y = \reals$ for simplicity; then for some $x \sim \D$, the variance is
\begin{equation*}
    \Var(h_S) = \E_S \left[ \left( h_S(x) - \E_S[h_S(x)] \right)^2 \right] \, .
\end{equation*}

\section{Intuition for the Tradeoff}
\label{sec:bv-intuition}

Similar to the idea that larger hypothesis classes lead to overfitting in \cref{sec:model-complexity}, in the bias-variance context, there is the idea that larger hypothesis classes lead to higher variance. This is illustrated in \cref{fig:bv-hypothesis-class}, which comes from \citet{Abu-Mostafa:2012:LD:2207825}. Also, illustrated is the idea that bias decreases when increasing the size of the hypothesis class because their will be more hypotheses that are closer to the true function $f$. In \cref{fig:bv-hypothesis-class}, a hypothesis class that contains only a single hypothesis is depicted on the left; this will, of course, lead to bias as that hypothesis does not match $f$, but it will also lead to zero variance, which is a positive. In contrast, on the right, there is a larger hypothesis class (\cref{fig:bv-hypothesis-class}); in this example, this leads to nearly zero bias, but it comes at the expense of incurring variance. The bottom of \cref{fig:bv-hypothesis-class} is shorthand that summarizes the idea that when you increase the size of the hypothesis class, you decreases bias and increase variance. \cref{fig:noisy_sinusoids} is another example of this: the leftmost learner has high bias and low variance, the rightmost learner has low bias and high variance, and the learner in the middle has something close to the optimal balance of bias and variance.

\begin{figure}[t]
    \centering
    \includegraphics[width=\textwidth]{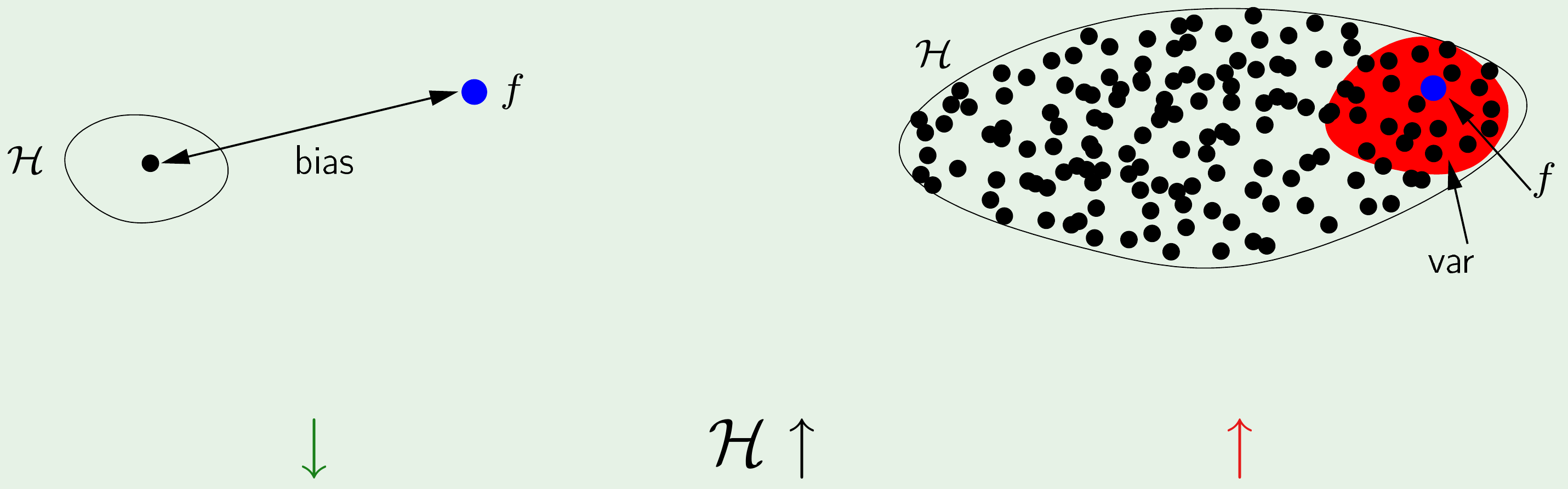}
    \caption[Bias-variance in simple vs.\ complex hypothesis class]{Bias-variance in simple vs.\ complex hypothesis class \citep{Abu-Mostafa:2012:LD:2207825}}
    \label{fig:bv-hypothesis-class}
\end{figure}

In their landmark paper, \citet{geman} capture the essence of the bias-variance tradeoff with the following claim: ``the price to pay for achieving low bias is high variance.'' In \cref{fig:bv-tradeoff-fortmann}, we see the common illustration of the bias-variance tradeoff (\citet{fortmann-roe_2012}).
Note the important U shape of the test error curve with increasing model complexity. The idea is that the optimal point on that U can be achieved by achieving the optimal balance of bias and variance. This tradeoff hypothesis is ubiquitious, as we will see in \cref{sec:textbooks}.

\begin{figure}[t]
    \centering
    \includegraphics[width=.7\textwidth]{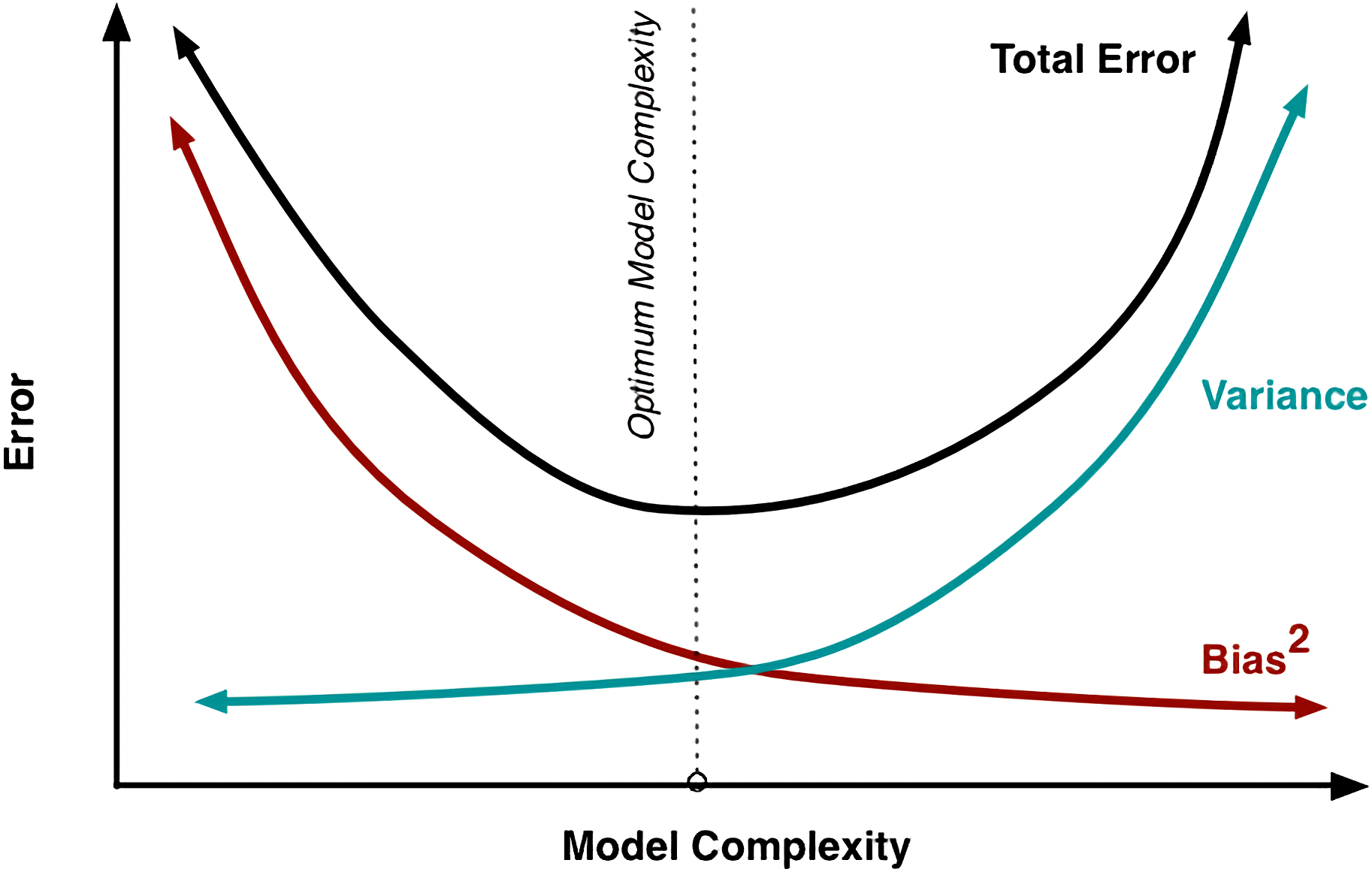}
    \caption[Illustration of the bias-variance tradeoff]{Illustration of the bias-variance tradeoff \citep{fortmann-roe_2012}}
    \label{fig:bv-tradeoff-fortmann}
\end{figure}

\section{The Bias-Variance Decomposition}
\label{sec:bias-variance-decomposition}

\citet{geman} considered the average case (over training sets) quantity $\E_S \risk(h_S)$ with squared-loss and showed that it can be cleanly decomposed into bias and variance components:
\begin{equation}
    \E_S \risk(h_S) = \ebias(h_S) + \evar(h_S) + \enoise
\end{equation}

Although a decomposition does not prove that the bias-variance tradeoff is true, it does show that the average error is made up of a sum of bias and variance components. Then, if the average error is held constant and bias is varied, variance must also vary (and vice versa). This greatly added to the strong intuition of the bias-variance tradeoff, and \citet{geman} became quite highly cited for their contribution.

Note that risks computed with classification losses (e.g cross-entropy or 0-1 loss)  do not have such a clean, additive bias-variance decomposition \citep{Domingos00aunified, James03varianceand}. However, because the concept of a tradeoff is not reliant on an additive decomposition (see \citet[Chapter 3.3]{hastie1990generalized} for presence of the bias-variance tradeoff before the bias-variance decomposition), the concept of the bias-variance tradeoff is applied extremely broadly (see \cref{sec:textbooks}), including settings where an additive decomposition does not seem possible.

\section{Why do we believe the Bias-Variance Tradeoff?}
\label{sec:why-we-believe-tradeoff}

A universal bias-variance tradeoff, without qualifications, is a mere hypothesis.
Potentially, its most significant appeal comes from its intuitiveness.
In this section, we review the history of the bias-variance tradeoff, the evidence in support of it, and its prevalence in textbooks (which also seem to contain much of the authoritative evidence).

\subsection{A History}

The concept that we know as the ``bias-variance tradeoff'' in machine learning has a long history, with its basis in statistics.
\textit{Neural Networks and the Bias/Variance Dilemma} \citep{geman} is the most cited work largely because it introduced the bias-variance decomposition to the machine learning community, provided convincing experiments with nonparametric methods, and popularized the bias-variance tradeoff in the neural network and machine learning  community. However, the bias-variance tradeoff was already present in a textbook in 1990 \citep{hastie1990generalized}, and it dates back at least as far back as 1952 in statistics when \citet{grenander1952} referred to the concept as an ``uncertainty principle.''

\subsubsection{Experimental Evidence}
\label{sec:bv-experimental-evidence}

\citet{geman} ran experiments using two nonparametric methods (KNN and kernel regression) and neural networks on a partially corrupted version of the handwritten digits \citet{guyon1988} collected (\cref{fig:geman}). The experiments on k-nearest neighbor (KNN) (\cref{fig:geman-knn}) and kernel regression (\cref{fig:geman-kernel-reg}) yield clear bias-variance tradeoff curves with U-shaped test error curves in their respective complexity parameters $k$ and $\sigma$. The experiment with neural networks (\cref{fig:geman-neural-net}) is substantially less conclusive. \citet{geman} maintain their claim that there is a bias-variance tradeoff in neural networks and explain their inconclusive experiments as a result of convergence issues:

\begin{quote}
The basic trend is what we expect: bias falls and variance increases with the number of hidden units. The effects are not perfectly demonstrated (notice, for example, the dip in variance in the experiments with the largest numbers of hidden units), presumably because the phenomenon of overfitting is complicated by convergence issues and perhaps also by our decision to stop the training prematurely.
\end{quote}

This is the first glimpse we see of the cracks in the bias-variance tradeoff hypothesis.

\begin{figure}[t]
    \centering
    \subfigure[K-nearest neighbor (KNN) (higher $k$ is \emph{less} complexity)]{
        \includegraphics[width=.3115\textwidth]{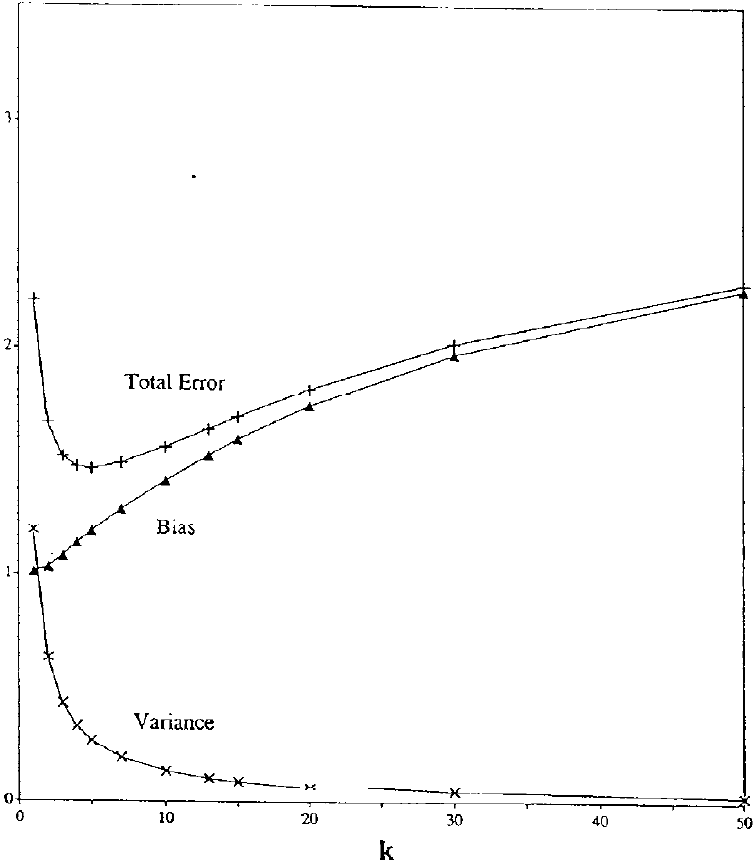}
        \label{fig:geman-knn}
    }
    \hfill
    \subfigure[Kernel regression (higher kernel width $\sigma$ is \emph{less} complexity)]{
        \includegraphics[width=.3115\textwidth]{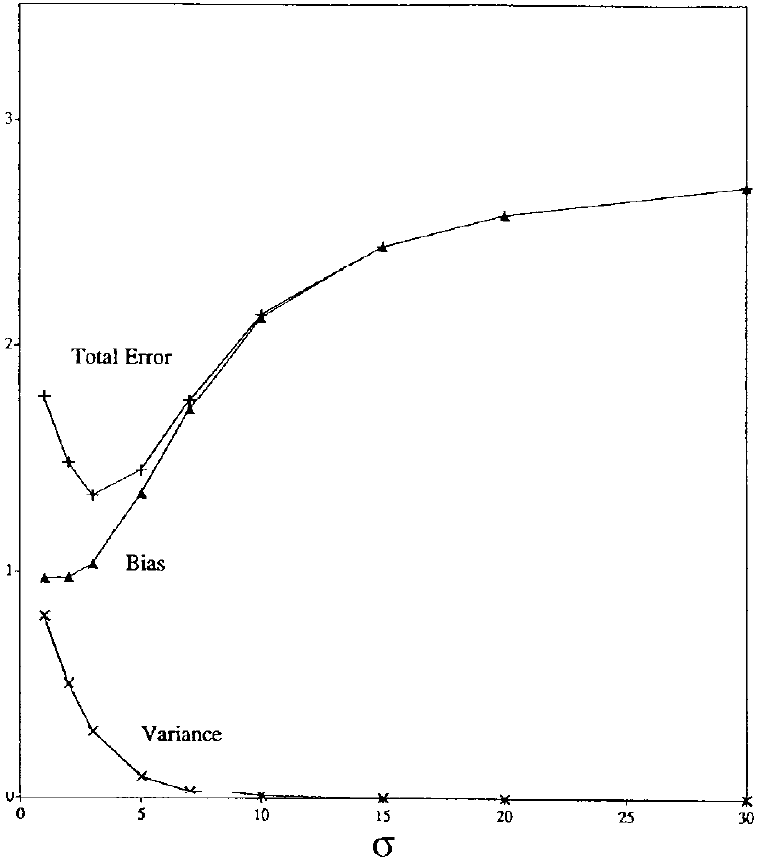}
        \label{fig:geman-kernel-reg}
    }
    \hfill
    \subfigure[Single hidden layer neural network (higher ``\# Hidden Units'' is \emph{more} complexity)]{
        \includegraphics[width=.3115\textwidth]{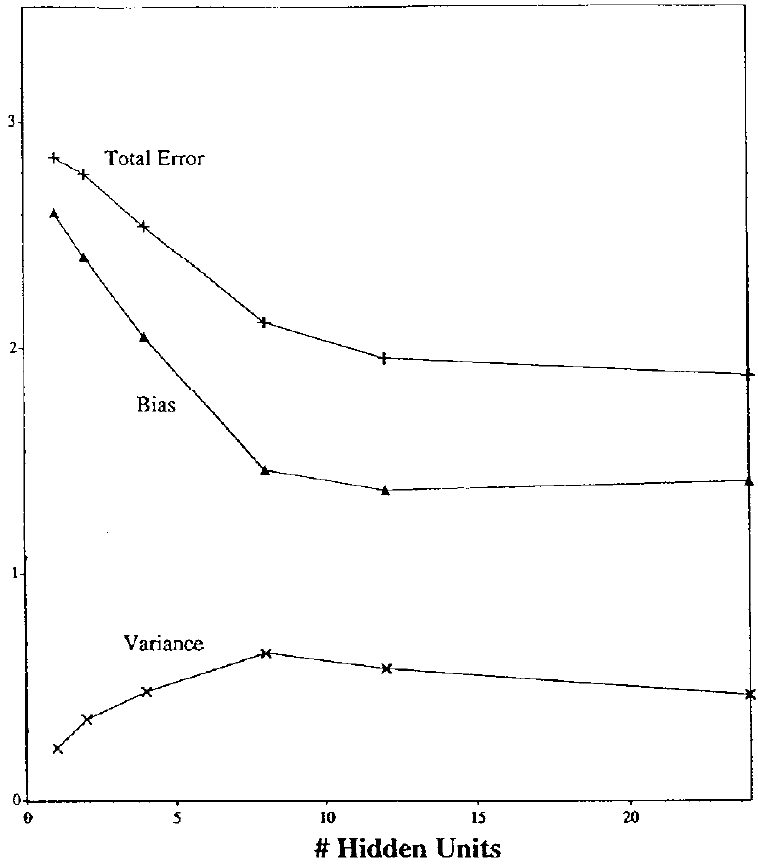}
        \label{fig:geman-neural-net}
    }
    \caption{\citet{geman}'s bias-variance experiments on handwritten digits}
    \label{fig:geman}
\end{figure}
There is a fair amount of empirical evidence for the bias-variance tradeoff in different complexity parameters in a variety of methods. \citet{Wahba1975} show a tradeoff in complexity with cubic splines when varying their smoothing parameter. \citet[Section 5.4.3]{geurts2002} show a bias-variance tradeoff in decision trees when varying tree size. \citet[Chapter 3.3]{hastie1990generalized} show a bias-variance tradeoff with a ``running-mean smoother'' (i.e. KNN, but in statistics) when varying $k$. \citet[Chapter 3.2]{Bishop:2006} show a bias-variance tradeoff with Gaussian basis function linear regression when varying the L2 regularization (weight decay) parameter $\lambda$. \citet[Chapters 5.2 and 5.3]{Goodfellow-et-al-2016} show a tradeoff in complexity with fitting polynomials when varying either the degree or the L2 regularization parameter.

\subsubsection{Supporting Theory}
\label{sec:bv-supporting-theory}

The first main theory that supports the bias-variance tradeoff is the sample of generalization upper bounds presented in \cref{sec:model-complexity} that grow with the number of parameters. This is simply because the bias-variance tradeoff is a clear conceptualization of one of the common interpretations of those bounds: models that are too simple will not perform well due to underfitting (high bias or high training error $\emprisk$) while models that are too complex will not perform well due overfitting (high variance caused by high model complexity such as $\VC(\HH)$). However, it should be noted that these upper bounds do not guarantee that a model with high $\VC(\HH)$ will actually have high variance (or high test error); a lower bound for cases seen in practice (not worst case) would be needed for that.

\citet[Chapter 7.3]{hastie01statisticallearning} show that you can derive closed-form expressions for the variance for simple models such as KNN and linear regression in the ``fixed-design'' setting where the design matrix $X$ is fixed. In this setting, $Y = f(X) + \epsilon$, where $f$ is the true function mapping examples $x$ (rows in $X$) to $y$ (elements in the vector $Y$). Then, the randomness in $Y$ is determined by the zero-mean random variable $\epsilon$. In this setting, \citet[Chapter 7.3]{hastie01statisticallearning} show that variance for KNN models scales as $\frac{1}{K}$. Similarly, they show that the variance for linear regression grows linearly with the number of parameters, assuming $X^T X$ is invertible. Note that in the over-parameterized setting ($X^T X$ is not invertible), we show that the variance of linear regression does not grow with the number of parameters (see \cref{sec:linear_models} in \cref{article:bias-variance}).

Because neural networks are much more complicated models than KNN and linear regression, we must resort to bounds on the variance of a neural network. \citet{Barron1994} derives an upper bound on the estimation error of a single hidden layer neural network that grows linearly with the number of hidden units. The estimation error is not the same thing as variance, but it is analogous (see \cref{sec:approx-est-tradeoff}). This kind of bound is similar to the bounds described in \cref{sec:model-complexity}. Again, it should be noted that because this is an upper bound, it does not actually imply that large neural networks will have high estimation error.


\subsection{The Textbooks}
\label{sec:textbooks}

The concept of the bias-variance tradeoff is ubiquitious, appearing in many of the textbooks that are used in machine learning education:
\citet[Chapters 2.9 and 7.3]{hastie01statisticallearning},
\citet[Chapter 3.2]{Bishop:2006}, \citet[Chapter 5.4.4]{Goodfellow-et-al-2016}),
\citet[Chapter 2.3]{Abu-Mostafa:2012:LD:2207825},
\citet[Chapter 2.2.2]{James:2014:ISL:2517747}, \citet[Chapter 3.3]{hastie1990generalized}, \citet[Chapter 9.3]{DudaHart2001}. Here are two excerpts:
\begin{itemize}
    \item ``As a general rule, as we use more flexible methods, the variance will
    increase and the bias will decrease. The relative rate of change of these
    two quantities determines whether the test MSE increases or decreases. As
    we increase the flexibility of a class of methods, the bias tends to initially
    decrease faster than the variance increases. Consequently, the expected
    test MSE declines. However, at some point increasing flexibility has little
    impact on the bias but starts to significantly increase the variance. When
    this happens the test MSE increases'' \citep[Chapter 2.2.2]{James:2014:ISL:2517747}.
    
    \item ``As the model complexity of our procedure is increased, the variance tends to increase and the squared bias tends to decrease'' \citep[Chapters 2.9]{hastie01statisticallearning}.
\end{itemize}

\section{Comparison to the Approximation-Estimation Tradeoff}
\label{sec:approx-est-tradeoff}

The bias-variance tradeoff is not the only tradeoff in machine learning that is related to generalization. For example, when $h_S$ is chosen from a hypothesis class $\HH$, $\risk(h_S)$ can be decomposed into approximation error and estimation error:
\begin{equation*}
    R(h_S) = \errapp + \errest
\end{equation*}
where $\errapp = \min_{h \in \HH} R(h)$ and $\errest = R(h_S) - \errapp$. \citet[Section 5.2]{shalev-shwartz_ben-david_2014} present this decomposition and frame it as a tradeoff. \citet{bottou2008} describe this as the ``well known tradeoff between approximation error and estimation error'' and present it in a slightly more lucid way as a decomposition of the \textit{excess risk}:
\begin{equation*}
    \E[R(h_S) - R(h^*)] = \E[R(h_{\HH}^*) - R(h^*)] + \E[R(h_S) - R(h_{\HH}^*)]
\end{equation*}
where $\risk(h^*)$ is the Bayes error and $h_{\HH}^* = \argmin_{h \in \HH} \risk(h)$ is the best hypothesis in $\HH$. The approximation error can then be interpreted as the distance of the best hypothesis in $\HH$ from the Bayes classifier, and the estimation error can be interpreted as the average distance of the learned hypothesis from the best hypothesis in $\HH$. It is common to associate larger $\HH$ with smaller approximation error and larger estimation error, just like it is common to associate larger $\HH$ with smaller bias and larger variance. While bias (variance) and approximation error (estimation error) are qualitatively similar, they are not the exact same.

\subsection{Universal Approximation Theorem for Neural Networks}

The commonly cited universal approximation property of neural networks \citep{Cybenko1989, HORNIK1991251, Leshno93multilayerfeedforward} means that the approximation error goes to 0 as the network width increases; these results do not say anything about estimation error. In other words, the universal approximation error does not imply that wider networks are better. It implies that wider networks yield lower \emph{approximation error}; traditional thinking suggests that this also means that wider networks yield \emph{higher estimation error}.


\chapter{The Lack of a Tradeoff}
\label{sec:lack-of-tradeoff}

\section{A Refutation of Geman et al.'s Claims}
\label{sec:geman-refutation}

In their highly influential paper, \citet{geman} make several claims of varying specificity. We will start with their most general claim, which is simply a statement of the bias-variance tradeoff:

\begin{quote}
the price to pay for achieving low bias is high variance.
\end{quote}

However, this claim is not true in the general case. The fact that the expected risk can be decomposed into squared bias and variance does imply that the two terms trade off. It is not necessary to trade bias for variance in all settings.
For example, it is not necessary to trade bias for variance in neural networks (\citet{neal2018}, \cref{article:bias-variance}).

\begin{figure}[t]
    \centering
    \includegraphics[width=\textwidth]{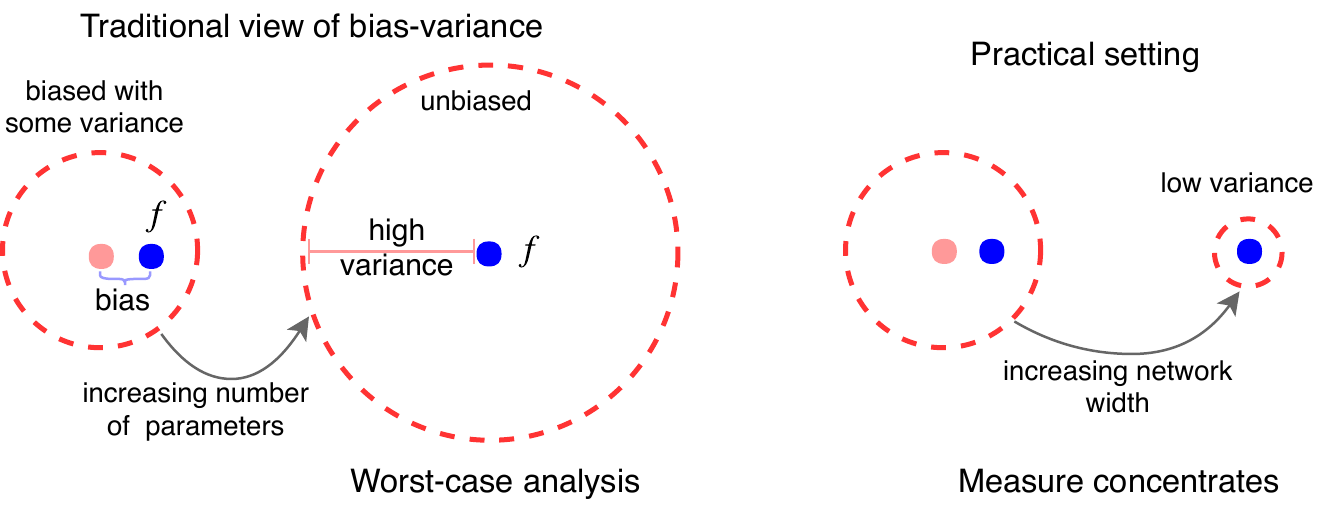}
    \caption{Bias-variance in simple vs.\ complex hypothesis class}
    \label{fig:bv-nn-hypothesis-class}
\end{figure}

As neural networks are a main focus of \citet{geman}'s work, they also make a clear claim relevant to neural networks: ``bias falls and variance increases with the number of hidden units.'' We directly test this claim in \cref{article:bias-variance} and show it to be false on a variety of datasets. \emph{Both} bias \emph{and} variance can decrease as network width increases. In \cref{fig:bv-nn-hypothesis-class} we contrast the common intuition about the bias-variance tradeoff (left, as inspired by \cref{fig:bv-hypothesis-class}) with what we observe in neural networks (right). This is a specific example of not having to pay any price of increased variance when decreasing bias.

In fact, \citet{geman}'s own experiments with neural networks do not even support their claim that ``bias falls and variance increases with the number of hidden units.'' \citet[Figures 16 and 8]{geman} run experiments with a handwritten digit recognition dataset (\cref{fig:geman-neural-net2}) and with a sinusoid dataset (\cref{fig:geman-neural-net-sin}). In both of these datasets, they see decreasing variance when increasing the number of hidden units (\cref{fig:geman-NN}). \citet[page 33]{geman} explain this seeming evidence against their claim as a product of ``convergence issues,'' and maintain their claim: ``The basic trend is what we expect: bias falls and variance increases with the number of hidden units. The effects are not perfectly demonstrated (notice, for example, the dip in variance in the experiments with the largest numbers of hidden units), presumably because the phenomenon of overfitting is complicated by convergence issues and perhaps also by our decision to stop the training prematurely.''

In their paper, \citet{geman} give the following prescription for choosing the width of a neural network: ``How big a network should we employ? A small network, with say one hidden unit, is likely to be biased, since the repertoire of available functions spanned by $f(x; w)$ over allowable weights will in this case be quite limited. If the true regression is poorly approximated within this class, there will necessarily be a substantial bias. On the other hand, if we overparameterize, via a large number of hidden units and associated weights, then the bias will be reduced (indeed, with enough weights and hidden units, the network will interpolate the data), but there is then the danger of a significant variance contribution to the mean-squared error.'' Although this fits the conventional wisdom laid out in \cref{sec:model-complexity} and \cref{sec:bias-variance}, we find this way of thinking to be misleading, leading researchers to incorrect predictions.

\begin{figure}[t]
    \centering
    \subfigure[Bias, variance, and total error as a function of number of hidden units in \citet{geman}'s handwritten digit recognition experiment. Note that variance is decreasing with width in roughly the last $\frac{2}{3}$ of the graph.]{
        \includegraphics[height=7cm]{figures/geman/Geman_NN.pdf}
        \label{fig:geman-neural-net2}
    }
    \qquad
    \subfigure[Bias (o), variance (x), and total error (+) as a function of number of hidden units in \citet{geman}'s noise-free (deterministic) and noisy (ambiguous) sinusoid classificiation experiments. Note that decreasing variance is already seen in the deterministic (top) experiment with neural networks as small as 7-15 hidden units.]{
        \includegraphics[height=7.1cm]{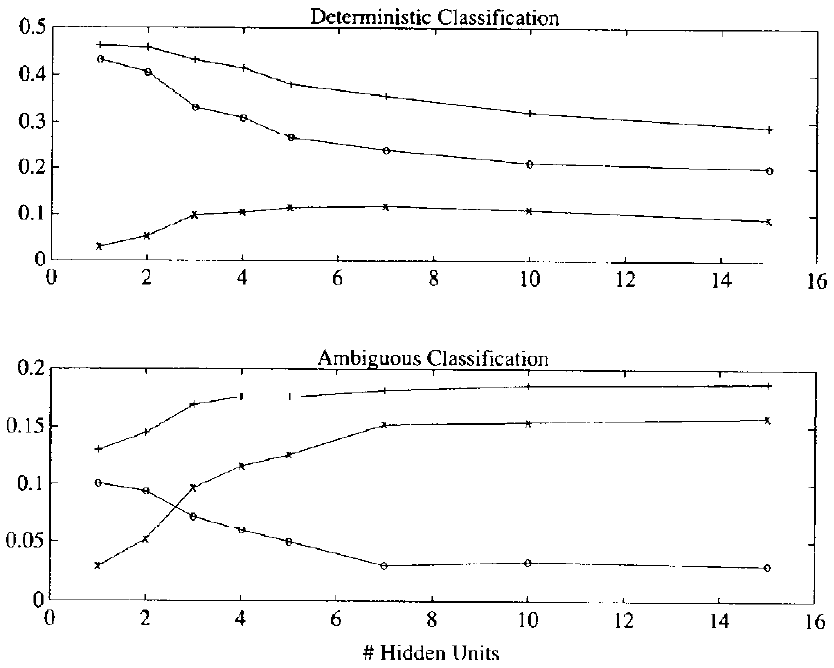}
        \label{fig:geman-neural-net-sin}
    }
    \caption{\citet{geman}'s neural network experiments}
    \label{fig:geman-NN}
\end{figure}

\section{Similar Observations in Reinforcement Learning}
\label{sec:rl-tradeoff}

\citet{DBLP:journals/corr/NeyshaburTS14} found that increasing the width of a single hidden layer neural network leads to \emph{decreasing} test error on MNIST and CIFAR10 until it levels off (never going back up). We explore whether this phenomenon extends to deep reinforcement learning. We provide some evidence that it does, finding that wider networks do seem to perform better than their smaller counterparts in deep reinforcement learning as well \citep{nealOverparamRL}. Combining that with the results from \cref{article:bias-variance}, we infer that very wide networks do not have to suffer from high variance (in exchange for low bias) in reinforcement learning either.

\section{Previous Work on Boosting}
\label{sec:boosting}

\citet{buhlmann2003boosting}'s work on the bias-variance tradeoff in boosting is motivated by ``boosting's resistance to overfitting'' when increasing the number of iterations. For example, in \citet[Figures 8-10]{Schapire1999}, they run experiments on many datasets, finding that, on some datasets, test error decreases and plateaus without increasing with more iterations of boosting. This is only the case in some of their experiments, as in roughly half of their experiments, \citet{Schapire1999} find that test error does eventually increase with number of iterations of boosting. Still, because roughly half of the experiments show ``boosting's resistance to overfitting,'' \citet{buhlmann2003boosting} study bias-variance in boosting.

\citet{buhlmann2003boosting} find that the lack of increasing test error when increasing number of iterations (``boosting's resistance to overfitting'') can be explained in terms of bias and variance. In Theorem 1, they show exponentially decaying bias and variance that grows at an exponentially decaying rate with number of iterations. There are some specifics to this that are related to the strength/weakness of the learner that is boosted, but this is how they explain why monotonically decreasing test error can sometimes be seen when increasing the number of iterations in boosting: ``(2) Provided that the learner is sufficiently weak, boosting always improves, as we show in Theorem 1'' \citep[Section 3.2.2]{buhlmann2003boosting}.

All this said, \citet{buhlmann2003boosting}'s work should not be interpreted as showing a lack of a bias-variance tradeoff in boosting. Rather, it shows a bias-variance tradeoff where the growth of variance with the complexity parameter is exponentially smaller than that in the traditional bias-variance tradeoff (see \cref{fig:boosting-vs-spline}), which implies that variance does not grow forever when increasing the number of iterations. Their work is an important example of a departure from the conventional bias-variance tradeoff.

\begin{figure}[t]
    \centering
    \subfigure[Test error as a function of the number of iterations of boosting. Although it decreases and then increases, it increases at a much slower rate than the classic bias-variance tradeoff would suggest (see figure on the right). \citet{buhlmann2003boosting} call this slower increase an ``a new exponential bias-variance tradeoff.'']{
        \includegraphics[width=.4\textwidth]{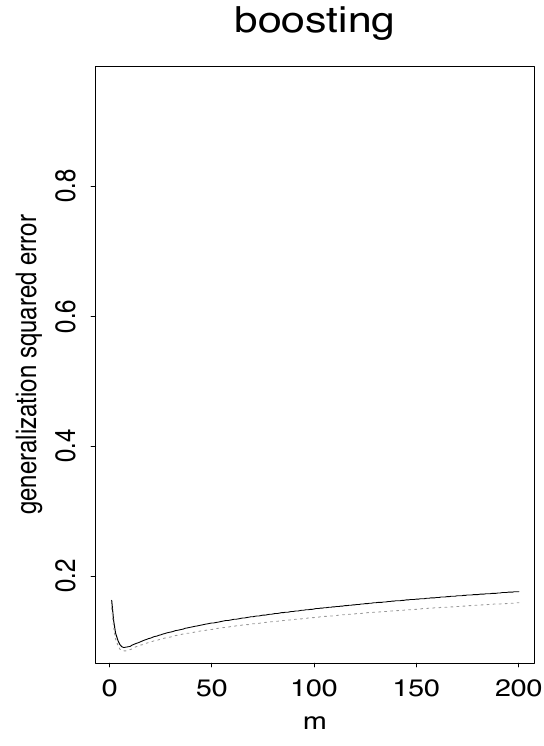}
        \label{fig:boosting-test-error}
    }
    \qquad
    \subfigure[Test error as a function of the amount of smoothing in a cubic spline. This is the test error curve that the classic bias-variance tradeoff would suggest; test error increases with the complexity parameter at a rate that is at least linear.]{
        \includegraphics[width=.4\textwidth]{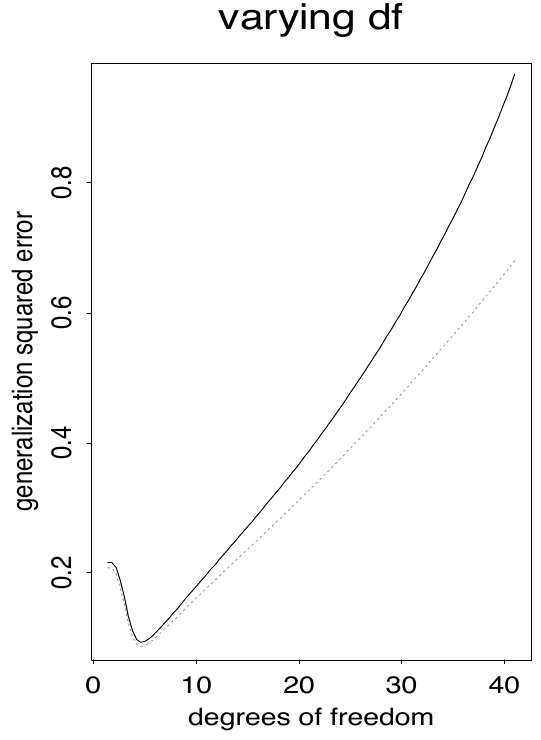}
        \label{fig:spline-test-error}
    }
    \caption[Exponential bias-variance tradeoff in boosting]{Exponential bias-variance tradeoff in boosting \citep{buhlmann2003boosting}}
    \label{fig:boosting-vs-spline}
\end{figure}

\section{The Double Descent Curve}
\label{sec:double-descent}

A conjecture that has recently gained popularity is the idea the risk behaves as a ``double descent'' curve in model complexity (\cref{fig:double-descent}). Specifically, the idea is that the risk behaves according to the classical bias-variance tradeoff wisdom (\cref{sec:bias-variance}) in the under-parameterized regime; the risk decreases with model complexity in the over-parameterized regime; and there is a sharp transition from the under-parameterized regime to the over-parameterized regime where the training error is 0. \citet{belkin2018} illustrate this in \cref{fig:double-descent}.

\begin{figure}[t]
    \centering
    \includegraphics[width=.7\textwidth]{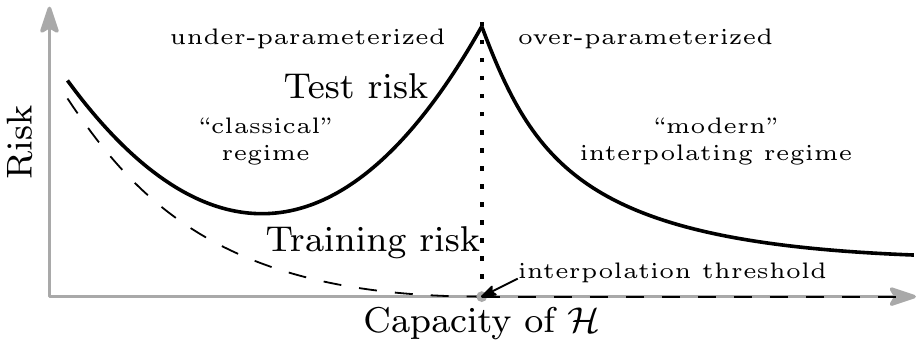}
    \caption[Double descent risk curve conjectured to replace traditional U-shaped curve]{Double descent curve, showing U-shaped risk curve in under-parameterized regime and decreasing curve in over-parameterized regime \citep{belkin2018}.}
    \label{fig:double-descent}
\end{figure}

In previous work, \citet{Advani2017HighdimensionalDO} observed this phenomenon in linear student-teacher\footnote{``Teacher'' here refers to the fact that the data is generated by a neural network.} networks and with nonlinear networks on MNIST. In concurrent (to \cref{article:bias-variance}) work, \citet{jamming,Geiger2019TheJT,belkin2018} also studied this phenomenon. \citet{jamming,Geiger2019TheJT} described the cusp in the double descent curve as corresponding to a phase transition and draw the analogy to the ``jamming transition'' in particle systems. \citet{belkin2018} conjectured that this phenomenon is fairly general (as opposed to just being restricted to neural networks). \citet{belkin2018} showed the phenomenon in random forests, in addition to neural networks, and coined the term ``double descent.'' \citet{deepdoubledescent} recently showed that this double descent phenomenon is present in many state-of-the-art architectures such as convolutional neural networks, ResNets, and transformers, as opposed to only being present in more toy settings. The double descent phenomenon in simple settings such as shallow linear models can be seen in work that dates as far back as 1995 \citep{opper1995statistical,opper2001learning,bosOpper1997dynamics}.

Our work in \cref{article:bias-variance} is consistent with the double descent curve. Although we were not looking for the cusp in the double descent curve (can require dense sampling of model sizes and specific experimental details), we do seem to see it in several variance figures in \cref{article:bias-variance}. All the works on the double descent curve examine the risk (or test error). In order to test the bias-variance hypothesis, it is important to actually measures bias and variance because test error and bias can decrease while variance still increases at an exponentially decaying rate (\cref{sec:boosting}).

\section{The Need to Qualify Claims about the Bias-Variance Tradeoff when Teaching}
\label{sec:need-to-qualify-claims}

Students who take introductory machine learning courses are typically taught the bias-variance tradeoff as a general, unavoidable truth that applies anywhere there is some notion of increasing model complexity (see \cref{sec:textbooks} for its prevalence and representative quotes from textbooks). This leads machine learning experts to sometimes make incorrect inferences about model selection with high confidence.
From extensive personal communication, it appears that most researchers unfamiliar with work like \citet{DBLP:journals/corr/NeyshaburTS14}'s react with incredulity to the results described in \cref{sec:geman-refutation}, \cref{sec:rl-tradeoff}, and \cref{article:bias-variance}.
Even among researchers who have familiarized themselves with \citet{DBLP:journals/corr/NeyshaburTS14}'s results on test error, one can find many who are surprised by our results.
We attribute this phenomenon to the strong influence \citet{geman}'s claims have had on the research community.
In other words, a sizable portion of researchers can be dogmatic about the conventional tradeoff wisdom described in \cref{sec:model-complexity} and \cref{sec:bias-variance}. Qualifying the conventional tradeoff wisdom in textbooks and introductory courses by noting that the tradeoff intuition is useful sometimes and misleading other times would prevent future students from subscribing to this intuitive dogma.

The goal in amending textbooks/lectures that teach the bias-variance tradeoff is to make it more clear that the bias-variance tradeoff is not a universal truth. Here, we present three simple qualifications that, if integrated, would more accurately represent the evidence we have on the bias-variance tradeoff and would help prevent students from interpreting that the bias-variance tradeoff is universal:

\begin{enumerate}
    \item Expected error can be decomposed into (squared) bias and variance, when using squared loss \citep{geman}, but this \textbf{decomposition does not imply a tradeoff.} This lack of implication should be made explicit in textbooks because the decomposition is often used in close proximity to the tradeoff as ambiguous evidence for it.
    
    \item \textbf{The bias-variance tradeoff should not be assumed to be universal.} There is evidence that bias and variance trade off in certain methods (e.g.\ KNN) when varying the right parameter (\cref{sec:bv-experimental-evidence}), but there are also counterexamples. For example, there are clear examples of a lack of a bias-variance tradeoff in neural networks (\cref{article:bias-variance}) and potentially other methods such as decision trees \citep{belkin2018}.
    
    \item
    It should be emphasized that the \textbf{PAC upper bounds on test error can be very loose} for the problems we care about in practice (see \cref{sec:model-complexity,sec:bv-supporting-theory} for examples of upper bounds on test error and estimation error). Because these upper bounds are so loose, their qualitative trend (e.g. as number of parameters increases) is not necessarily an accurate reflection of the qualitative trend of the test error in practice. 
\end{enumerate}

                     
\anglais
\articleenchapitre

\revue{the ICML 2019 Workshop on Identifying and Understanding Deep Learning Phenomena}
\article{A Modern Take on the Bias-Variance Tradeoff in Neural Networks}
\label{article:bias-variance}

 %
 %
 %
\contributions
{
    \begin{itemize}
        \item Lead the project
        \item Hypothesis that the bias-variance tradeoff would not be seen in neural networks when varying width
        \item The whole initial codebase
        \item Ran many of the experiments
        \item \cref{prop:linear_model}
        \item Majority of paper writing
    \end{itemize}

    Ioannis Mitliagkas proposed the variance decomposition and supervised the project.
    
    Sarthak Mittal ran many experiments.
    
    Vinayak Tantia contributed significantly to the codebase and ran experiments.
    
    Aristide Baratin contributed the initialization term to \cref{prop:linear_model}.
    
    Aristide Baratin and Ioannis Mitliagkas proved related versions of \cref{thm:init-var-decay}
    
    Ioannis Mitliagkas and Aristide Baratin contributed significantly to the writing of the paper.\\[1cm]
}
 %

\auteur{Brady Neal}
\auteur{Sarthak Mittal}
\auteur{Aristide Baratin}
\auteur{Vinayak Tantia}
\auteur{Matthew Scicluna}
\auteur{Simon Lacoste-Julien}
\auteur{Ioannis Mitliagkas}
\adresse{Mila - Quebec AI Institute}


\maketitle

\begin{abstract}{bias-variance tradeoff, neural networks, over-parameterization, generalization}
The  bias-variance tradeoff tells us that as model complexity increases, bias falls and variances increases, leading to a U-shaped test error curve. However, recent empirical results with over-parameterized neural networks are marked by a striking absence of the classic U-shaped test error curve: test error keeps decreasing in wider networks. This suggests that there might not be a bias-variance tradeoff in neural networks with respect to network width, unlike was originally claimed by, e.g., \citet{geman}. Motivated by the shaky evidence used to support this claim in neural networks, we measure bias and variance in the modern setting. We find that \emph{both} bias \emph{and} variance can decrease as the number of parameters grows. To better understand this, we introduce a new decomposition of the variance to disentangle the effects of optimization and data sampling. We also provide theoretical analysis in a simplified setting that is consistent with our empirical findings.
\end{abstract}

\begin{resume}{compromis biais-variance, réseaux de neurones, sur-paramétrage, généralisation}
Le compromis biais-variance nous indique qu'à mesure que la complexité du modèle augmente, le biais diminue et la variance augmente, ce qui conduit à une courbe d'erreur de test en forme de U. Cependant, les résultats empiriques récents avec des réseaux neuronaux sur-paramétrés sont marqués par une absence frappante de la courbe d'erreur de test classique en forme de U : l'erreur de test continue de diminuer dans les réseaux plus larges. Cela donne à penser qu'il n'y a peut-être pas de compromis sur la variance de biais dans les réseaux de neurones en ce qui concerne la largeur du réseau, contrairement à ce que prétendaient, à l'origine, par exemple, \citet{geman}. Motivés par les preuves incertaines utilisées à l'appui de cette affirmation dans les réseaux de neurones, nous mesurons les biais et la variance dans le contexte moderne. Nous constatons que le biais d'accentuation et la variance peuvent diminuer à mesure que le nombre de paramètres augmente. Pour mieux comprendre cela, nous introduisons une nouvelle décomposition de la variance pour démêler les effets de l'optimisation et de l'échantillonnage des données. Nous fournissons également une analyse théorique dans un cadre simplifié qui est conforme à nos constatations empiriques.
\end{resume}



\section{Introduction}

There is a dominant dogma in machine learning:
\begin{quote}
    The price to pay for achieving low bias is high variance \citep{geman}.
\end{quote}
The quantities of interest here are the bias and variance of a learned model's {\em prediction} on a new input, where the randomness comes from the sampling of the training data.
This idea that bias decreases while variance increases with model capacity, leading to a U-shaped test error curve is commonly known as the \emph{bias-variance tradeoff} (\cref{fig:main_common_intuition_wrong} (left)).

There exist experimental evidence and theory that support the idea of a tradeoff. In their landmark paper, \citet{geman} measure bias and variance in various models.
They show convincing experimental evidence for the bias-variance tradeoff in nonparametric methods such as kNN (k-nearest neighbor) and kernel regression. They also show experiments on neural networks and claim that bias decreases and variance increases with network width. Statistical learning theory \citep{vapnik1998statistical} successfully predicts these U-shaped test error curves implied by a tradeoff for a number of classic machine learning models.
A key element is identifying a notion of model capacity, understood as the main parameter controlling this tradeoff.



Surprisingly, there is a growing amount of empirical evidence that \textit{wider} networks generalize \textit{better} than their smaller counterparts \citep{DBLP:journals/corr/NeyshaburTS14,wide_resnet,novak2018sensitivity, lee2018deep,belkin2018, jamming,fisher-rao_metric,DBLP:journals/corr/CanzianiPC16}.
 In those cases the classic U-shaped test error curve is not observed. 

A number of different research directions have spawned in response to these findings. \citet{DBLP:journals/corr/NeyshaburTS14} hypothesize the existence of an implicit regularization mechanism.
Some study the role that optimization plays \citep{implicit_bias__gd_linear_sep, implicit_bias_opt_geo}. Others suggest new measures of capacity \citep{fisher-rao_metric,neyshabur2018the}.
All approaches focus on test error, rather than studying bias and variance directly \citep{neyshabur2018the,scaling,fisher-rao_metric,belkin2018}.

\begin{figure}[t]
    \centering
    \begin{subfigure}
        \centering
        \includegraphics[width=0.47\textwidth]{figures/fortmann-roe_biasvariance_big}
    \end{subfigure}
    \hfill
    \begin{subfigure}
        \centering
         \includegraphics[width=0.47\textwidth]{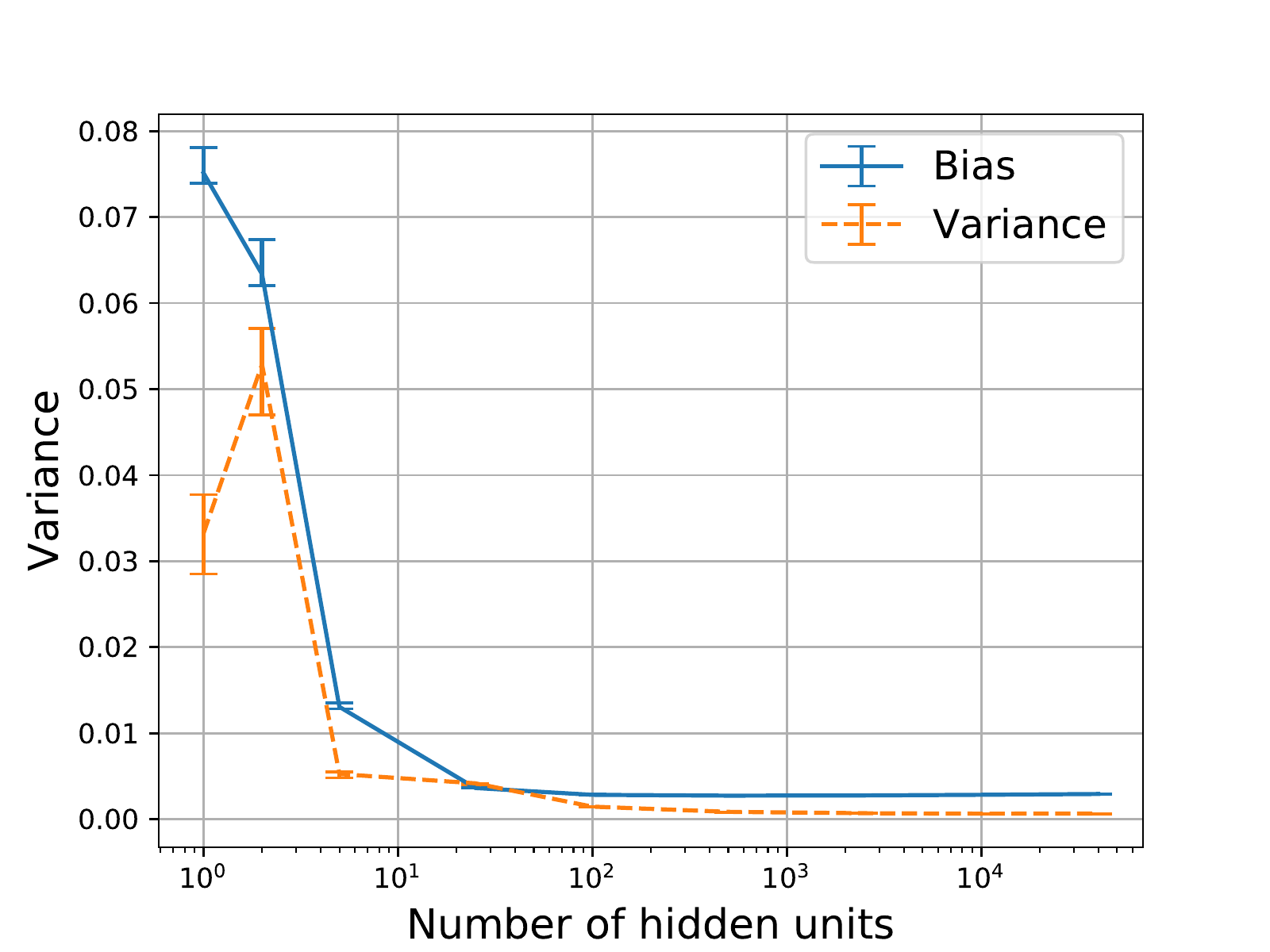}
    \end{subfigure}
    \caption[Traditional bias-variance tradeoff vs.\ decreasing variance in neural networks]{On the left is an illustration of the common intuition for the bias-variance tradeoff \citep{fortmann-roe_2012}. We find that {\em both} bias and variance decrease when we increase network width on MNIST (right) and other datasets (\cref{sec:width}). These results seem to contradict the traditional intuition of a strict tradeoff.}
    \label{fig:main_common_intuition_wrong}
\end{figure}

Test error analysis does not give a definitive answer on the lack of a bias-variance tradeoff. 
Consider boosting:
it is known that its test error often decreases with the number of rounds \cite[Figures 8-10]{Schapire1999}.
In spite of this monotonicity in test error, \citet{buhlmann2003boosting}
show that variance grows at an exponentially decaying rate, calling this an ``exponential bias-variance tradeoff'' (see \cref{sec:boosting}).
To study the bias-variance tradeoff, one has to isolate and measure bias and variance individually.
To the best of our knowledge, there has not been published work reporting such measurements on neural networks since 
\citet{geman}.




We go back to basics and study bias and variance.
We start by taking a closer look at \citet[Figure 16 and Figure 8 (top)]{geman}'s experiments with neural networks.
We notice that their experiments do not support their claim that ``bias falls and variance increases with the number of hidden units.''
The authors attribute this inconsistency to convergence issues and maintain their claim that the bias-variance tradeoff is universal.
Motivated by this inconsistency, we perform a set of bias-variance experiments with modern neural networks.

We measure prediction bias and variance of fully connected neural networks.
These measurements allow us to reason directly about whether there exists a tradeoff with respect to network width.
We find evidence that \emph{both} bias \emph{and} variance can decrease at the same time as network width increases in common classification and regression settings (\cref{fig:main_common_intuition_wrong,sec:width}).

We observe the qualitative lack of a bias-variance tradeoff in network width with a number of gradient-based optimizers. 
In order to take a closer look at the roles of optimization and data sampling, we propose a simple decomposition of total prediction variance (\cref{sec:decomposition}). 
We use the law of total variance to get a term that corresponds to average (over data samplings) variance due to optimization and a term that corresponds to variance due to training set sampling of an ensemble of differently initialized networks. 
Variance due to optimization is significant in the under-parameterized regime and monotonically decreases with width in the over-parameterized regime. 
There, total variance is much lower and dominated by variance due to sampling (\cref{fig:all_variances}).

We provide theoretical analysis, consistent with our empirical findings,
in simplified analysis settings:
i) prediction variance does not grow arbitrarily with number of parameters in fixed-design linear models;
ii) variance due to optimization diminishes with number of parameters in neural networks under strong assumptions.

\begin{figure}[t]
    \centering
    \begin{subfigure}
        \centering
         \includegraphics[width=0.47\textwidth]{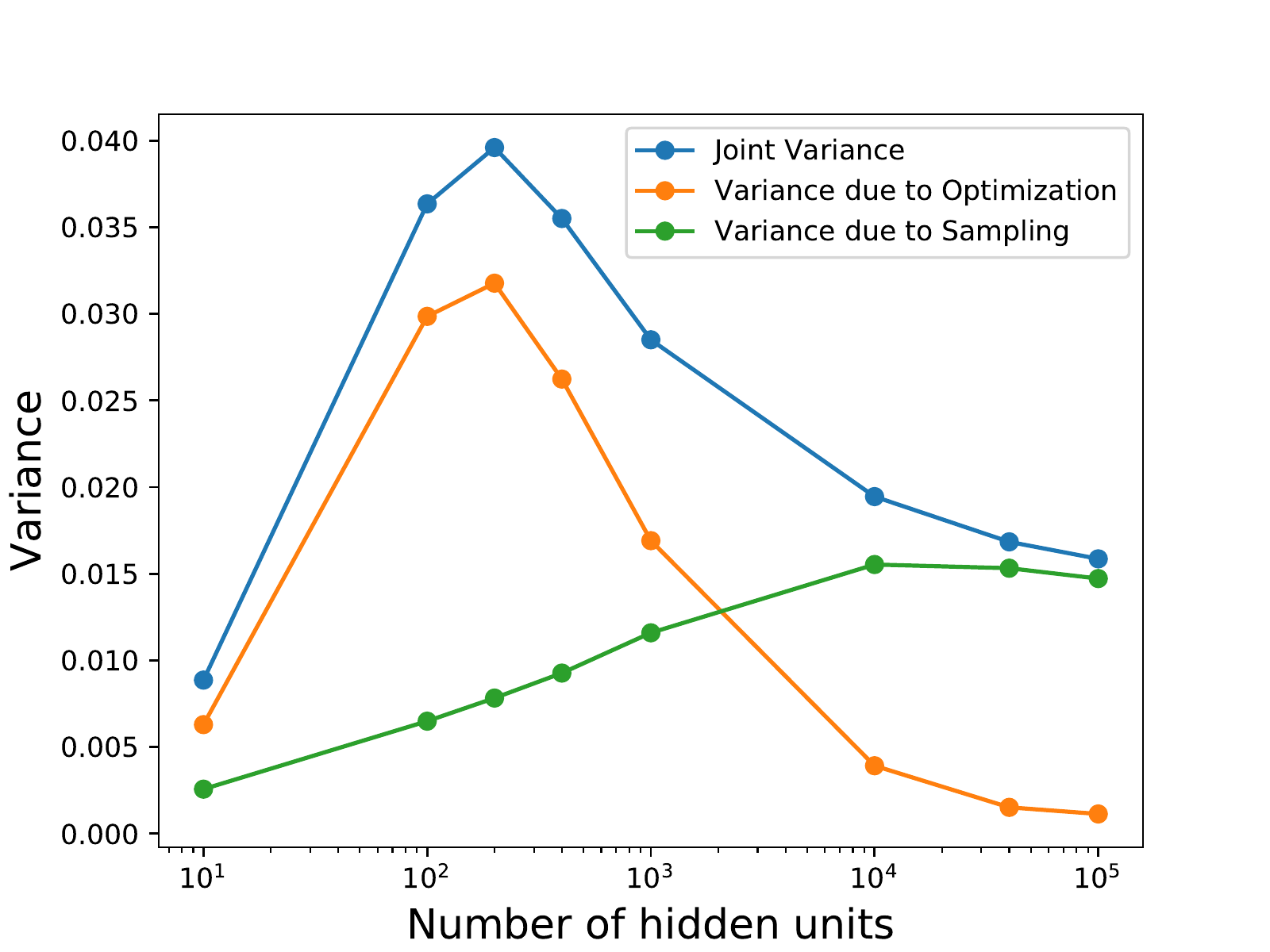}
    \end{subfigure}
    \hspace{.5cm}
    \begin{subfigure}
        \centering
     \includegraphics[width=0.47\textwidth]{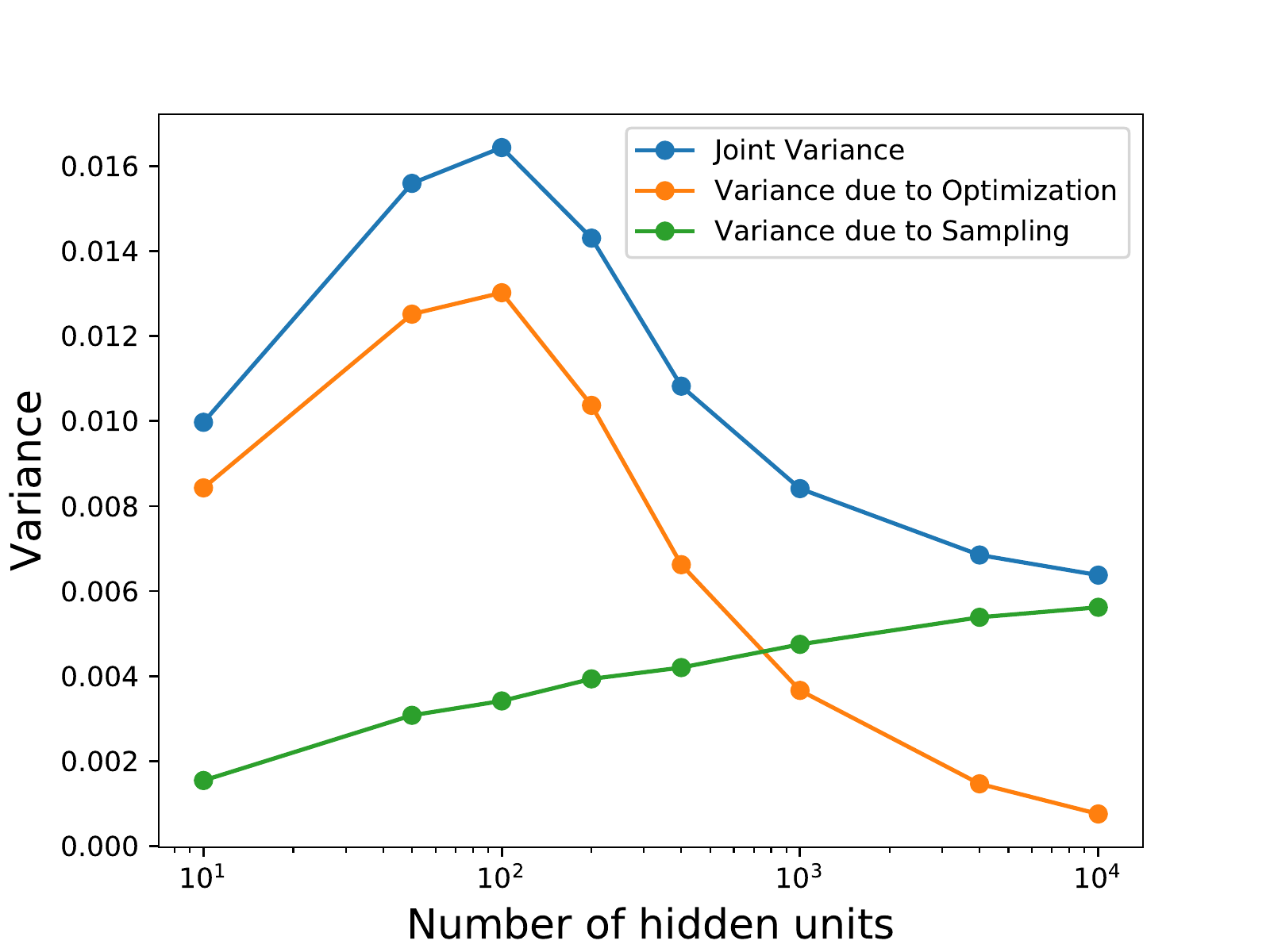}
    \end{subfigure}
    \caption[Variance due to optimization and variance due to sampling]{Trends of variance due to sampling and variance due to optimization with width on CIFAR10 (left) and on SVHN (right).
    Variance due to optimization decreases with width, once in the over-parameterized setting.
    Variance due to sampling plateaus and remains constant. This is in contrast with what the bias-variance tradeoff would suggest.}
    \label{fig:all_variances}
\end{figure}





    


\paragraph{Organization}
The rest of this paper is organized as follows. We discuss relevant related work in \cref{sec:related-work}. \Cref{sec:preliminaries} establishes necessary preliminaries, including our variance decomposition. In \cref{sec:width}, we empirically study the impact of network width on variance. In \cref{sec:theory}, we present theoretical analysis in support of our findings.

\section{Related work}
\label{sec:related-work}

\citet{DBLP:journals/corr/NeyshaburTS14,neyshaburthesis} point out that because increasing network width does not lead to a U-shaped test error curve, there must be some form of implicit regularization controlling capacity. Our work is consistent with this finding, but by approaching the problem from the bias-variance perspective, we gain additional insights: 1) We specifically address the hypothesis that decreased bias must come at the expense of increased variance (see \citet{geman} and \cref{app:intuitions}) by measuring both quantities.
2) Our more fine-grain approach reveals that variance due to optimization vanishes with width, while variances due to sampling increases and levels off.
This insight about variance due to sampling is consistent with existing variance results for boosting \citep{buhlmann2003boosting}.
To ensure that we are studying networks of increasing capacity, one of the experimental controls we use throughout the paper is to verify that bias is decreasing.

In independent concurrent work, \citet{jamming, belkin2018} point out that generalization error acts according to conventional wisdom in the under-parameterized setting, that it decreases with capacity in the over-parameterized setting, and that there is a sharp transition between the two settings. Although the phrase ``bias-variance trade-off'' appears in \citet{belkin2018}'s title, their work really focuses on the shape of the test error curve: they argue it is not the simple U-shaped curve that conventional wisdom would suggest, and it is not the decreasing curve that \citet{DBLP:journals/corr/NeyshaburTS14} found; it is ``double descent curve,'' which is essentially a concatenation of the two curves. This is in contrast to our work, where we actually measure bias, variance, and components of variance in this over-parameterized regime. Interestingly, \citet{belkin2018}'s empirical study of test error provides some evidence that our bias-variance finding might not be unique to neural networks and might be found in other models such as decision trees.

In subsequent work,\footnote{By ``subsequent work,'' we mean work that appeared on arXiv five months after our paper appeared on arXiv.} \citet{belkin2019models,hastie2019surprises} perform a theoretical analysis of student-teacher linear models (with random features), showing the double descent curve theoretically. \citet{Advani2017HighdimensionalDO} also performed a similar analysis. \citet{hastie2019surprises} is the only one to theoretically analyze variance. Their work differs from ours in that we run experiments with neural networks on complex, real data, while they carry out a theoretical analysis of linear models in a simplified teacher (data generating distribution) setting.




\section{Preliminaries}
\label{sec:preliminaries}

\subsection{Set-up}
\label{sec:setup}

We consider the typical supervised learning task of predicting  an output $y \in \mathcal{Y}$ from an input $x \in \mathcal{X}$, where the pairs $(x, y)$ are drawn from some unknown joint distribution, $\mathcal{D}$. The learning problem consists of learning a function $h_S:  \mathcal{X} \to \mathcal{Y}$ from a finite training dataset $S$ of $m$ i.i.d.\  samples from $\mathcal{D}$. The quality of a predictor $h$ can quantified by the expected error,
\beq \label{populrisk}
\mathcal{E}(h) = \E_{(x, y) \sim \mathcal{D}} \, \ell(h(x), y) \, ,
\eeq
for some loss function $\ell : \mathcal{Y} \times \mathcal{Y} \to \R$. 

In this paper, predictors $h_{\theta}$ are parameterized by the weights $\theta \in \R^N$ of neural networks.
We consider
the average performance over possible training sets (denoted by the random variable $S$) of size $m$. This is the same quantity \citet{geman} consider. While $S$ is the only random quantity studied in the traditional bias-variance decomposition, we also study randomness coming from optimization. We denote the random variable for optimization randomness (e.g.\ initialization) by $O$.

Formally, given a fixed training set $S$ and fixed optimization randomness $O$, the learning algorithm $\A$ produces $\theta$ = $\A(S, O)$. Randomness in optimization translates to randomness in $\A(S, \cdot)$. Given a fixed training set, we encode the randomness due to $O$ in a conditional distribution $p(\theta |S)$.
Marginalizing over the training set $S$ of size $m$ gives a marginal distribution $p(\theta)=\E_S p(\theta | S) $ on the weights learned by $\mathcal{A}$ from $m$ samples. In this context, the average performance of the learning algorithm using training sets of size $m$ can be expressed in the following ways:
\beq \label{fullrisk}
\mathcal{R}_m = \E_{\theta \sim p} \mathcal{E}(h_\theta) = \E_S \E_{\theta\sim p(\cdot |S)} \mathcal{E}(h_\theta) = \E_S \E_O \mathcal{E}(h_\theta)
\eeq

\subsection{Bias-variance decomposition}
\label{Sec:bv}

We briefly recall the standard bias-variance decomposition in the case of squared-loss. We work in the context of classification, where each  class $k \in \{1\cdots K\}$ is represented by a one-hot vector in $\mathbb{R}^K$. The predictor outputs a score  or probability vector in $\R^K$. In this context, the risk in \cref{fullrisk}   decomposes into three sources of error \citep{geman}:
\beq  \label{bv}
\mathcal{R}_m = \enoise + \ebias + \evar 
\eeq  
The first term is an intrinsic error term independent of the predictor:
\begin{equation*}
\enoise = \E_{(x,y)} \left[\|y - \bar{y}(x)\|^2 \right] \, .
\end{equation*}
The second term is a bias term:
\begin{equation*}
\ebias = \E_{x} \left[\|\E_\theta [h_\theta(x)] - \bar{y}(x)\|^2 \right] \, ,
\end{equation*}
where  $\bar{y}(x)$ denotes the expectation $\E[y | x]$ of $y$ given $x$. The third term is the expected variance of the output predictions: 
\begin{equation*}
    \evar = \E_{x}  \Var(h_\theta(x)),
\end{equation*}
\begin{equation*}
    \Var(h_\theta(x)) =
    \E_\theta \left[ \|h_\theta(x) - \E_\theta[h_\theta(x)] \|^2\right],
 \end{equation*}
where the expectation over $\theta$ can be done as in \cref{fullrisk}. Interpreting this bias-variance decomposition as a bias-variance tradeoff is quite pervasive (see, e.g., \citet[Chapter 2.9]{hastie01statisticallearning}, \citet[5.4.4]{Goodfellow-et-al-2016}, \citet[Chapter 3.2]{Bishop:2006}). It is generally invoked to emphasize that the model selected should be of the complexity that achieves the optimal balance between bias and variance.

Note that risks computed with classification losses (e.g cross-entropy or 0-1 loss)  do not have such a clean bias-variance decomposition \citep{Domingos00aunified, James03varianceand}. However, it is natural to expect that bias and variance are useful indicators of the performance of models that are not assessed with squared error. In fact, we show the classification risk can be bounded as 4 times the regression risk in \cref{app:classification_regression_relation}. To empirically examine this connection, in all of our graphs that have ``test error'' or ``training error'' on some classification task, we plot the 0-1 classification error (see, e.g., \cref{fig:small_data_tuned_error}).

\subsection{Further decomposing variance into its sources}
\label{sec:decomposition}

In the set-up of \cref{sec:setup} 
the prediction is a random variable that depends on two sources of randomness:
the randomly drawn training set, $S$,
and any optimization randomness, $O$, encoded into the conditional $p(\cdot |S)$.
In certain regimes, one gets significantly different predictions when using a different initialization.
Similarly, the output of a learned predictor changes when we use a different training set.
How do we start disentangling variance caused by sampling from variance caused by optimization?  
There are few different ways; here we describe one of them.

Our goal is to measure prediction variance due to sampling,
while controlling for the effect of optimization randomness.
\begin{definition}[(Ensemble) Variance due to sampling]
We consider the variance of an ensemble of infinitely many predictors with different optimization randomness (e.g.\ random initializations):
$$
    \var_{S}\left(
        \E_O\left[
            h_\theta(x) | S
        \right]\right).
        $$
\end{definition}

A common practice to estimate variance due to optimization effects is to run multiple seeds on a fixed training set. 
\begin{definition}[(Mean) Variance due to optimization]
We consider the average (over training sets) variance over optimization randomness for a fixed training set: $$ \E_{S}\left[
        \var_O\left(
            h_\theta(x) | S
        \right)
    \right].
$$
\end{definition}

The law of total variance naturally decomposes variance into these very terms:
\begin{align}
    \label{eqn:total-variance}
    \Var(h_\theta(x))
    = 
    &\E_{S}\left[
        \var_O\left(
            h_\theta(x) | S
        \right)
    \right]
        +
    \var_{S}\left(
        \E_O\left[
            h_\theta(x) | S
        \right]
    \right) 
\end{align}
We use this decomposition to get a finer understanding of our observations (\cref{fig:all_variances}).


\begin{figure}[t]
    \centering
    \subfigure[Variance decreases with width, even in the small MNIST setting.]{
        \includegraphics[width=0.3115\textwidth]{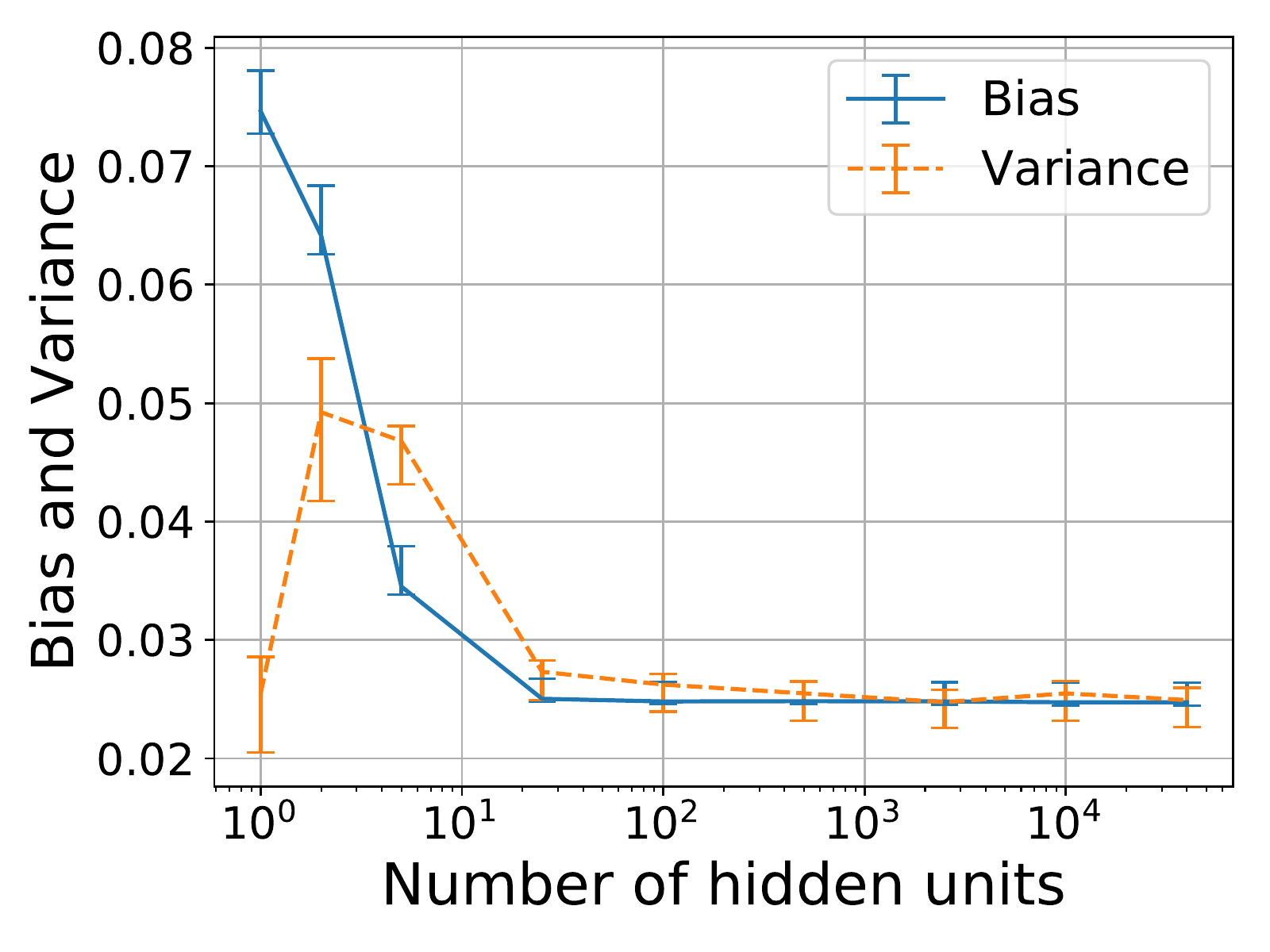}
        \label{fig:small_data_tuned_bv}
    }
    \hfill
    \subfigure[Test error trend is same as bias-variance trend (small MNIST).]{
        \includegraphics[width=0.3115\textwidth]{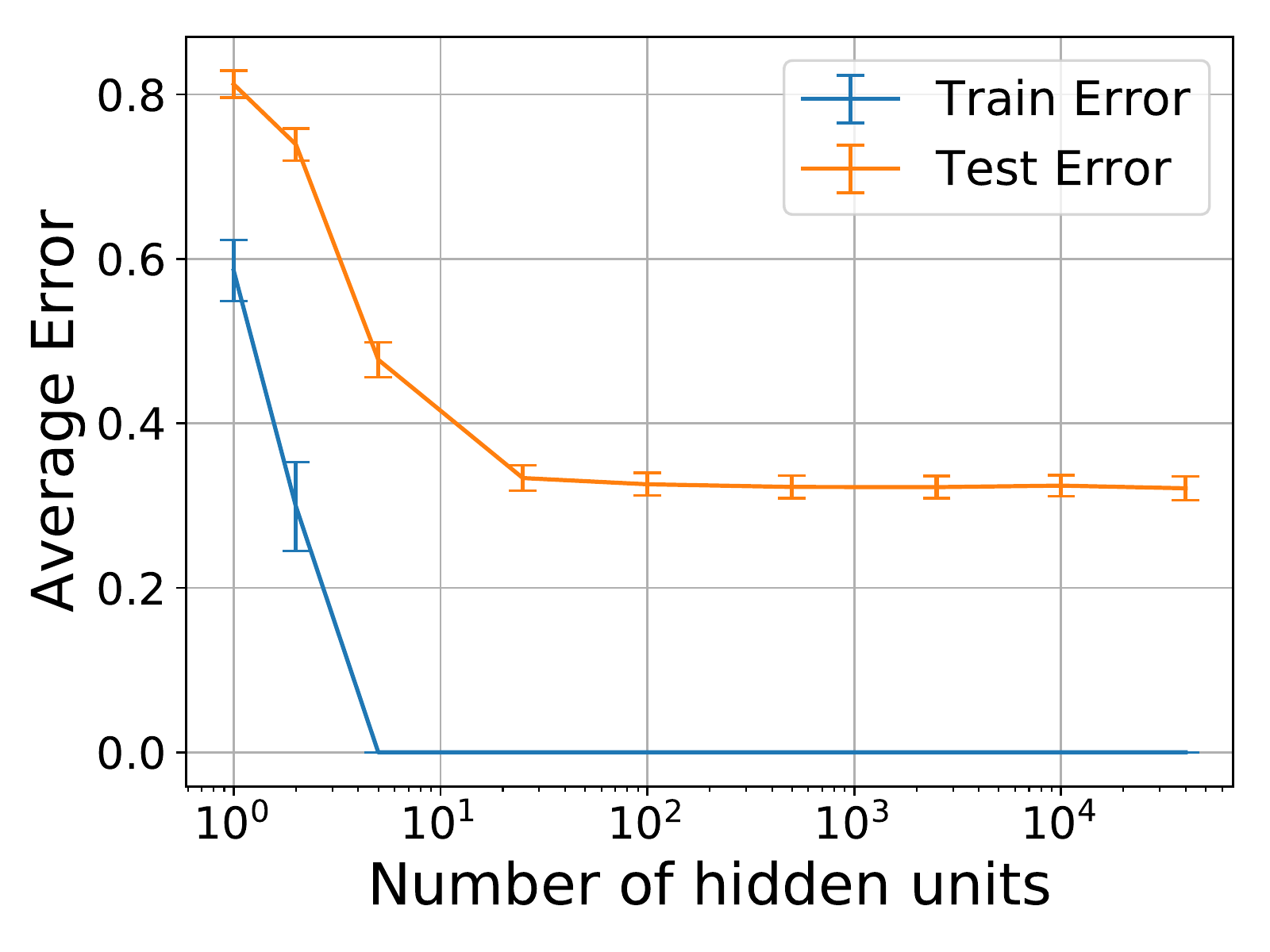}
        \label{fig:small_data_tuned_error}
    }
    \hfill
    \subfigure[Similar bias-variance trends on sinusoid regression task.]{
        \includegraphics[width=0.3115\textwidth]{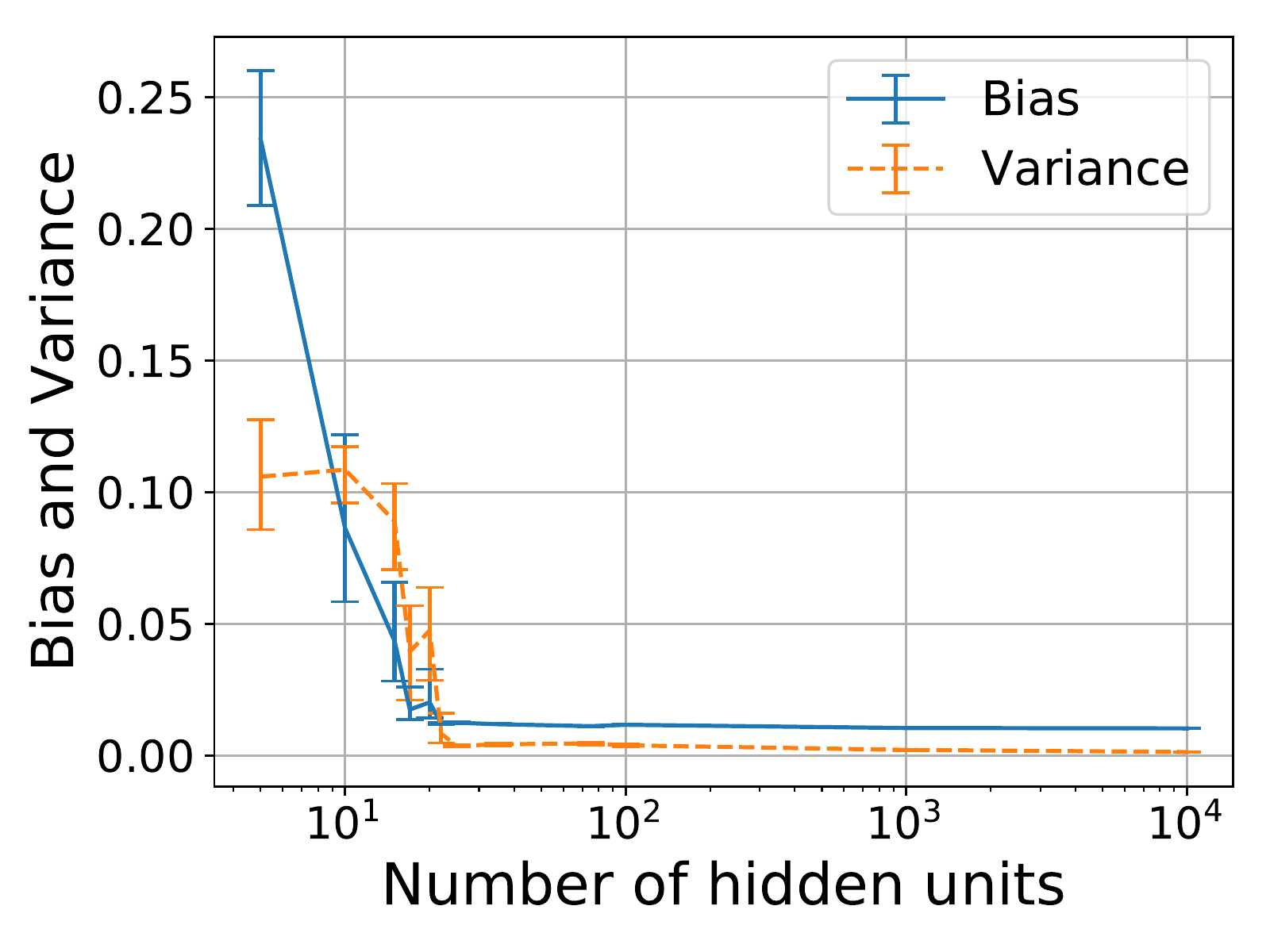}
         \label{fig:sinusoid_bv}
    }
    \caption[Small data bias-variance experiments]{We see the same bias-variance trends in small data settings: small MNIST (left) and a regression setting (right).}
\end{figure}

\section{Experiments}
\label{sec:width}
In this section, we study how variance of fully connected single hidden layer networks varies with width. We provide evidence against \citet{geman}'s important claim about neural networks: \begin{quote}
    ``The basic trend is what we expect: bias falls and variance increases with the number of hidden units.''
\end{quote}
Our main finding is that, for all tasks that we study, bias and variance both decrease as we scale network width.
We also provide a meaningful decomposition of prediction variance into a variance due to sampling term and a variance due to optimization term.

\subsection{Common experimental details} 

We run experiments on different datasets: MNIST, SVHN, CIFAR10, small MNIST, and a sinusoid regression task. Averages over data samples are performed by taking the training set $S$ and creating 50 bootstrap replicate training sets $S'$ by sampling with replacement from $S$. We train 50 different neural networks for each hidden layer size using these different training sets. Then, we estimate $\ebias$\footnote{Because we do not have access to $\bar{y}$, we use the labels $y$ to estimate $\ebias$. This is equivalent to assuming noiseless labels and is standard procedure for estimating bias \citep{Kohavi:1996, Domingos00aunified}.} and $\evar$ as in Section \ref{Sec:bv}, where the population expectation  $\E_x$ is estimated with an average over the test set. To estimate the two terms from the law of total variance (\autoref{eqn:total-variance}), we use 10 random seeds for the outer expectation and 10 for the inner expectation, resulting in a total of 100 neural networks for each hidden layer size. Furthermore, we compute 99\% confidence intervals for our bias and variance estimates using the bootstrap \citep{efron1979}.

The networks are initialized using PyTorch's default initialization, which scales the variance of the weight initialization distribution inversely proportional to the width \citep{LeCun:1998,xavier2010}. The networks are trained using SGD with momentum and generally run for long after 100\% training set accuracy is reached (e.g.\ 500 epochs for full data MNIST and 10000 epochs for small data MNIST). The overall trends we find are robust to how long the networks are trained after the training error converges. The step size hyperparameter is specified in each of the sections, and the momentum hyperparameter is always set to 0.9. To make our study as general as possible, we consider networks without regularization bells and whistles such as weight decay, dropout, or data augmentation, which \citet{zhang} found to not be necessary for good generalization.

\subsection{Decreasing variance in full data setting}

We find a clear decreasing trend in variance with width of the network in the full data MNIST setting (\cref{fig:main_common_intuition_wrong}). We also see the same trend with CIFAR10 (\cref{app:CIFAR10_width}) and SVHN (\cref{app:SVHN_width}). In these experiments, the same step size is used for all networks for a given dataset (0.1 for MNIST and 0.005 for CIFAR10 and SVHN). The trend is the same with or without early stopping, so early stopping is not necessary to see decreasing variance, similar to how it was not necessary to see better test set performance with width in \citet{DBLP:journals/corr/NeyshaburTS14}. Wider ResNets are known to achieve lower test error \citep{wide_resnet}; this likely translates to decreasing variance with width in convolutional networks as well. Much of the over-parameterization literature focuses on over-parameterization in width; interestingly, the variance trend is not the same when varying depth (\cref{app:depth}).

\begin{figure}[t]
    \centering
    \subfigure{
        \includegraphics[width=0.3115\textwidth]{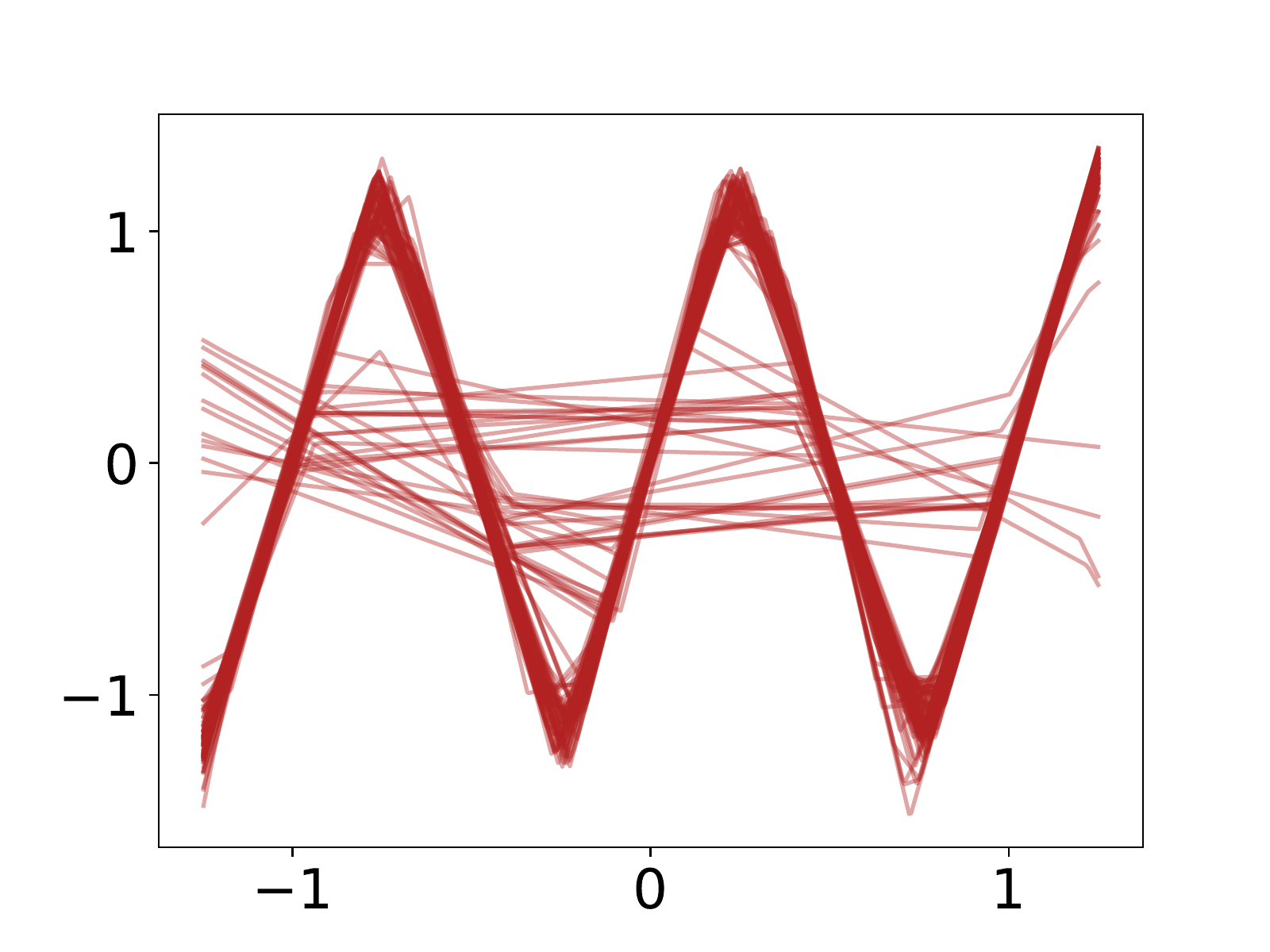}
    }
    \hfill
    \subfigure{
        \includegraphics[width=0.3115\textwidth]{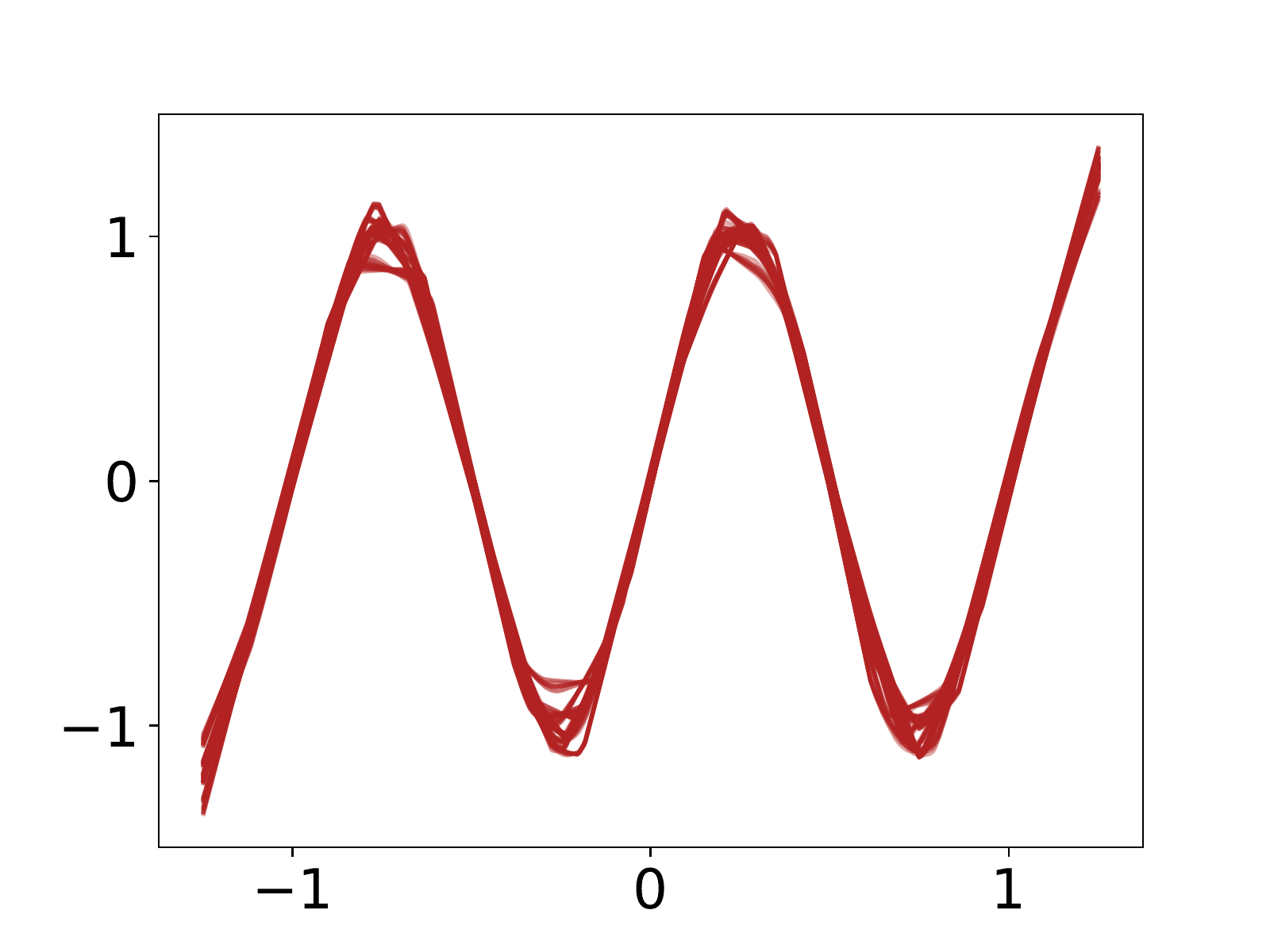}
    }
    \hfill
    \subfigure{
        \includegraphics[width=0.3115\textwidth]{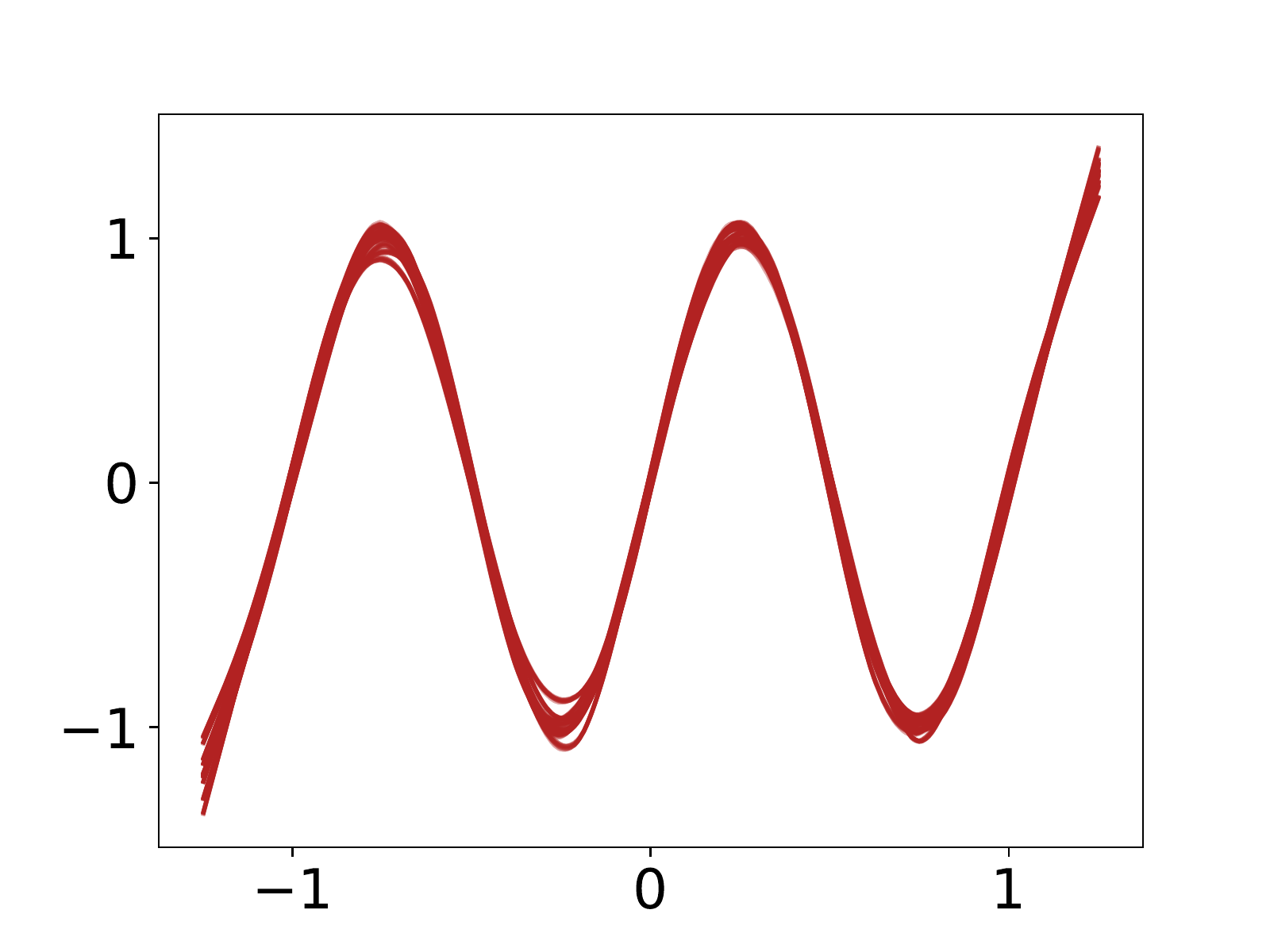}
    }
    \caption[Visualization of low variance functions learned by large networks]{Visualization of the 100 different learned functions of single hidden layer neural networks of widths 15, 1000, and 10000 (from left to right) on the task of learning a sinusoid. The learned functions are increasingly similar with width, suggesting decreasing variance.
    More in \cref{app:sinusoid_regression}.}
    \label{fig:sinusoid}
\end{figure}

\subsection{Testing the limits: decreasing variance in the small data setting}
\label{sec:small_data}

Decreasing the size of the dataset can only increase variance. To study the robustness of the above observation, we decrease the size of the training set to just 100 examples. In this small data setting, somewhat surprisingly, we still see that \emph{both} bias \emph{and} variance decrease with width (\cref{fig:small_data_tuned_bv}). The test error behaves similarly (\cref{fig:small_data_tuned_error}). Because performance is more sensitive to step size in the small data setting, the step size for each network size is tuned using a validation set (see \cref{app:tuned_lr} for step sizes). The training for tuning is stopped after 1000 epochs, whereas the training for the final models is stopped after 10000 epochs.  Note that because we see decreasing bias with width, effective capacity is, indeed, increasing while variance is decreasing.

One control that motivates the experimental design choice of optimal step size is that it leads to the conventional decreasing bias trend (\cref{fig:small_data_tuned_bv}) that indicates increasing effective capacity. In fact, in the corresponding experiment where step size is the same 0.01 for all network sizes, we do not see monotonically decreasing bias (\cref{app:fixed_lr}). 

This sensitivity to step size in the small data setting is evidence that we are testing the limits of our hypothesis. By looking at the small data setting, we are able to test our hypothesis when the ratio of size of network to dataset size is quite large, and we still find this decreasing trend in variance (\cref{fig:small_data_tuned_bv}). 

To see how dependent this phenomenon is on SGD, we also run these experiments using batch gradient descent and PyTorch's version of LBFGS. Interestingly, we find a decreasing variance trend with those optimizers as well. These experiments are included in \cref{app:other_optimizers}.

\subsection{Decoupling variance due to sampling from variance due to optimization}
\label{sec:width_var_decoupling}

In order to better understand this variance phenomenon in neural networks, we separate the variance due to sampling from the variance due to optimization, according to the law of total variance (\autoref{eqn:total-variance}). Contrary to what traditional bias-variance tradeoff intuition would suggest, we find variance due to sampling increases slowly and levels off, once sufficiently over-parameterized (\cref{fig:all_variances}). Furthermore, we find that variance due to optimization decreases with width, causing the total variance to decrease with width (\cref{fig:all_variances}).

A body of recent work has provided evidence that over-parameterization (in width) helps gradient descent optimize to global minima in neural networks \citep{globalminima_iclr2019, pmlr-v80-du18a, DBLP:journals/corr/SoltanolkotabiJ17, NIPS2014_5267, DBLP:journals/corr/abs-1801-02254}. Always reaching a global minimum implies low variance due to optimization on the \textit{training set}. Our observation of decreasing variance on the \textit{test set} shows that the over-parameterization (in width) effect on optimization seems to extend to generalization, on the data sets we consider.

\subsection{Visualization with regression on sinusoid}

We trained different width neural networks on a noisy sinusoidal distribution with 80 independent training examples. This sinusoid regression setting also exhibits the familiar bias-variance trends (\cref{fig:sinusoid_bv}) and trends of the two components of the variance and the test error (\cref{fig:sinusoid_curves_app} of \cref{app:sinusoid_regression}).

Because this setting is low-dimensional, we can visualize the learned functions. The classic caricature of high capacity models is that they fit the training data in a very erratic way (example in \cref{fig:high_var_caricature} of \cref{app:sinusoid_regression}). We find that wider networks learn sinusoidal functions that are much more similar than the functions learned by their narrower counterparts (\cref{fig:sinusoid}). We have analogous plots for all of the other widths and ones that visualize the variance similar to how it is commonly visualized for Gaussian processes in \cref{app:sinusoid_regression}.

\section{Discussion and theoretical insights}
\label{sec:theory}

Our empirical results demonstrate that in the practical setting, variance due to optimization decreases with network width while variance due to sampling increases slowly and levels off once sufficiently over-parameterized. In \cref{sec:linear_models}, we discuss the simple case of linear models and point out that non-increasing variance can already be seen in the over-parameterized setting. In \cref{sec:back-to-nn} we take inspiration from linear models to provide arguments for the behavior of variance in increasingly wide neural networks, and we discuss the assumptions we make.

\subsection{Insights from linear models}
\label{sec:linear_models}

In this section, we review the classic result that the variance of a linear model grows with the number of parameters \citep[Section 7.3]{hastie_09} and point out that variance behaves differently in the over-parameterized setting.

We consider least-squares linear regression in a standard setting which assumes a noisy linear mapping $y = \theta^T x + \epsilon$ between input feature vectors $x\in \R^N$ and real outputs, where $\epsilon$ denotes the noise random variable with $\E[\epsilon] = 0$ and $\Var(\epsilon) = \sigma_{\epsilon}^2$. In this context, the over-parameterized setting is when the dimension $N$ of the input space is larger than the number $m$ of examples.

Let $X$ denote the $m \times N$ design matrix whose $i$\textsuperscript{th} row is the training point $x_i^T$, let $Y$ denote the corresponding labels, and let $\Sigma = X^T X$ denote the empirical covariance matrix.  
We consider the fixed-design setting where $X$ is fixed, so all of the randomness due to data sampling comes solely from $\epsilon$. $\A$ learns weights $\hat{\theta}$ from $(X, Y)$, either by a closed-form solution or by gradient descent, using a standard initialization $\theta_0 \sim \mathcal{N}(0, \frac{1}{N} I)$. The predictor makes a prediction on $x \sim \D$: $h(x) = \hat{\theta}^T x$. Then, the quantity we care about is $\E_x \Var(h(x))$.

\subsubsection{Under-parameterized setting}

The case where $N \leq m$ is standard: if  $X$ has maximal rank, $\Sigma$ is invertible; the solution is independent of the initialization and given by $\hat{\theta} = \Sigma^{-1} X^T Y$. All of the variance is a result of randomness in the noise $\epsilon$.  For a fixed $x$,
\begin{equation}
    \Var(h(x)) = \sigma_\epsilon^2 \Tr(x x^T \Sigma^{-1}) \, .
\end{equation}
This grows with the number of parameters $N$. For example, taking the expected value over the empirical distribution, $\hat{p}$, of the sample, we recover that the variance grows with $N$:
\begin{equation}
    \E_{x \sim \hat{p}} [\Var(h(x))] = \frac{N}{m} \sigma_\epsilon^2 \, .
\end{equation}
We provide a reproduction of the proofs in \cref{app:linear_underparam}.

\subsubsection{Over-parameterized setting}

The over-parameterized case where $N > m$ is more interesting: even if $X$ has maximal rank,  $\Sigma$ is not invertible. This leads to a subspace of solutions, but gradient descent yields a unique solution from updates that belong to the span of the training points $x_i$ (row space of $X$) \citep{lecun1991}, which is of dimension $r = \rank(X) = \rank(\Sigma)$. Correspondingly, no learning occurs in the null space of $X$, which is of dimension $N - r$. 
Therefore, gradient descent yields the solution that is closest to initialization: $\hat{\theta} = P_\perp(\theta_0) + \Sigma^+ X^T Y$, where $P_{\perp}$ projects onto the null space of $X$ and $+$ denotes the Moore-Penrose inverse. 

The variance has two contributions: one due to initialization and one due to sampling (here, the noise $\epsilon$), as in \cref{eqn:total-variance}. These are made explicit in \cref{prop:linear_model}.  
\begin{prop}[Variance in over-parameterized linear models]
\label{prop:linear_model}
Consider the over-parameterized setting where $N > m$.  For a fixed $x$, the variance decomposition of \cref{eqn:total-variance} yields
\beq \label{eq:over_lin_variance}
\Var(h(x))=  \frac{1}{N}\|P_\perp(x)\|^2 + \sigma_\epsilon^2 \Tr(x x^T \Sigma^+) \, . \eeq 
\end{prop}
This does not grow with the number of parameters $N$. In fact,  because $\Sigma^{-1}$ is replaced with $\Sigma^+$, the variance \emph{scales as the dimension of 
the data} (i.e the rank of $X$), as opposed to the number of parameters. For example, taking the expected value over the empirical distribution, $\hat{p}$, of the sample, we obtain
\begin{equation}
    \E_{x \sim \hat{p}} [\Var(h(x))] = \frac{r}{m} \sigma_\epsilon^2 \, ,
\end{equation}
where $r = \rank(X)$.
We provide the proofs for over-parameterized linear models in \cref{app:linear_overparam}.

\subsection{A more general result} 

\label{sec:back-to-nn}

We will illustrate our arguments in the following simplified setting, where $\mathcal{M}$, $\mathcal{M}^\perp$, and $d(N)$ are the more general analogs of $\rowspace(X)$, $\nullspace(X)$, and $r$ (respectively):

\textbf{Setting. }
Let $N$ be the dimension of the parameter space. 
The prediction for a fixed example $x$, given by a trained network parameterized by $\theta$ depends on:

(i) a subspace of the parameter space, $\mathcal{M} \in \mathbb{R}^N$ with relatively small dimension, $d(N)$,  which depends only on the learning task. 

(ii) parameter components corresponding to directions orthogonal to $\mathcal{M}$. The orthogonal $\mathcal{M}^\perp$ of $\mathcal{M}$ has dimension, $N-d(N)$, and is essentially irrelevant to the learning task.

We can write the parameter vector as a sum of these two components $    \theta = \theta_\mathcal{M} + \theta_{\mathcal{M}^\perp}
$. We will further make the following assumptions.
\vspace{-0.1cm}
\begin{enumerate}[label=\textbf{Assumption \arabic*}, labelindent=0pt, wide, labelwidth=!]
\item
\label{assum:opt_invariant}
The optimization of the loss function is invariant with respect to $\theta_{\mathcal{M}\perp}$.
\item
Regardless of initialization, the optimization method consistently yields a solution with the same $\theta_\mathcal{M}$ component (i.e.\ the same vector when projected onto $\mathcal{M}$).
\end{enumerate}

\subsubsection{Variance due to initialization}

\label{sec:variance-from-optimization}

Given the above assumptions, the following result shows that the variance from initialization\footnote{Among the different sources of optimization randomness, we focus on randomness from initialization and do not focus on randomness from stochastic mini-batching because we found the phenomenon of decreasing variance with width persists when using \textit{batch} gradient descent (\cref{sec:small_data}, \cref{app:other_optimizers}).} vanishes as we increase $N$. The full proof, which builds on concentration results for Gaussians (based on Levy's lemma \citep{ledoux2001concentration}), is given in \cref{app:more_general_setting}.

\newcommand{\thmp}{\theta_{\mathcal{M}^\perp}}

\begin{theorem}[Decay of variance due to initialization] \label{thm:init-var-decay}
Consider the setting of Section~\ref{sec:back-to-nn} 
Let $\theta$ denote the parameters at the end of the learning process.
Then, for a fixed data set and parameters  initialized as $\theta_0 \sim 
    \mathcal{N}(0, \frac{1}{N} I)$, the variance of the prediction satisfies the inequality, 
 \beq 
  \mbox{Var}_{\theta_0}(h_\theta(x)) \leq C \frac{2 L^2}{N}
 \eeq
where $L$ is the Lipschitz constant of the prediction with respect to $\theta$, and for some universal constant $C >O$. 
\end{theorem} 
This result guarantees that the variance decreases to zero as $N$ increases, provided 
the Lipschitz constant $L$ grows more slowly than the square root of dimension,
$L=o(\sqrt{N})$.

\subsubsection{Variance due to sampling}

Under the above assumptions,  the parameters at the end of learning take the form $\theta = \theta_\mathcal{M}^* + \theta_{0 \mathcal{M}^\perp}$.  For fixed initialization, the only source of variance of the prediction is the randomness of   $\theta_\mathcal{M}^*$ on the learning manifold. The variance depends on the parameter dimensionality only through 
\mbox{$\dim \mathcal{M} = d(N)$}, and hence remains constant if $d(N)$ does (see \citet{Li18IntDim}'s ``intrinsic dimension'').

\paragraph{Discussion on assumptions}
\label{sec:discussion_deep_net_assumptions}

We made strong assumptions,  but there is some support for them in the literature. The existence of a subspace $\mathcal{M}_\perp$ in which no learning occurs was also conjectured by \citet{Advani2017HighdimensionalDO} and  shown to hold in linear neural networks under a simplifying assumption that decouples the dynamics of the weights in different layers. \citet{Li18IntDim} 
empirically showed the existence of a critical number $d(N) = d$ of relevant parameters for a given learning task, independent of the  size of the model. 
\citet{Sagun17} showed that the spectrum of the Hessian for over-parameterized networks splits into $(i)$ a bulk centered near zero and $(ii)$ a small number of large eigenvalues; and  \citet{Gur-Ari2018} recently gave evidence that the small subspace spanned by the Hessian's top eigenvectors is preserved over long periods of training. These results suggest that learning occurs mainly in a small number of directions. 


\section{Conclusion and future work}


We provide evidence against \citet{geman}'s claim that ``the price to pay for achieving low bias is high variance,'' finding that \emph{both} bias \emph{and} variance decrease with network width. \citet{geman}'s claim is found throughout machine learning and is meant to generally apply to all of machine learning (\cref{app:intuitions}), and it is correct in many cases (e.g.\ kNN, kernel regression, splines). Is this lack of a tradeoff specific to neural networks or is it present in other models as well such as decision trees?

We propose a new decomposition of the variance, finding variance due to sampling (analog of regular variance in simple settings) does not appear to be dependent on width, once sufficiently over-parameterized, and that variance due to optimization decreases with width. 
By taking inspiration from linear models, we perform a theoretical analysis of the variance that is consistent with our empirical observations.

We view future work that uses the bias-variance lens as promising. For example, a probabilistic notion of effective capacity of a model is natural when studying generalization through this lens (\cref{app:prob_capacity}).  We did not study how bias and variance change over the course of training; that would make an interesting direction for future work. We also see further theoretical treatment of variance as a fruitful direction for better understanding complexity and generalization abilities of neural networks.

\subsection*{Acknowledgments}

We thank Yoshua Bengio, Lechao Xiao, Aaron Courville, Sharan Vaswani, Roman Novak, Xavier Bouthillier, Stanislaw Jastrzebski,  Gaetan Marceau Caron, Rémi Le Priol, Guillaume Lajoie, and Joseph Cohen for helpful discussions. Additionally, we thank SigOpt for access to their professional hyperparameter tuning services. This research was partially supported by the NSERC Discovery Grant (RGPIN-2017-06936 and RGPIN-2019-06512), by a Google Focused Research Award, the FRQNT nouveaux chercheurs program (2019-NC-257943), a startup grant by IVADO and the Canada CIFAR AI chair program.
We thank NVIDIA for donating a DGX-1 computer used in this work.

\chapter{Conclusion and Discussion}
\label{sec:conclusion}

It is time that the bias-variance tradeoff sections of textbooks are updated. We reviewed the history of the bias-variance tradeoff (including evidence for it) and its prevalence in textbooks in \cref{sec:why-we-believe-tradeoff}. We refuted \citet{geman}'s influential claims in \cref{sec:geman-refutation} by referencing recent measurements of bias and variance in neural networks (\cref{article:bias-variance}, \citet{neal2018}). We covered the emerging alternative hypothesis that the test error (risk) actually follows a ``double descent'' curve (as opposed to a U-shaped curve) in \cref{sec:double-descent}. Finally, we suggested specific changes to the bias-variance tradeoff section of textbooks in \cref{sec:need-to-qualify-claims}.

The specific changes can be simple qualifications. For example, the bias-variance decomposition (\cref{sec:bias-variance-decomposition}) is often used as evidence for the bias-variance tradeoff. However, this is misleading, unless one assumes fixed risk, which is often not the case when increasing model complexity. Additionally, though there seems to be clear evidence for the bias-variance tradeoff in many nonparametric methods (\cref{sec:bv-experimental-evidence}), we should not generalize this to all learning algorithms and assume that the bias-variance tradeoff is universal. These points can be easily clarified in teaching by clearly pointing out that the bias-variance decomposition does not imply a tradeoff and that seeing a tradeoff in certain models does not mean that we see a tradeoff in other models.

It remains to be seen whether there is a general shape of the risk curve that we should expect for all models as we increase model complexity. For now, it appears that there is not: we see clear U-shaped curves in some nonparametric methods such as KNN, kernel regression, and splines (\cref{sec:bv-experimental-evidence}), and we see clear double descent curves in neural networks \citep{deepdoubledescent,belkin2018,jamming,Geiger2019TheJT,Advani2017HighdimensionalDO}, with preliminary evidence of double descent curves in random forests \citep{belkin2018}.





\bibliographystyle{custom_icml2019}
\bibliography{bibliography}




\appendix

\chapter{Probabilistic notion of effective capacity}
\label{app:prob_capacity}

The problem with classical complexity measures is that they do not take into account optimization and have no notion of what will actually be learned. \citet[Section 1]{arpit17} define a notion of an \emph{effective} hypothesis class to take into account what functions are possible to be learned by the learning algorithm.

However, this still has the problem of not taking into account what hypotheses are \emph{likely} to be learned. To take into account the probabilistic nature of learning, we define the $\epsilon$-\textit{hypothesis class} for a data distribution $\D$ and learning algorithm $\A$, that contains the hypotheses which are at least $\epsilon$-likely for some $\epsilon > 0$:
\begin{equation}
    \mathcal{H}_{\D}(\A) = \{h : p(h(\A, S)) \geq \epsilon\},
    \label{eqn:def-eps-hypothesis-class}
\end{equation}

where $S$ is a training set drawn from $\D^m$, $h(\A, S)$ is a random variable drawn from the distribution over learned functions induced by $\D$ and the randomness in $\A$; $p$ is the corresponding density. Thinking about a model's $\epsilon$-hypothesis class can lead to drastically different intuitions for the complexity of a model and its variance (\cref{fig:var_spectrum}). This is at the core of the intuition for why the traditional view of bias-variance as a tradeoff does not hold in all cases.

\begin{figure}[h]
 \centering
 \includegraphics[width=.9\textwidth]{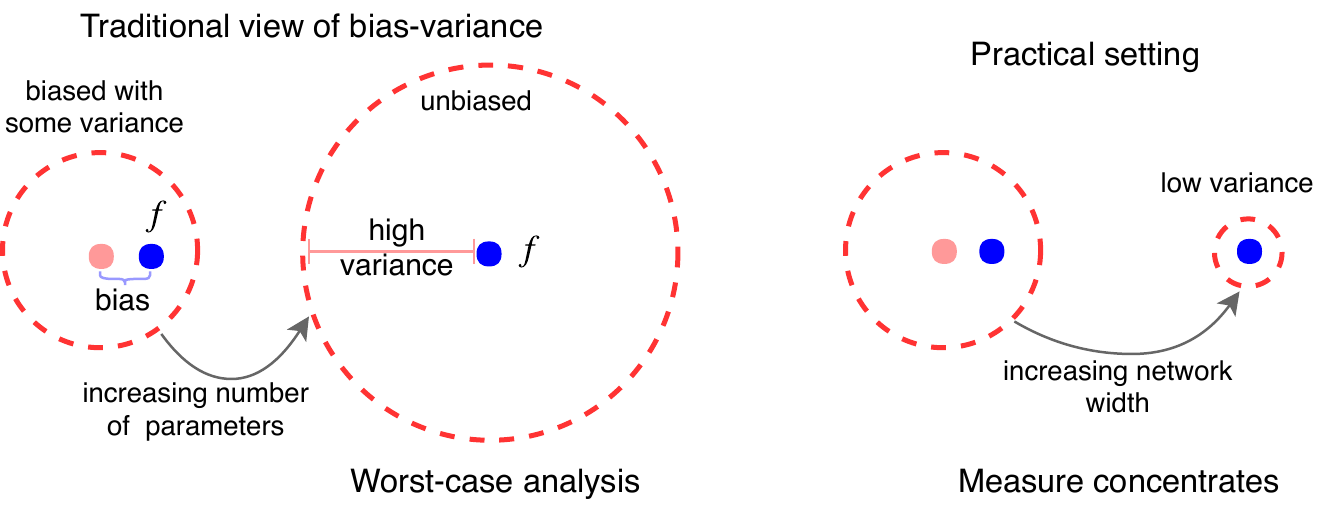}
 \caption[Illustration of probabilistic notion of effective capacity]{The dotted red circle depicts a cartoon version of the $\epsilon$-hypothesis class of the learner. The left side reflects common intuition, as informed by the bias-variance tradeoff and worst-case analysis from statistical learning theory. The right side reflects our view that variance can decrease with network width.}
 \label{fig:var_spectrum}
\end{figure}

\chapter{Additional empirical results and discussion}
\label{app:empirical}





\section{CIFAR10}
\label{app:CIFAR10_width}

\begin{figure}[ht]
    \centering
    \subfigure{
        \includegraphics[width=.475\textwidth]{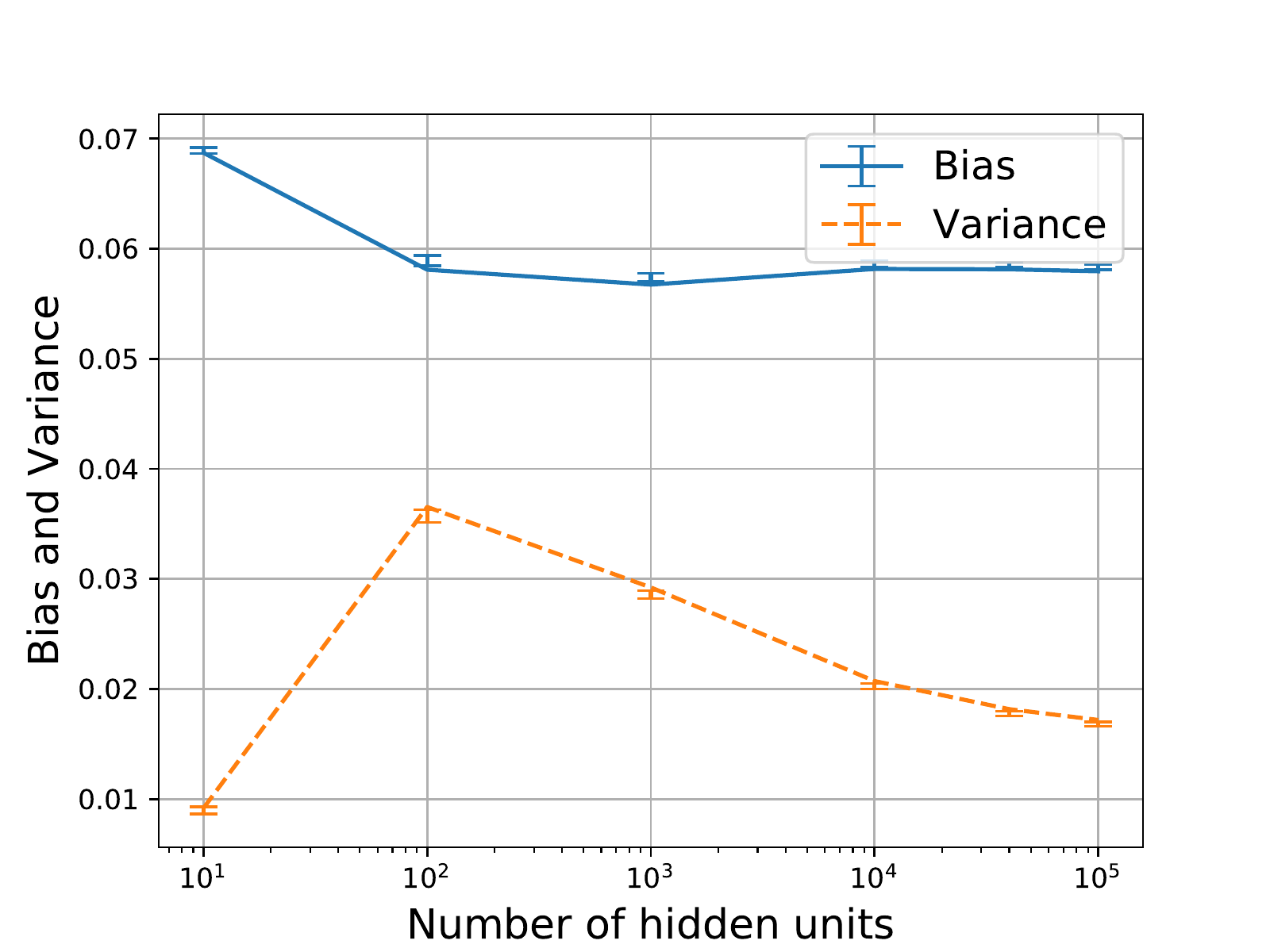}
    }
    \hfill
    \subfigure{
        \includegraphics[width=.475\textwidth]{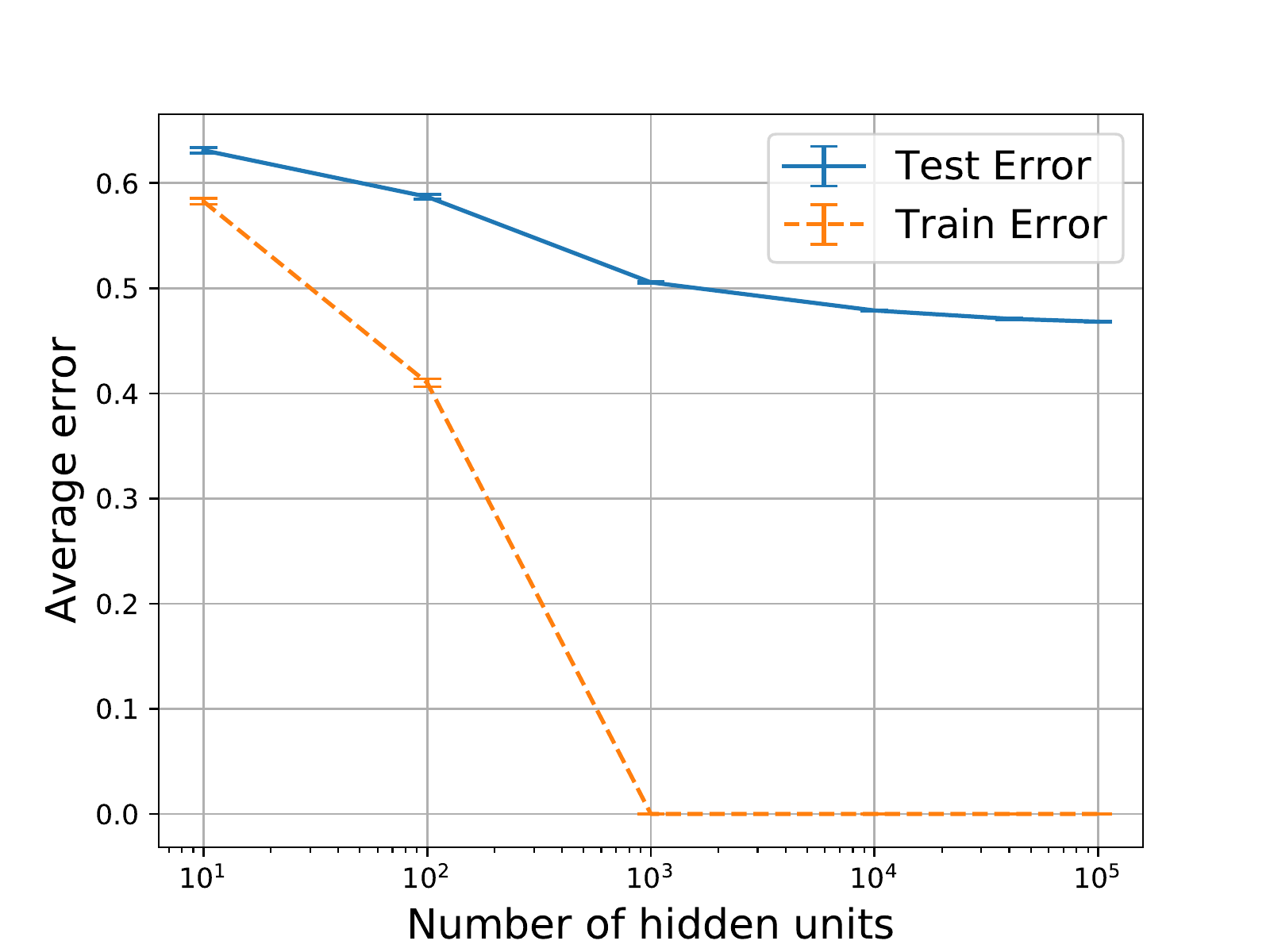}
    }
    \caption[CIFAR10: bias, variance, train error, and test error]{Bias-variance plot (left) and corresponding train and test error (right) for CIFAR10 after training for 150 epochs with step size 0.005 for all networks.}
\end{figure}

\begin{figure}[ht]
    \centering
    \subfigure{
        \includegraphics[width=.475\textwidth]{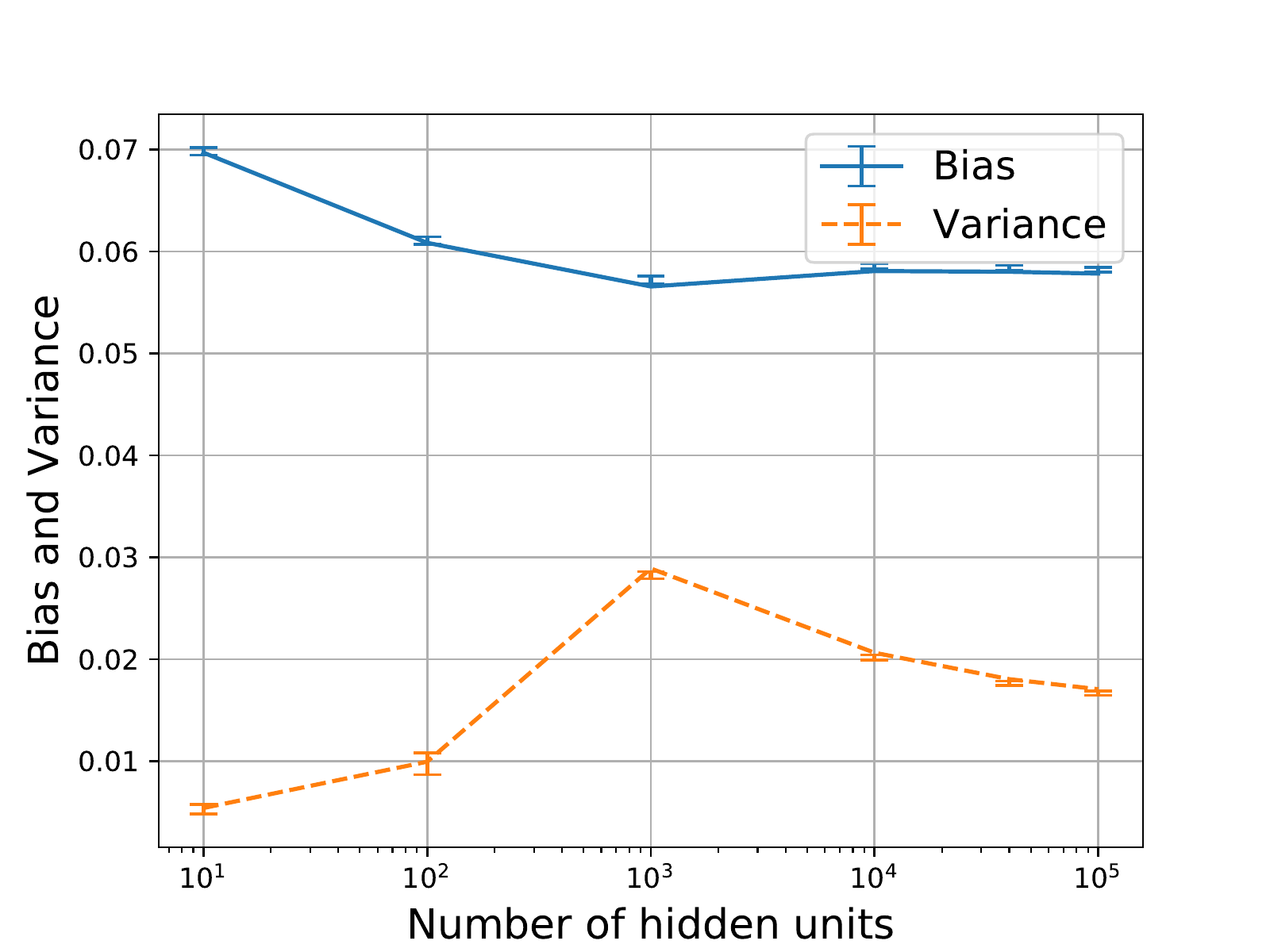}
    }
    \hfill
    \subfigure{
        \includegraphics[width=.475\textwidth]{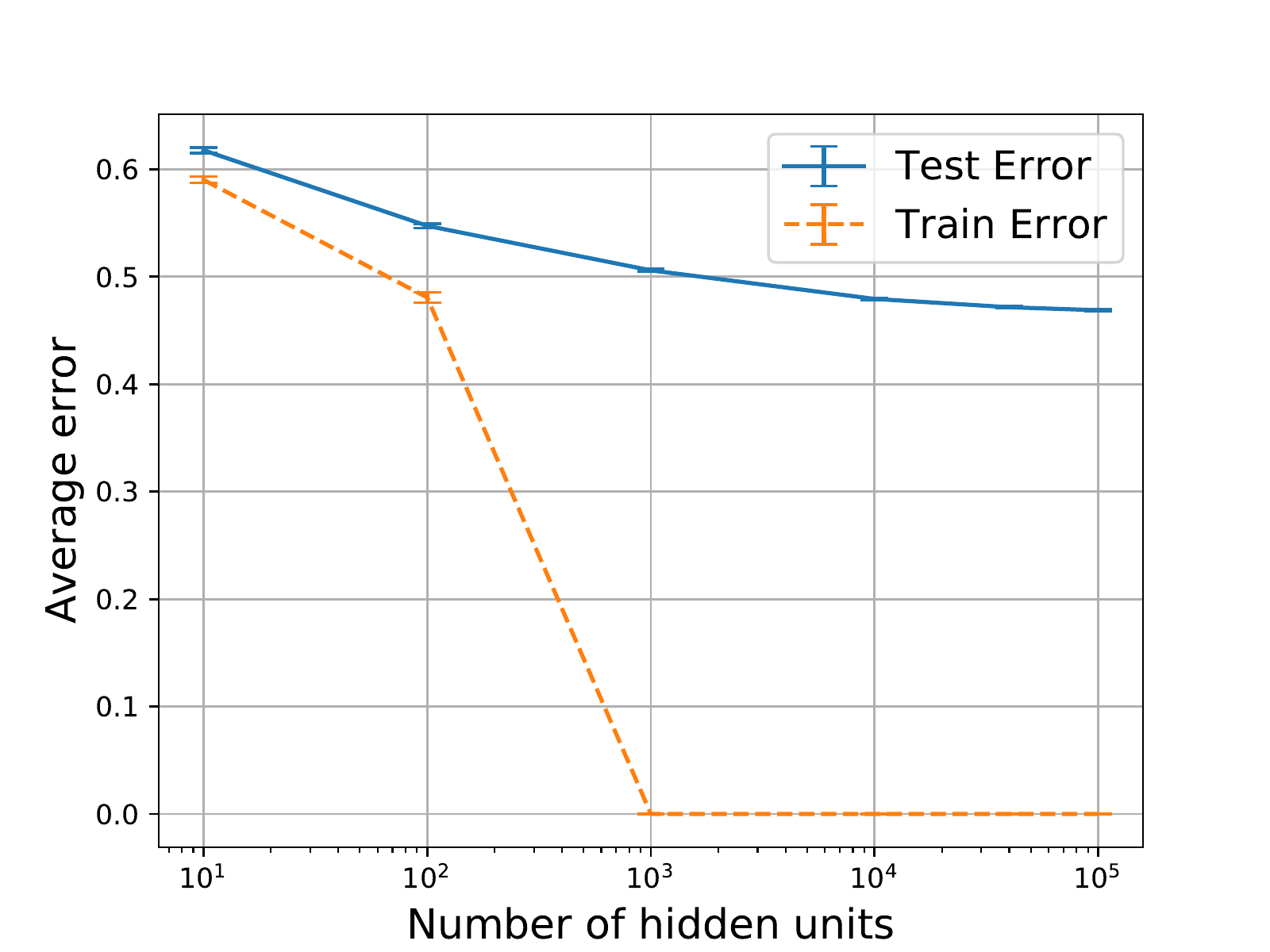}
    }
    \caption[CIFAR10 with early stopping: bias, variance, train error, and test error]{Bias-variance plot (left) and corresponding train and test error (right) for CIFAR10 after training for using \emph{early stopping} with step size 0.005 for all networks.}
\end{figure}

\section{SVHN}
\label{app:SVHN_width}

\begin{figure}[H]
    \centering
    \subfigure{
        \includegraphics[width=.475\textwidth]{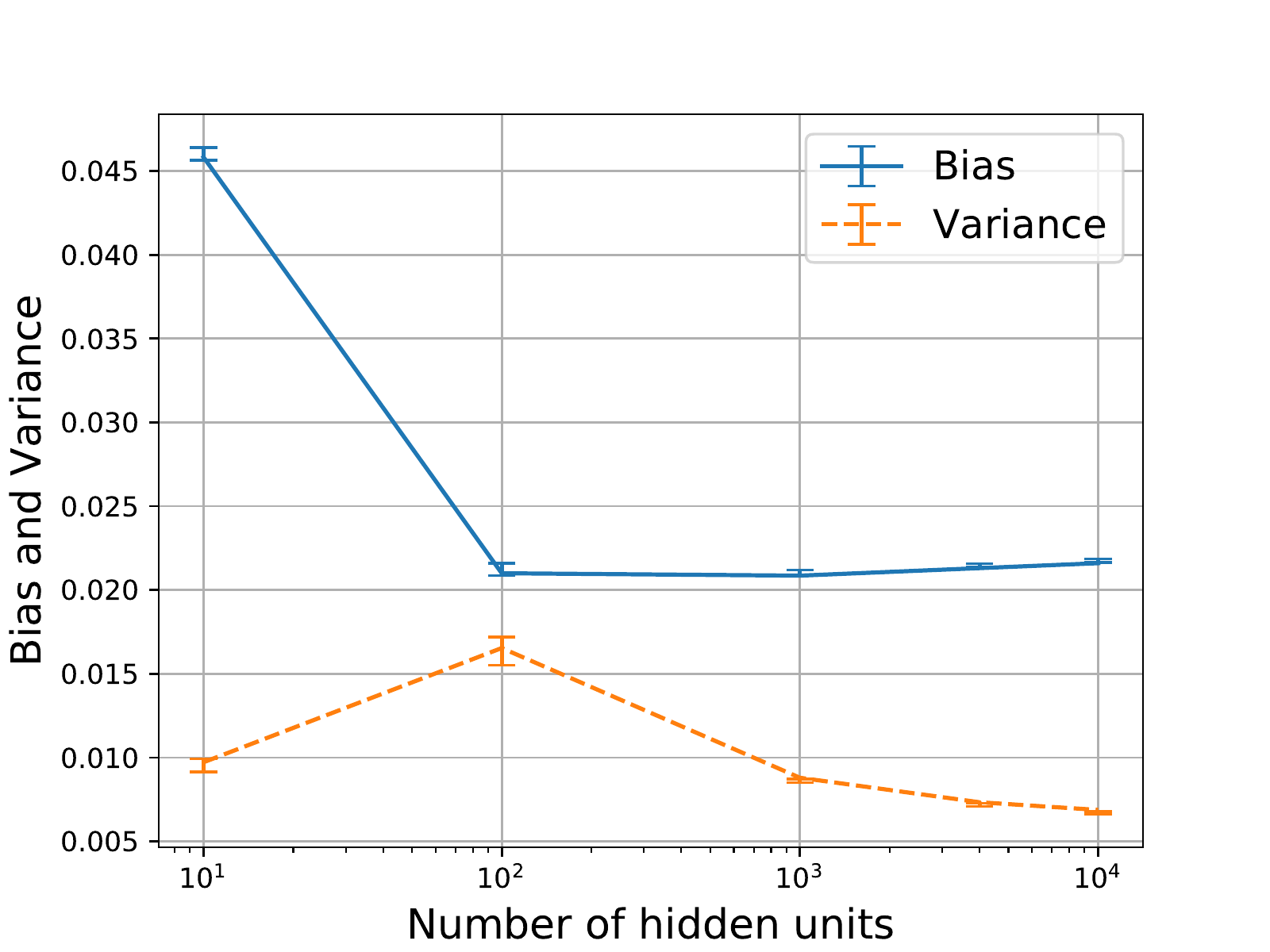}
    }
    \hfill
    \subfigure{
        \includegraphics[width=.475\textwidth]{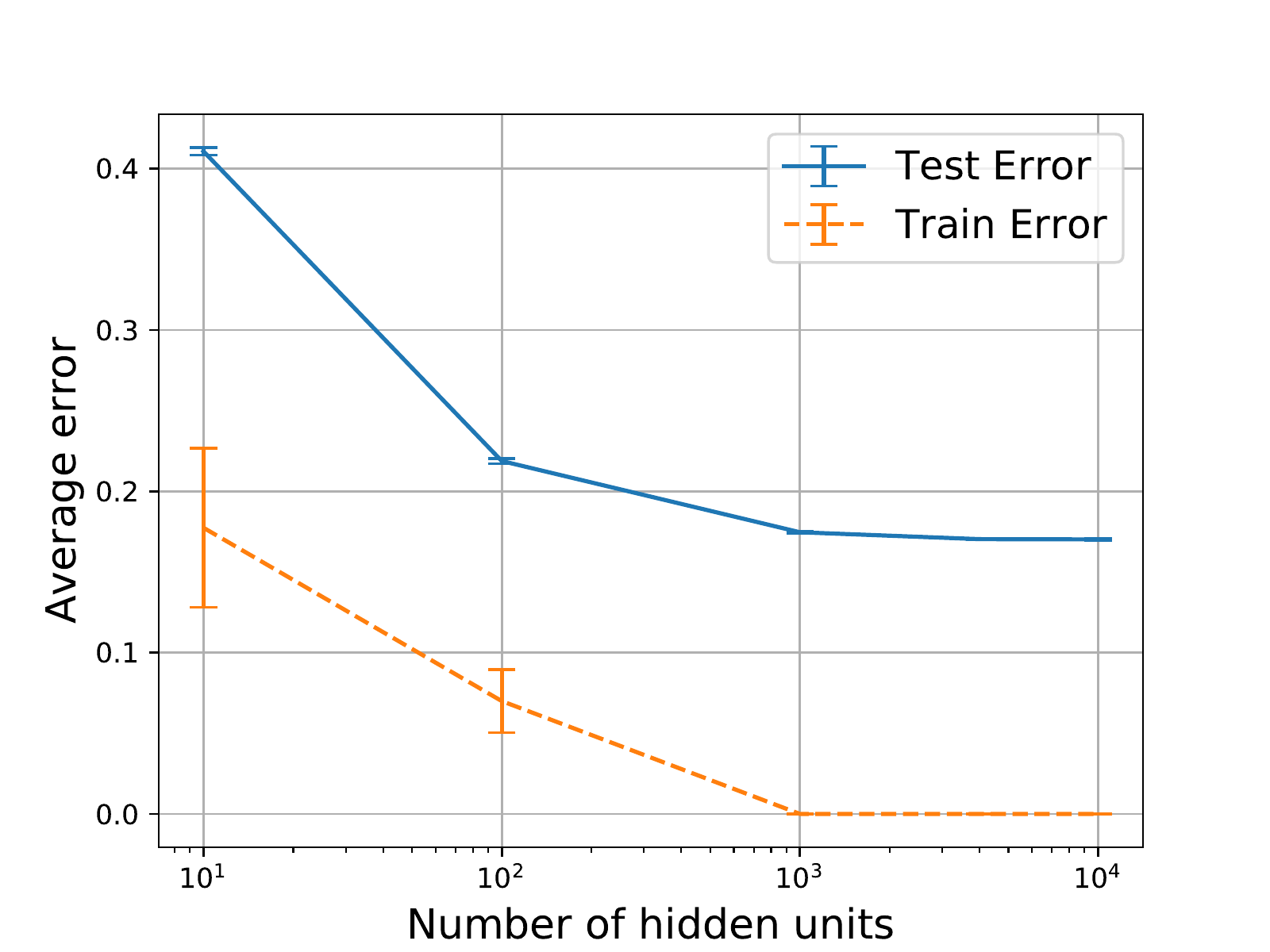}
    }
    \caption[SVHN: bias, variance, train error, and test error]{Bias-variance plot (left) and corresponding train and test error (right) for SVHN after training for 150 epochs with step size 0.005 for all networks.}
\end{figure}

\section{MNIST}
\label{app:MNIST_width}

\begin{figure}[H]
    \centering
    \subfigure{
        \includegraphics[width=.475\textwidth]{figures/full_data_width/MNIST/bias-variance_long}
    }
    \subfigure{
        \includegraphics[width=.475\textwidth]{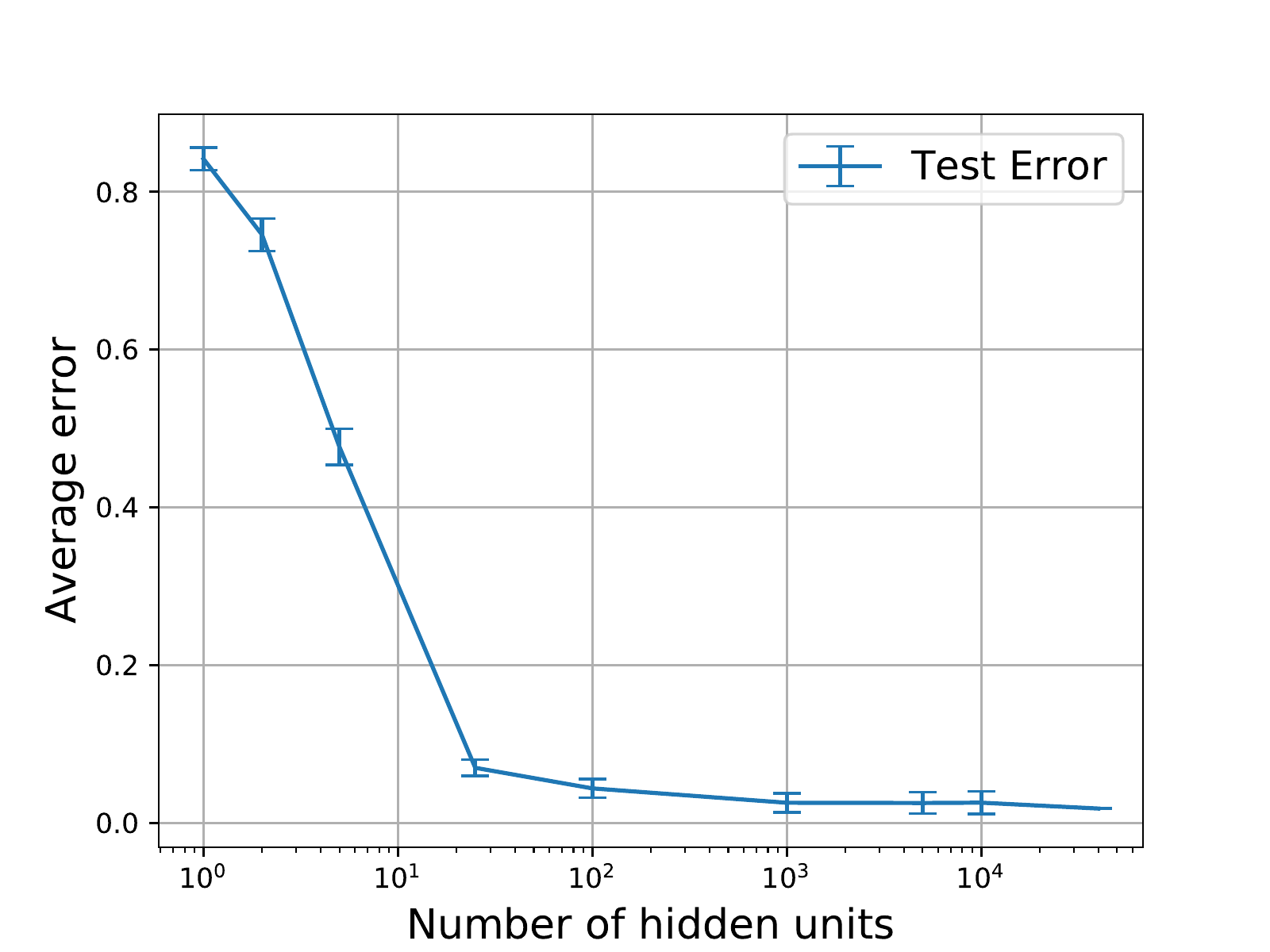}
    }
    \caption[MNIST: Test error, along with bias and variance]{MNIST bias-variance plot from main paper (left) next to the corresponding test error (right)}
\end{figure}

\begin{figure}[H]
    \centering
    \includegraphics[width=.6\textwidth]{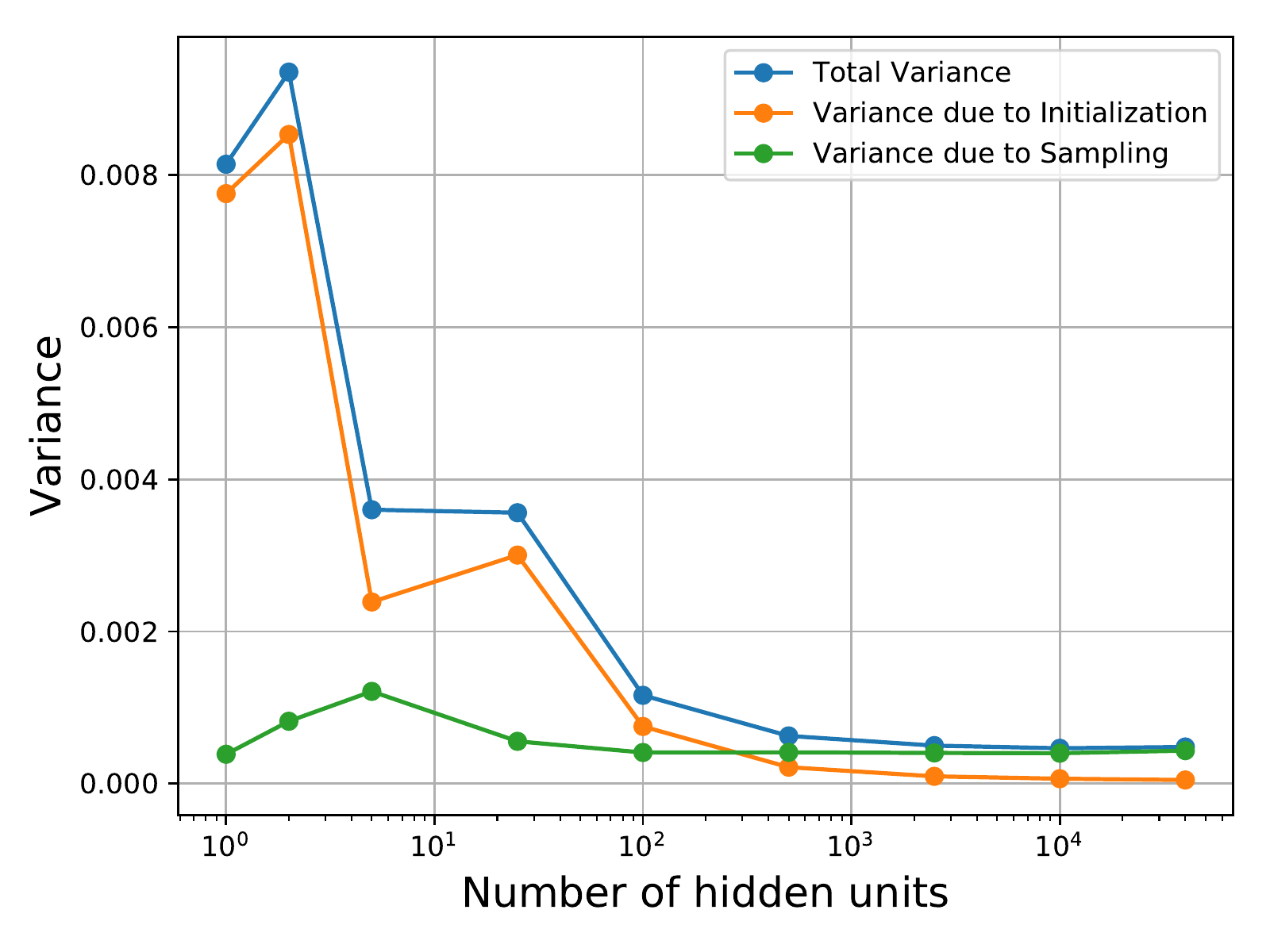}
    \caption{Decomposed variance on MNIST}
\end{figure}

\section{Tuned learning rates for SGD}
\label{app:tuned_lr}

\begin{figure}[H]
    \centering
    \subfigure[Variance decreases with width, even in the small data setting (SGD). This figure is in the main paper, but we include it here to compare with the corresponding step sizes used.]{
        \includegraphics[width=.475\textwidth]{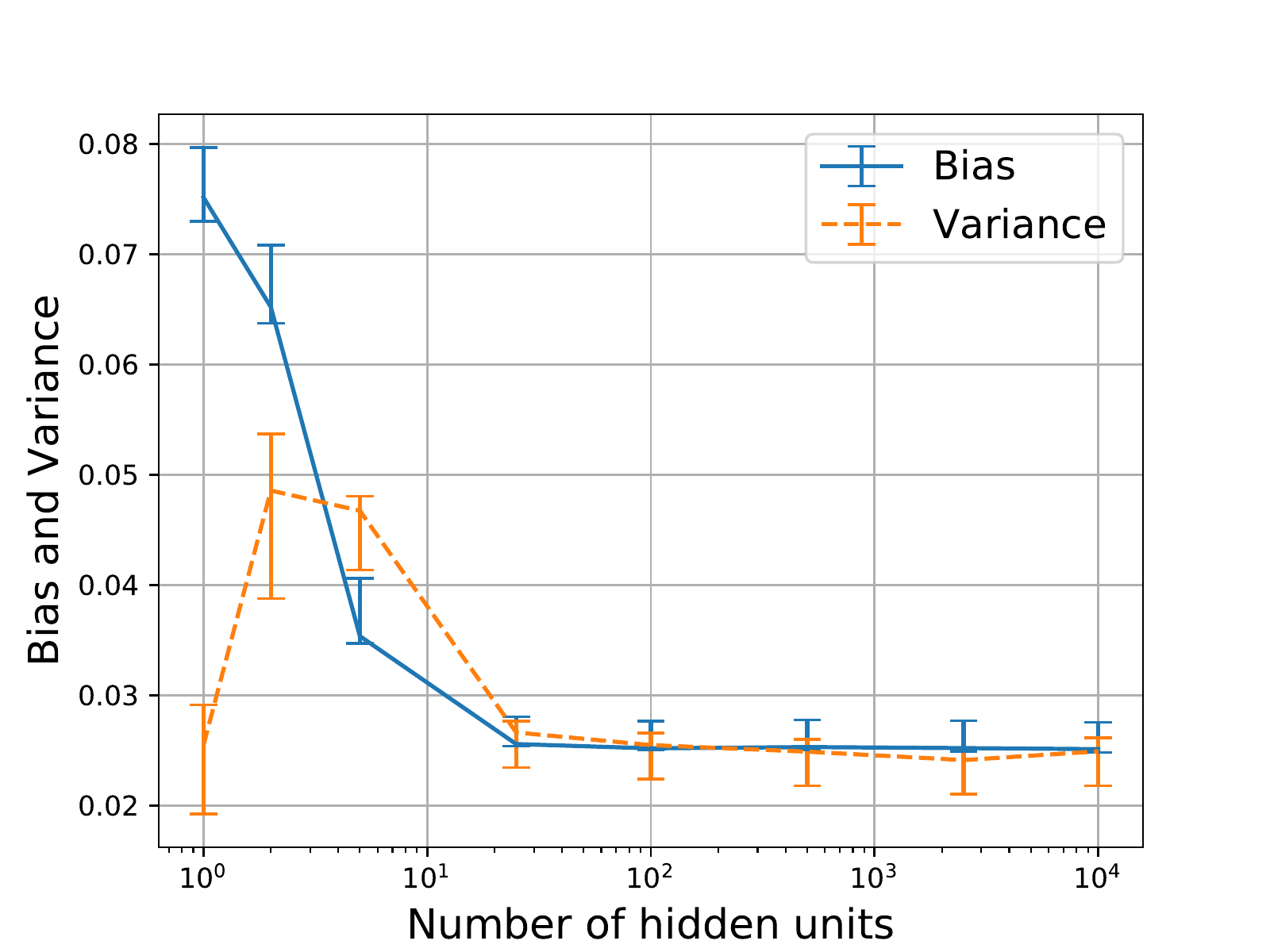}
    }
    \hfill
    \subfigure[Corresponding optimal learning rates found, by random search, and used.]{
        \includegraphics[width=.475\textwidth]{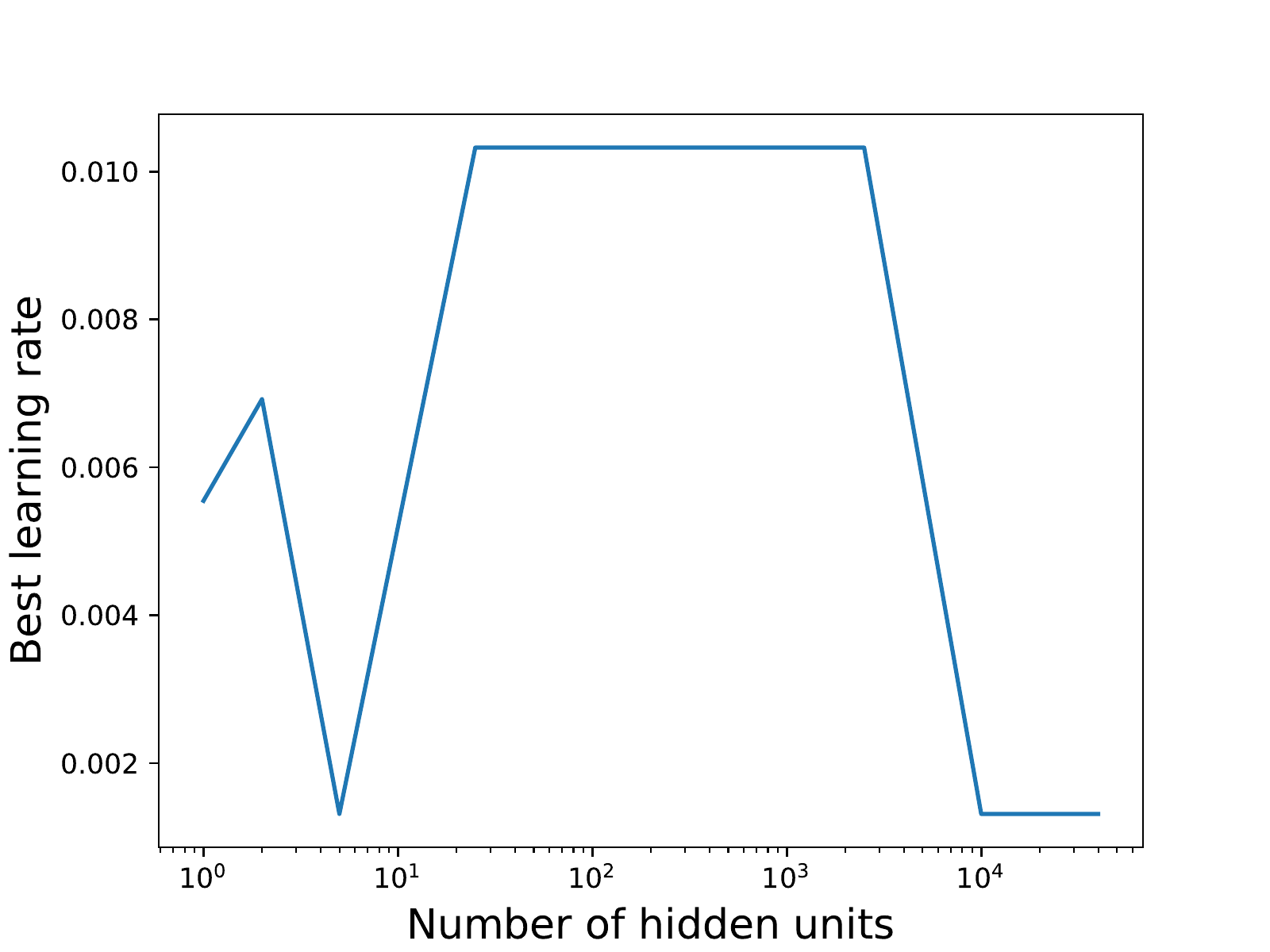}
    }
    \caption{Tuned learning rates used for small data MNIST}
\end{figure}

\section{Fixed learning rate results for small data MNIST}
\label{app:fixed_lr}

\begin{figure}[H]
    \centering
    \subfigure{
        \includegraphics[width=.475\textwidth]{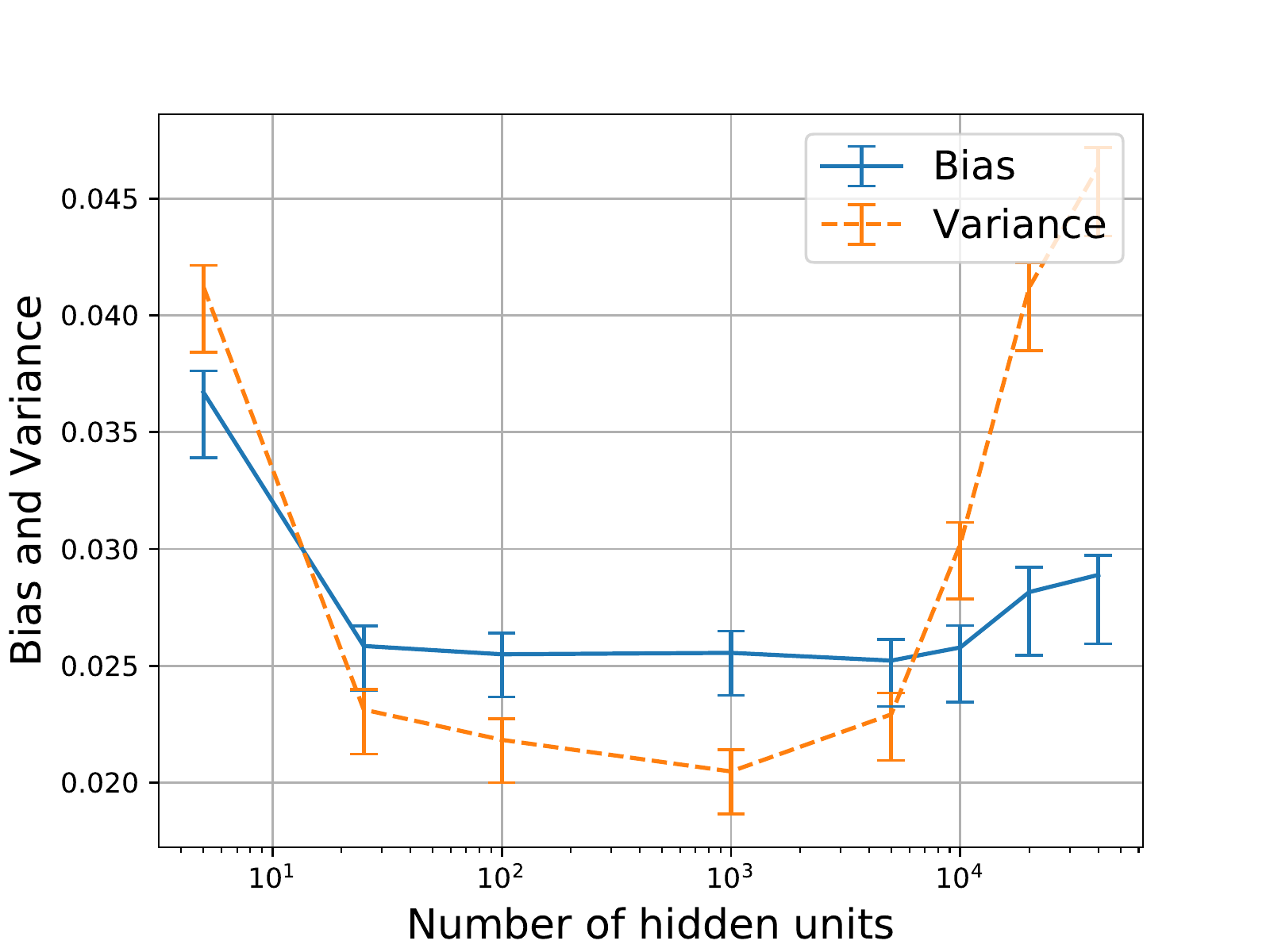}
    }
    \hfill
    \subfigure{
        \includegraphics[width=.475\textwidth]{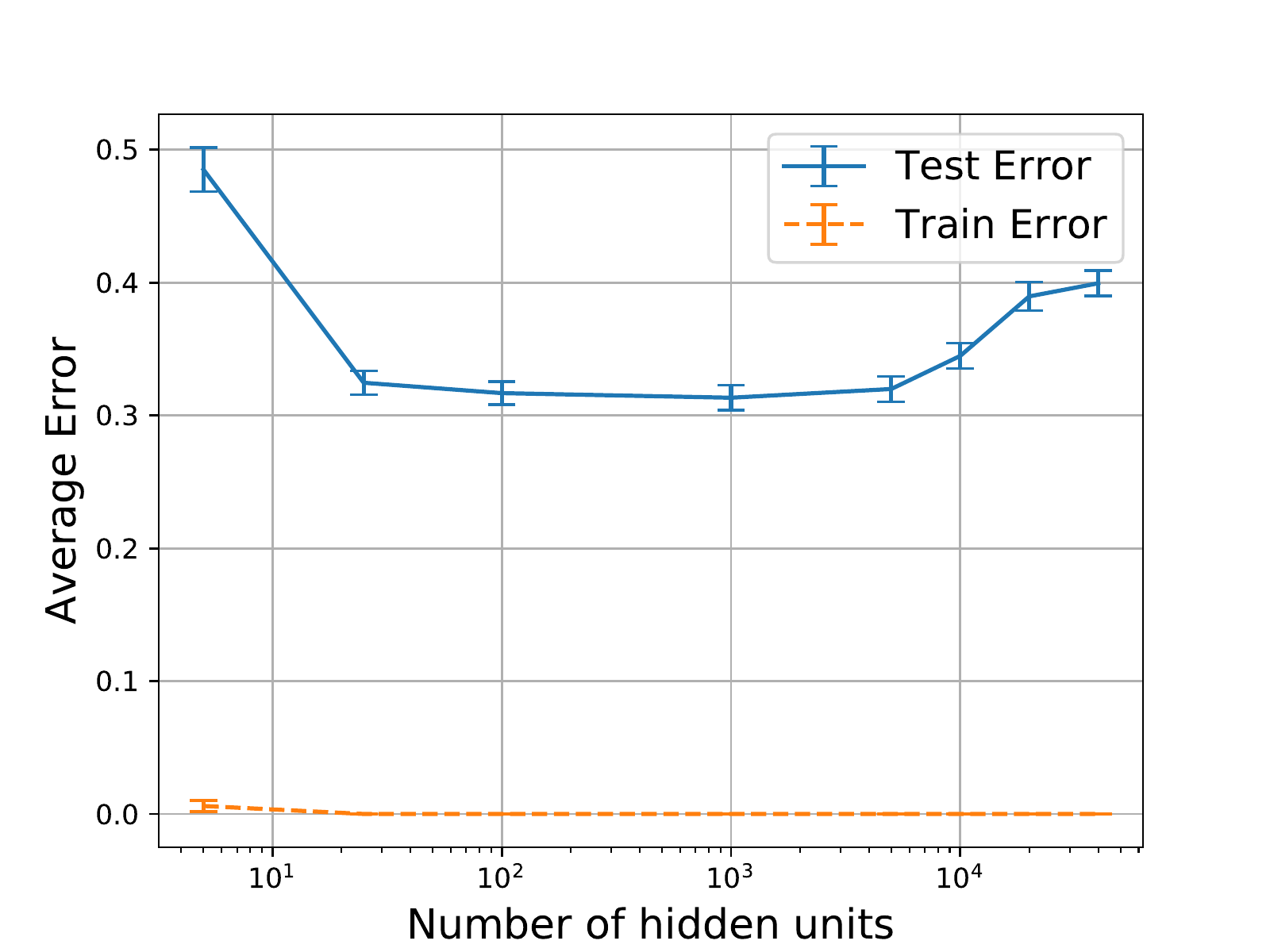}
    }
    \caption{Variance on small data with a fixed learning rate of 0.01 for all networks.}
    \label{fig:fixed_step_size}
\end{figure}

Note that the U curve shown in \cref{fig:fixed_step_size} when we do not tune the step size is explained by the fact that the constant step chosen is a ``good'' step size for some networks and ``bad'' for others. Results from \citet{Keskar2017}, \citet{smith2018}, and \citet{DBLP:journals/corr/abs-1711-04623} show that a step size that corresponds well to the noise structure in SGD is important for achieving good test set accuracy. Because our networks are different sizes, their stochastic optimization process will have a different landscape and noise structure. By tuning the step size, we are making the experimental design choice to keep \emph{optimality of step size} constant across networks, rather than keeping step size constant across networks. To us, choosing this control makes much more sense than choosing to control for step size. Note that \citet{park2019} show that, as long as the network size is not too big, controlling for ``optimality of step size'' and ``keeping step size constant'' with increasing network width actually correspond to the same thing, as long as the networks are not too big, which fits well with the fact that we used the same step size across network widths for almost all of our experiments.

\section{Other optimizers for width experiment on small data MNIST}
\label{app:other_optimizers}

\begin{figure}[H]
    \centering
    \subfigure{
        \includegraphics[width=.475\textwidth]{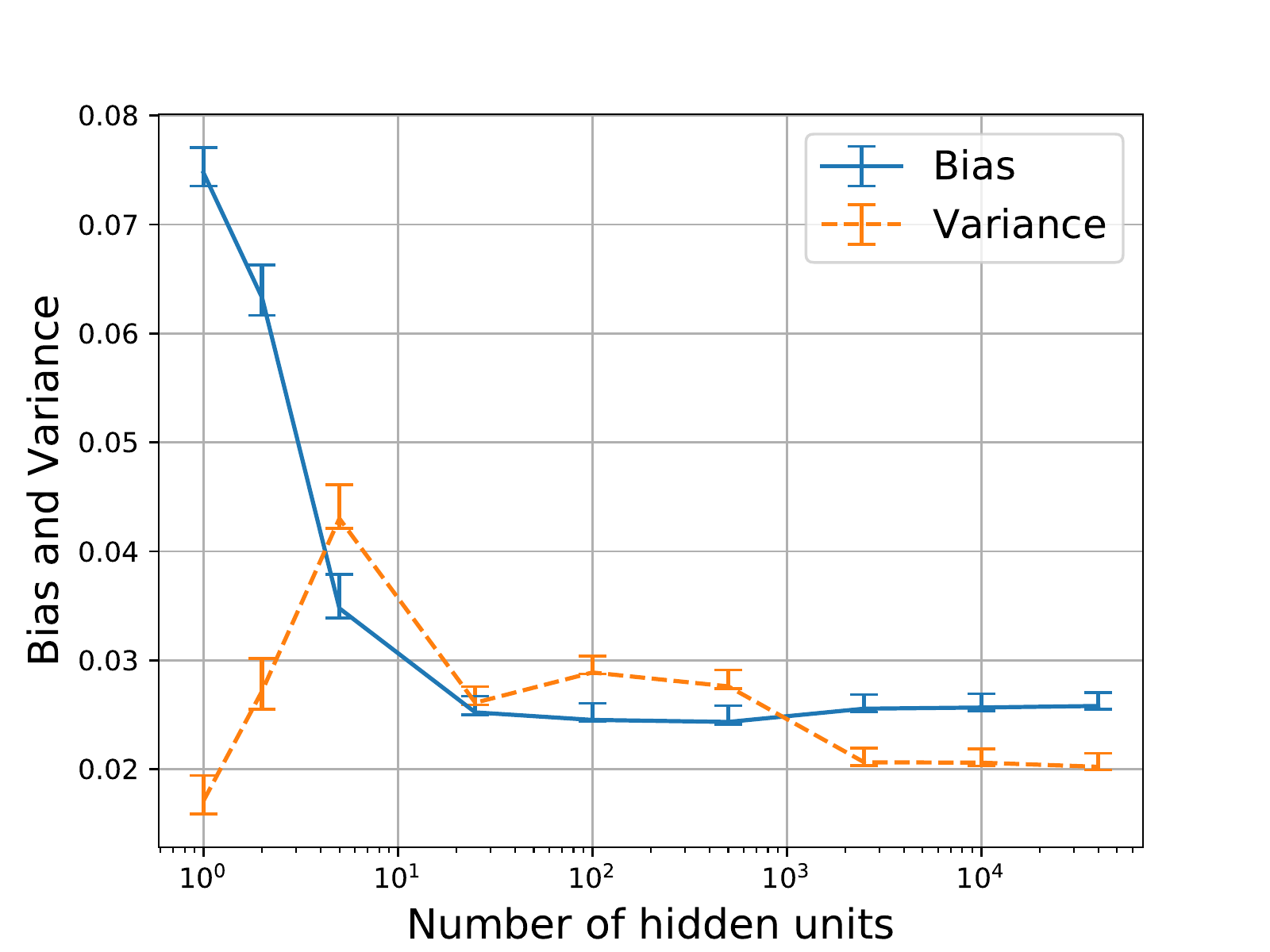}
    }
    \hfill
    \subfigure{
        \includegraphics[width=.475\textwidth]{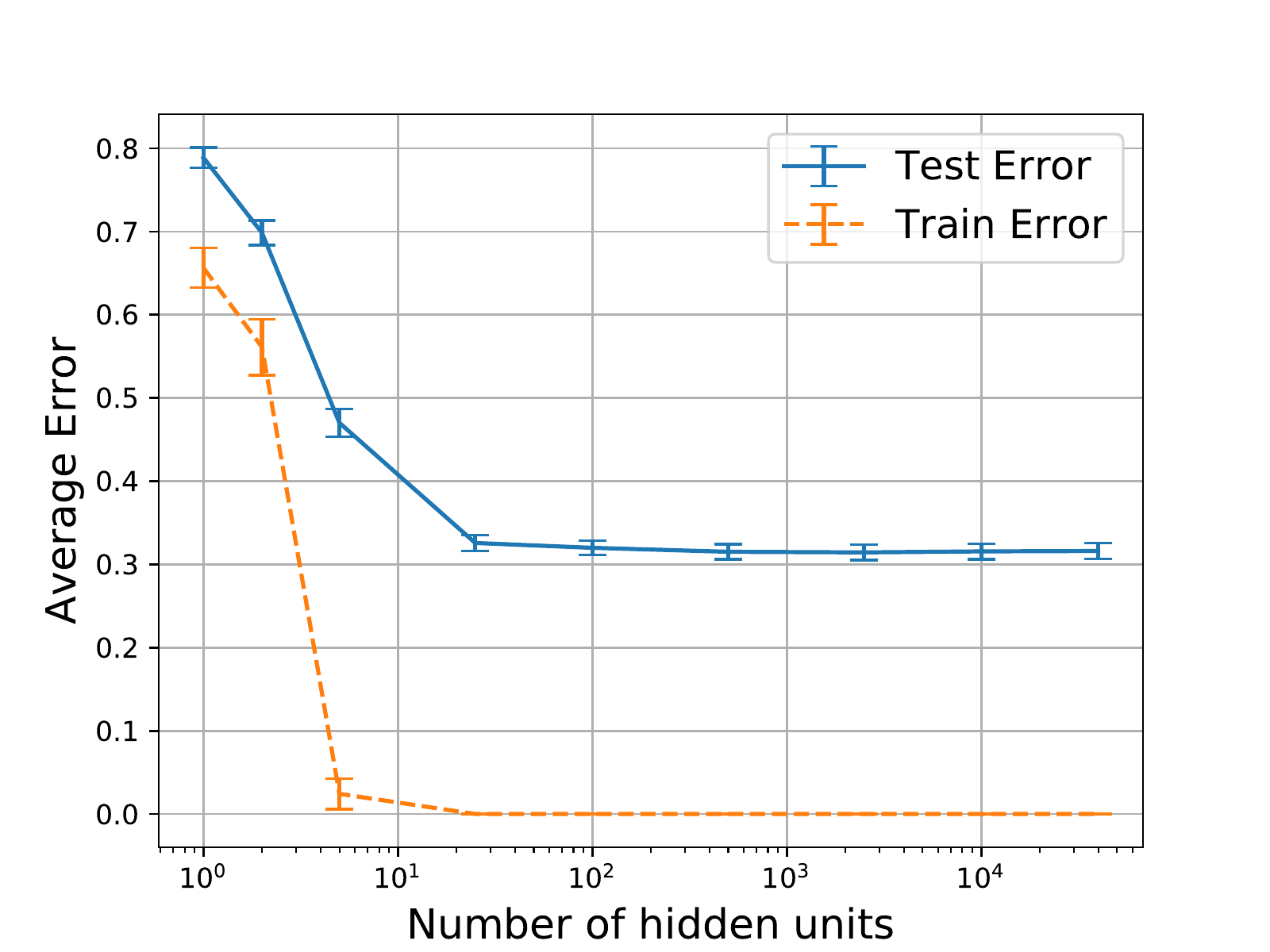}
    }
    \caption[Bias-variance experiments with full batch gradient descent]{Variance decreases with width in the small data setting, even when using batch gradient descent.}
\end{figure}

\begin{figure}[H]
    \centering
    \subfigure{
        \includegraphics[width=.475\textwidth]{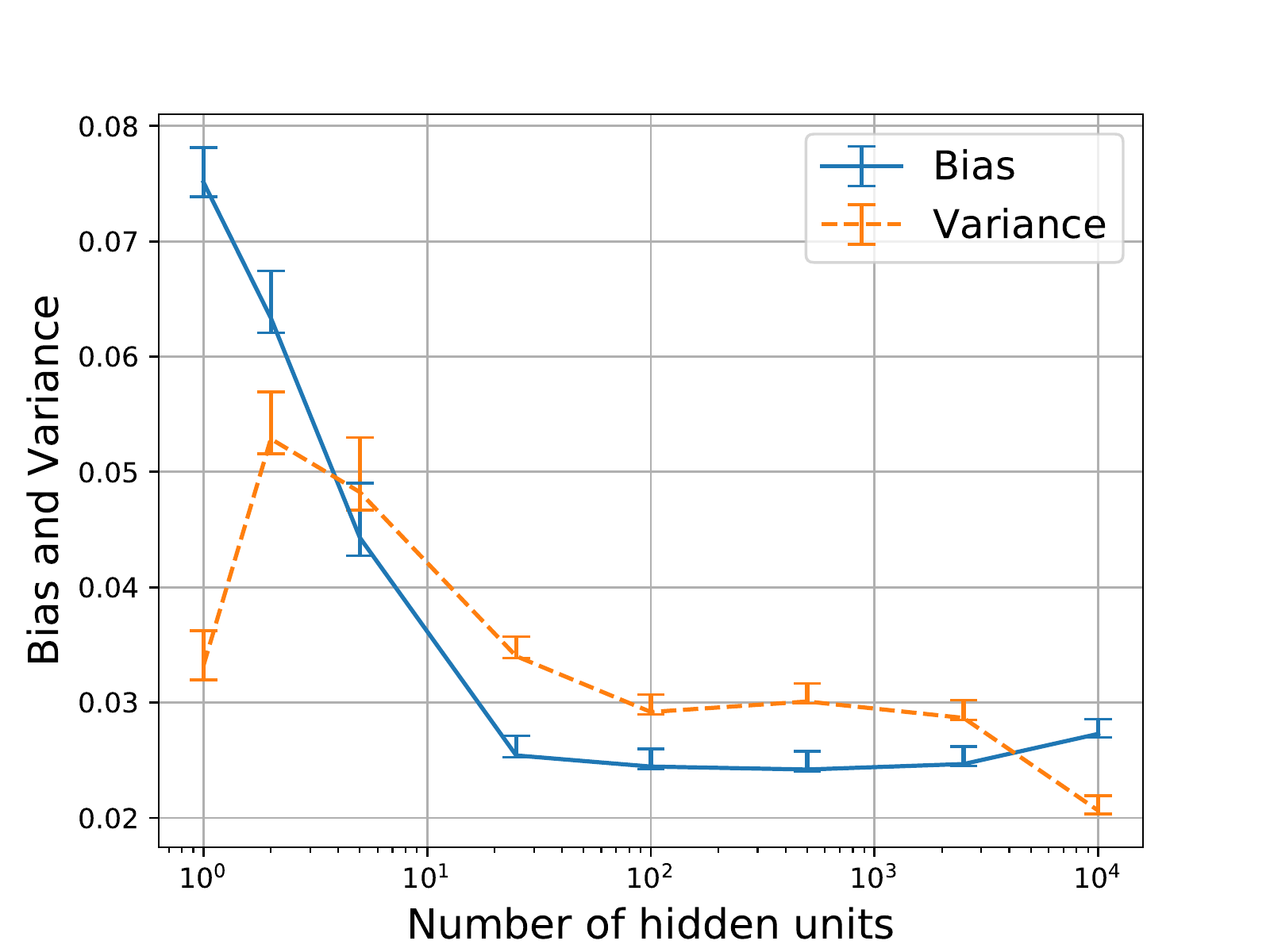}
    }
    \hfill
    \subfigure{
        \includegraphics[width=.475\textwidth]{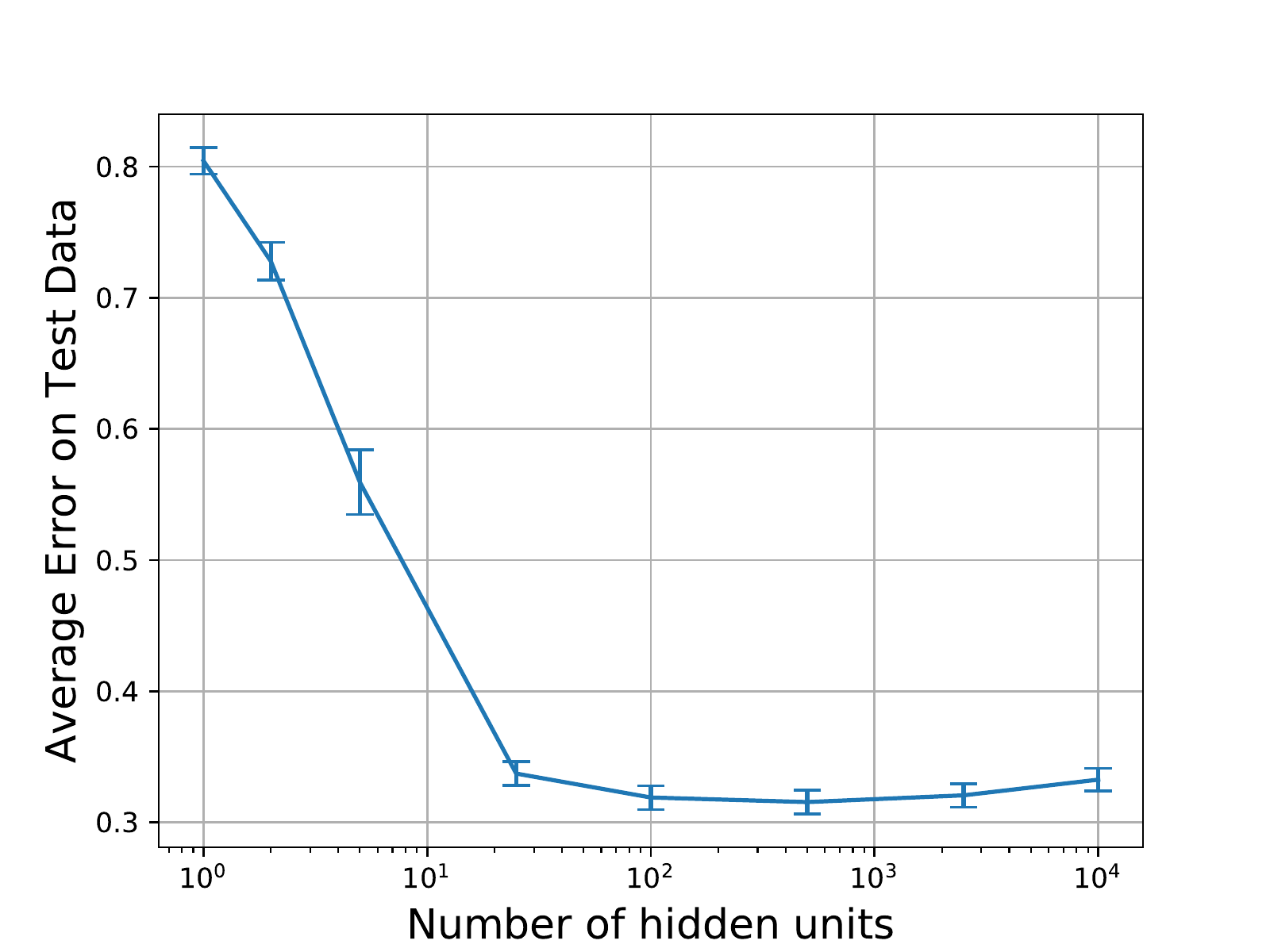}
    }
    \caption[Bias-variance experiments with LBFGS]{Variance decreases with width in the small data setting, even when using a strong optimizer, such as PyTorch's LBFGS, as the optimizer.}
\end{figure}

\newpage
\section{Sinusoid regression experiments}
\label{app:sinusoid_regression}

\begin{figure}[H]
    \centering
    \subfigure[Example of the many different functions learned by a high variance learner \citep{Bishop:2006}]{
        \includegraphics[width=.475\textwidth]{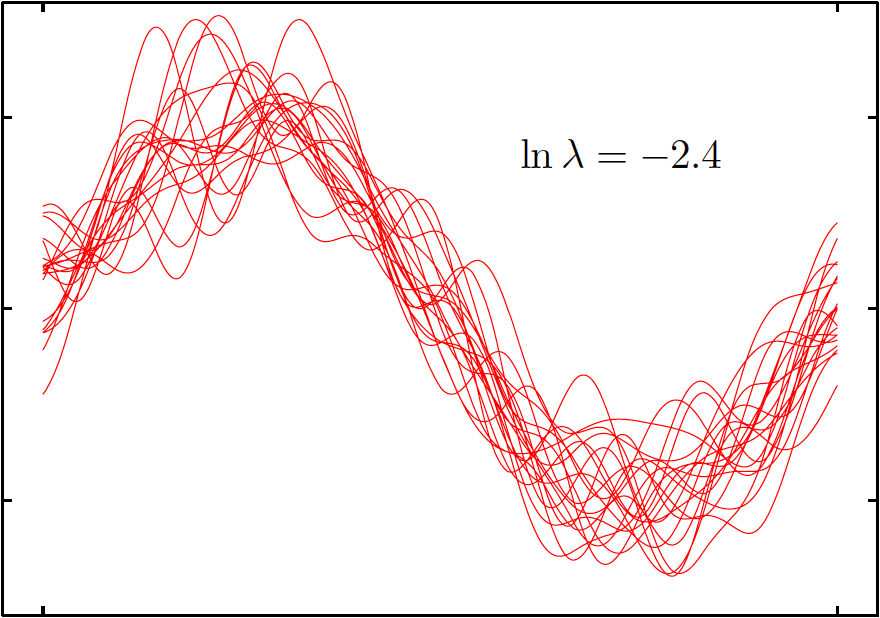}
    }
    \hfill
    \subfigure[Caricature of a single function learned by a high variance learner \citep{wtf_is_bv}]{
        \includegraphics[width=.475\textwidth]{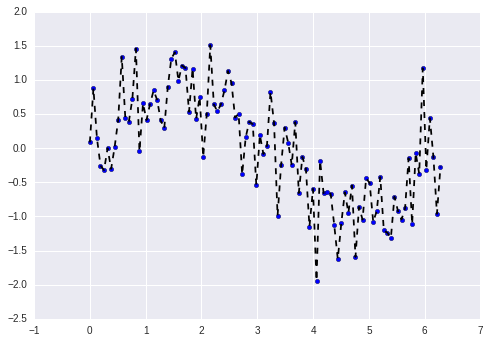}
    }
    \caption[Visualizations of high variance learners]{Caricature examples of high variance learners on sinusoid task. Below, we find that this does not happen with increasingly wide neural networks (\cref{fig:app_all_learned_sinusoids} and \cref{fig:app_mean_and_var_vis}).}
    \label{fig:high_var_caricature}
\end{figure}

\begin{figure}[H]
    \centering
    \includegraphics[width=.55\textwidth]{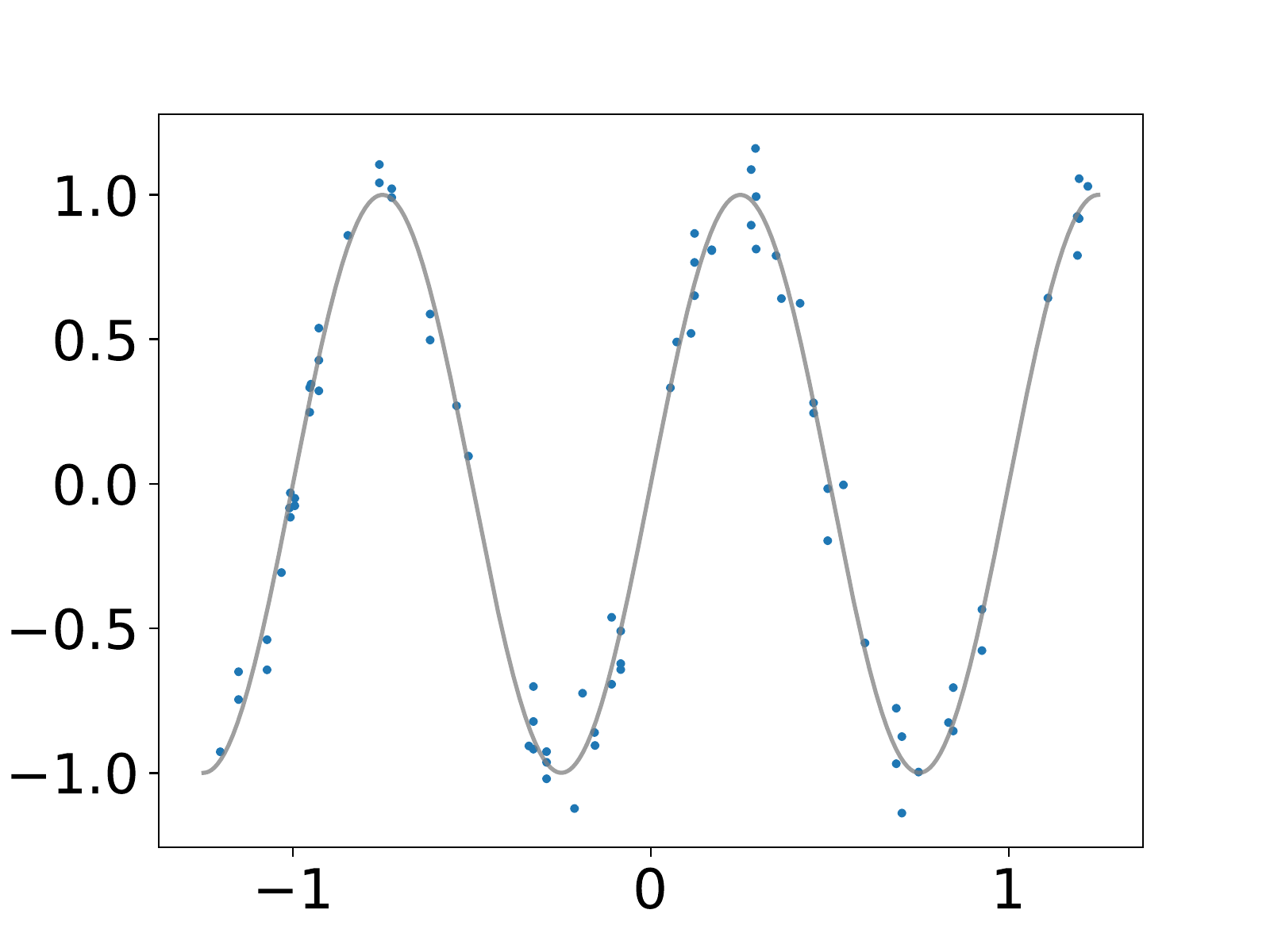}
    \caption[Target function and sampled dataset for sinusoid regression task]{Target function of the noisy sinusoid regression task (in gray) and an example of a training set (80 data points) sampled from the noisy distribution.}
\end{figure}

\begin{figure}[H]
    \centering
    \subfigure{
        \includegraphics[width=.31\textwidth]{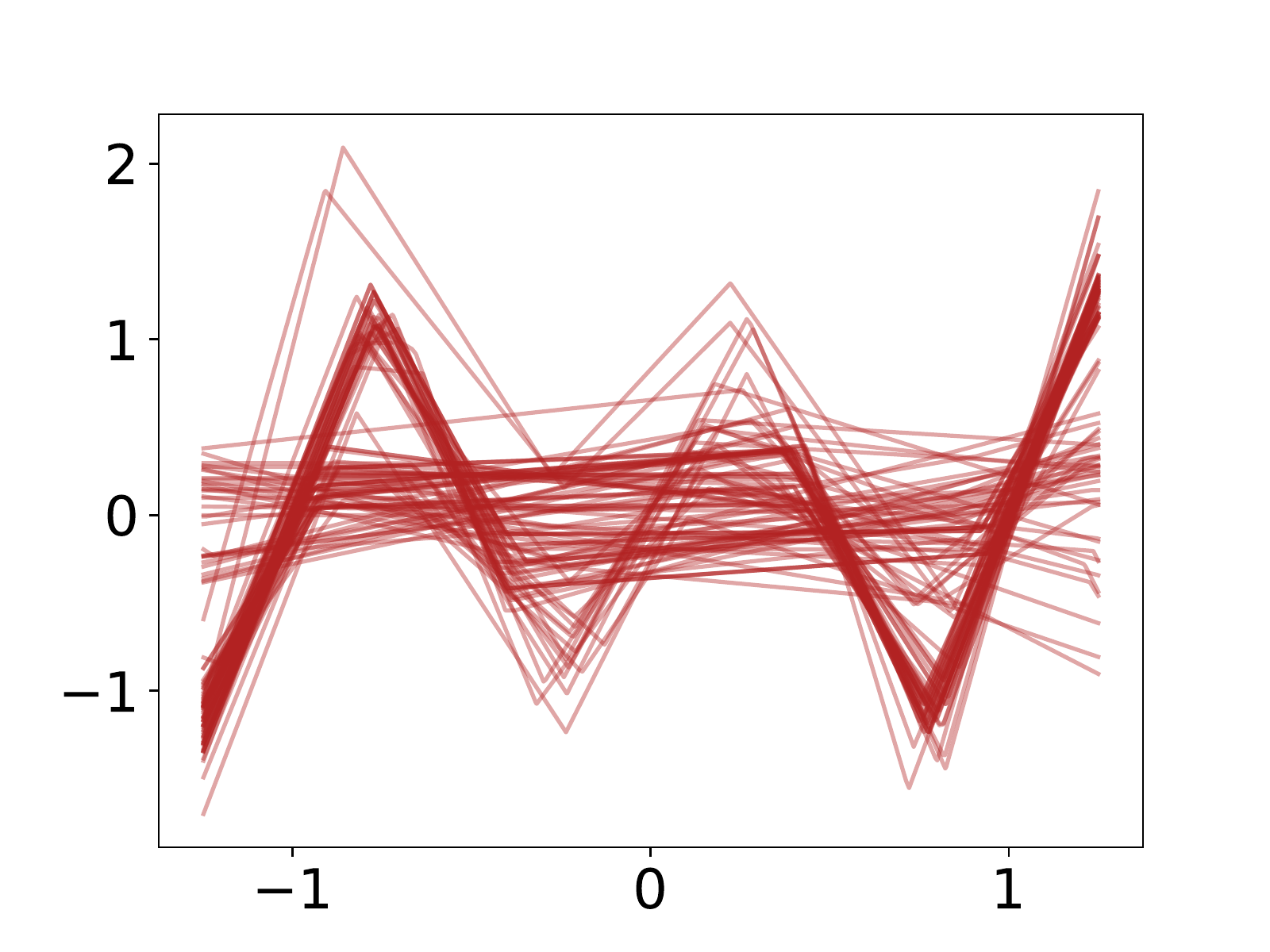}
    }
    \hfill
    \subfigure{
        \includegraphics[width=.31\textwidth]{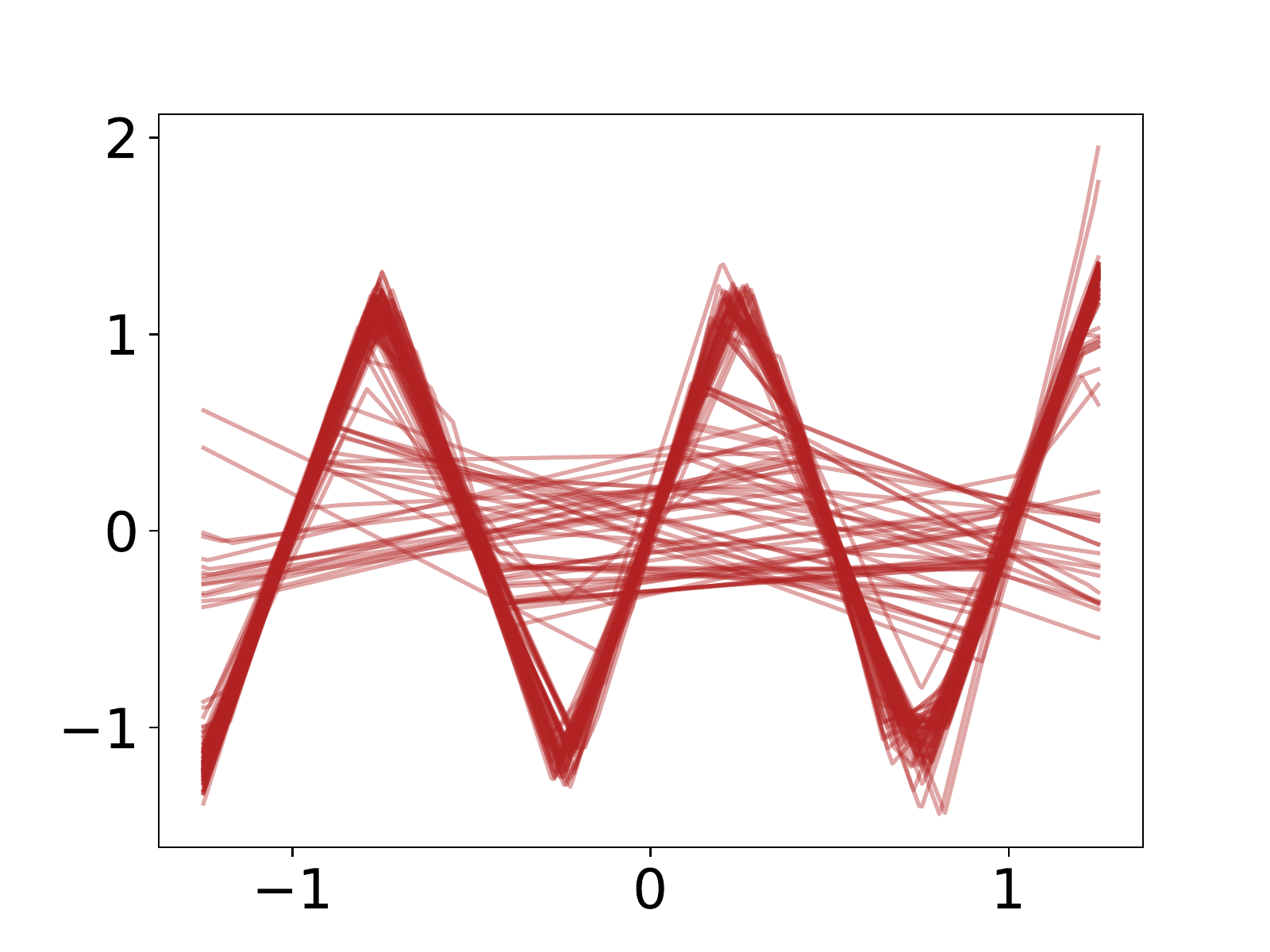}
    }
    \hfill
    \subfigure{
        \includegraphics[width=.31\textwidth]{figures/sinusoids/all_functions_vis/width15}
    }
    
    \subfigure{
        \includegraphics[width=.31\textwidth]{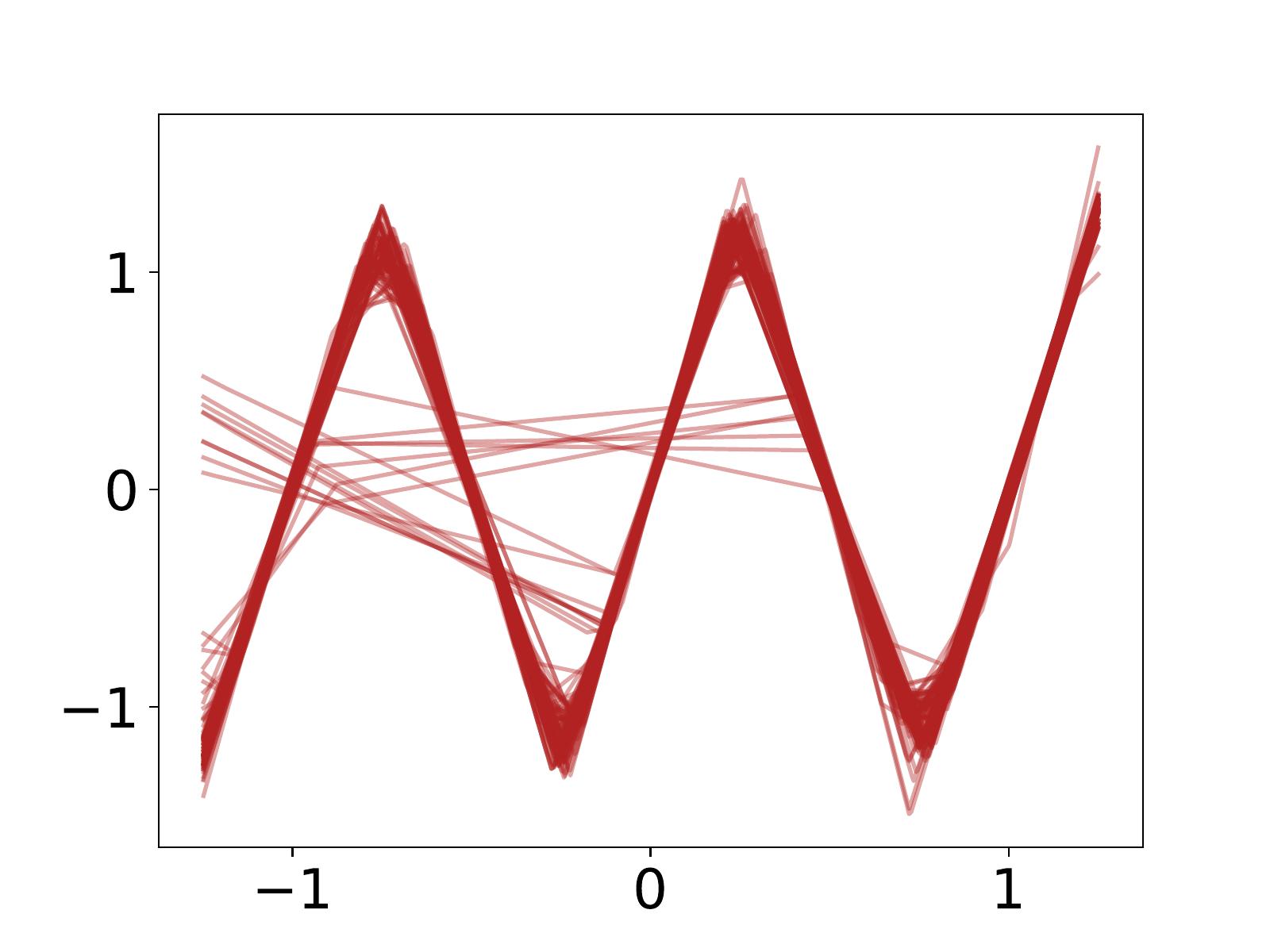}
    }
    \hfill
    \subfigure{
        \includegraphics[width=.31\textwidth]{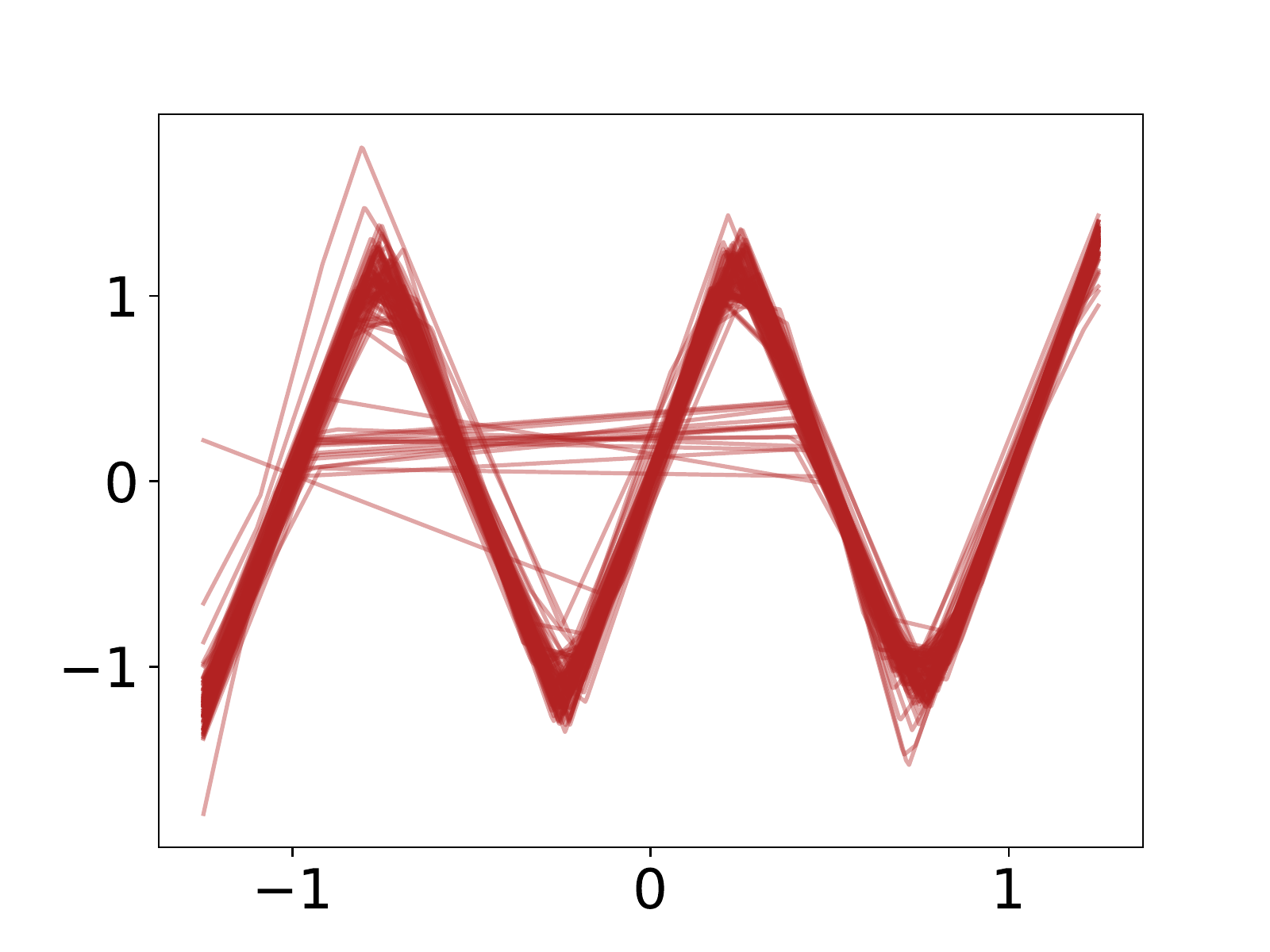}
    }
    \hfill
    \subfigure{
        \includegraphics[width=.31\textwidth]{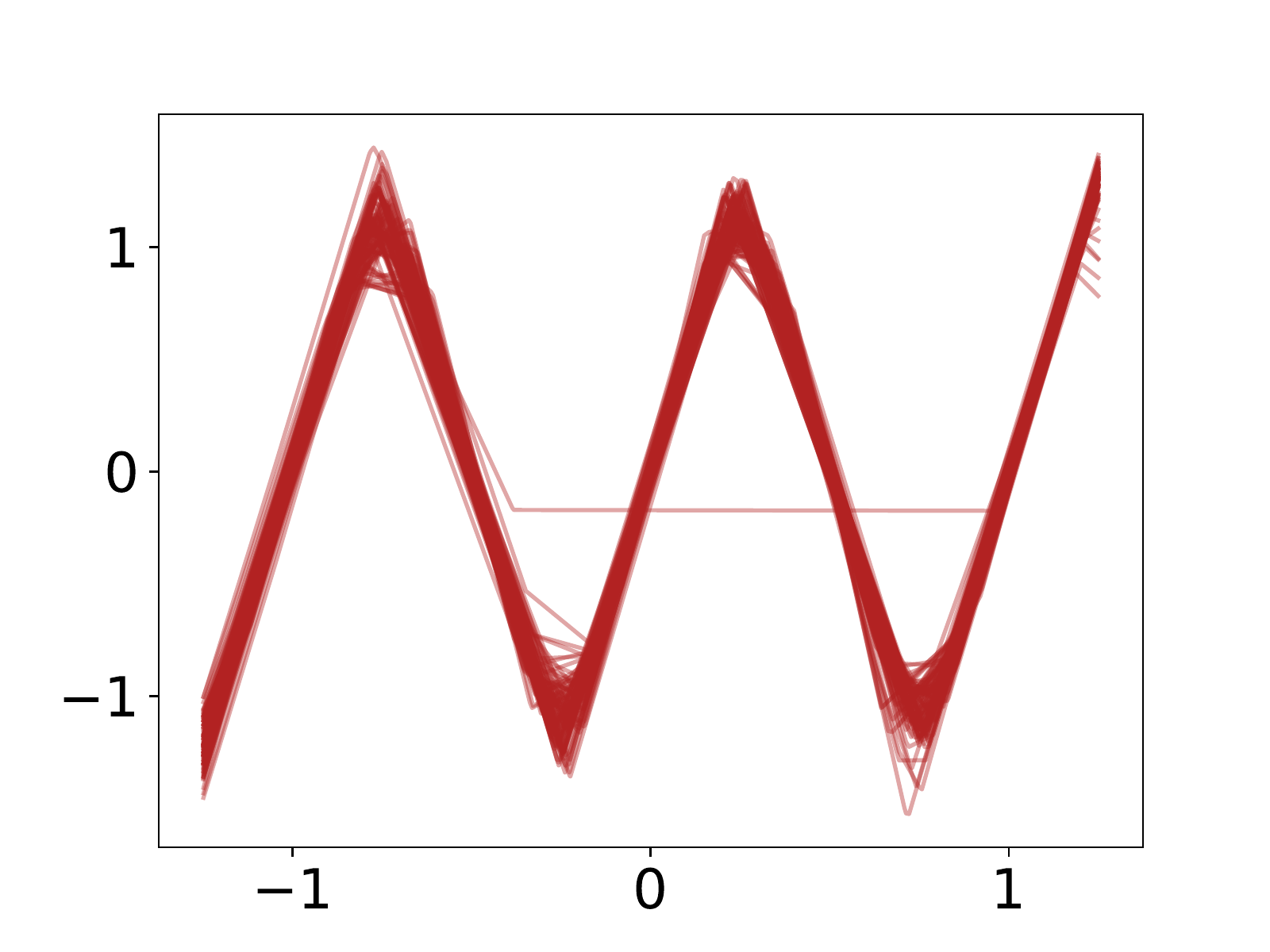}
    }
    
    \subfigure{
        \includegraphics[width=.31\textwidth]{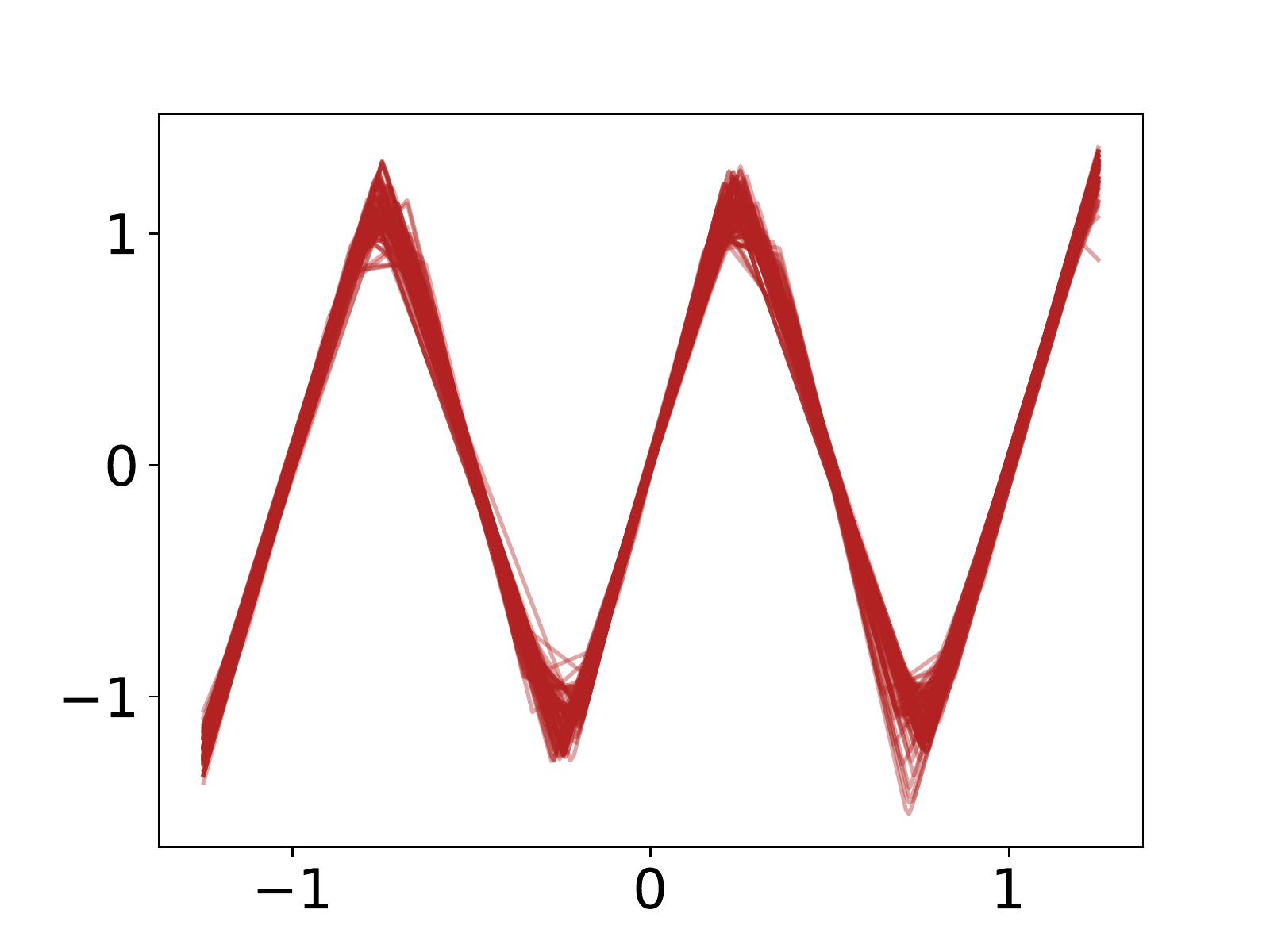}
    }
    \hfill
    \subfigure{
        \includegraphics[width=.31\textwidth]{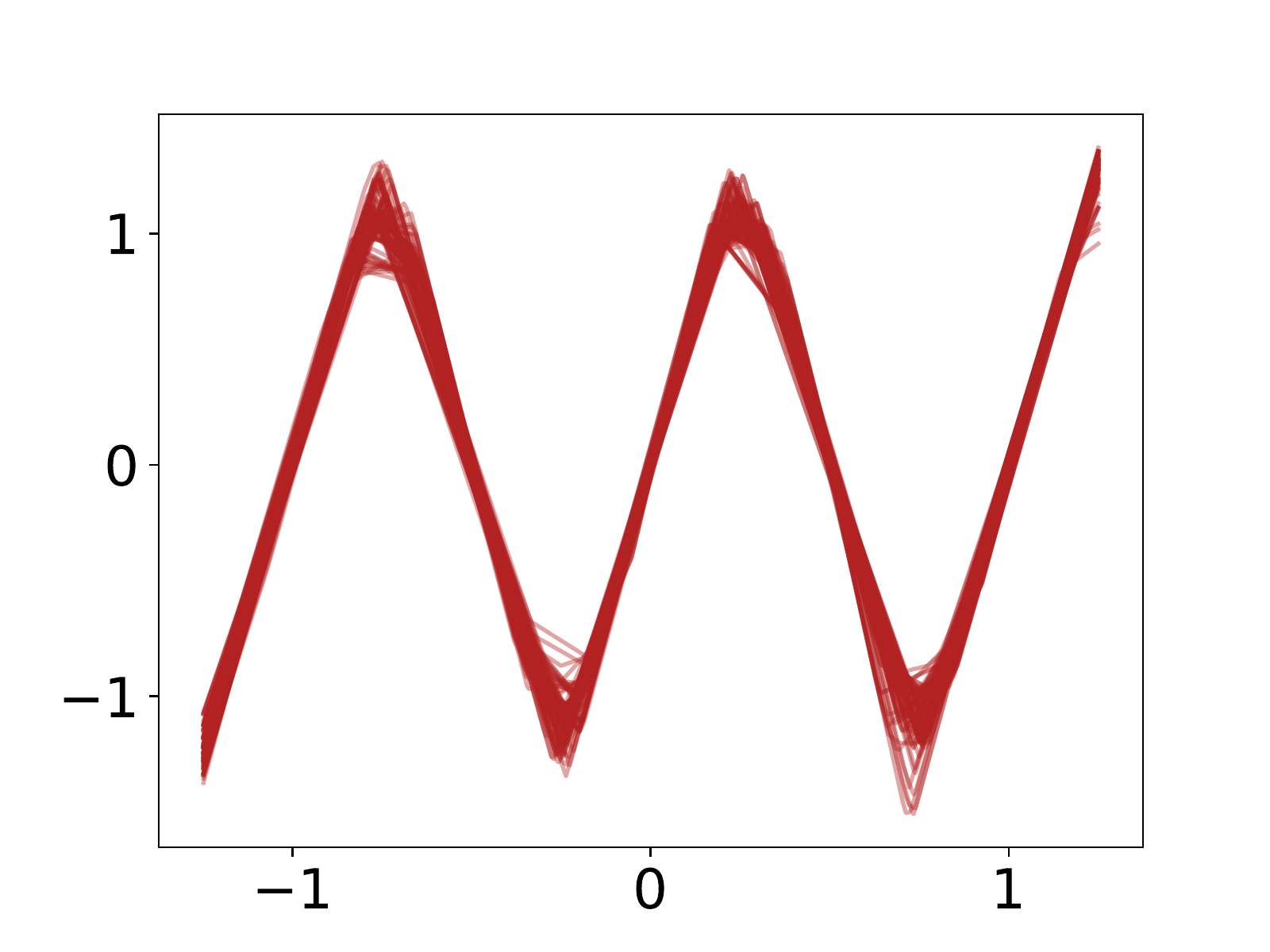}
    }
    \hfill
    \subfigure{
        \includegraphics[width=.31\textwidth]{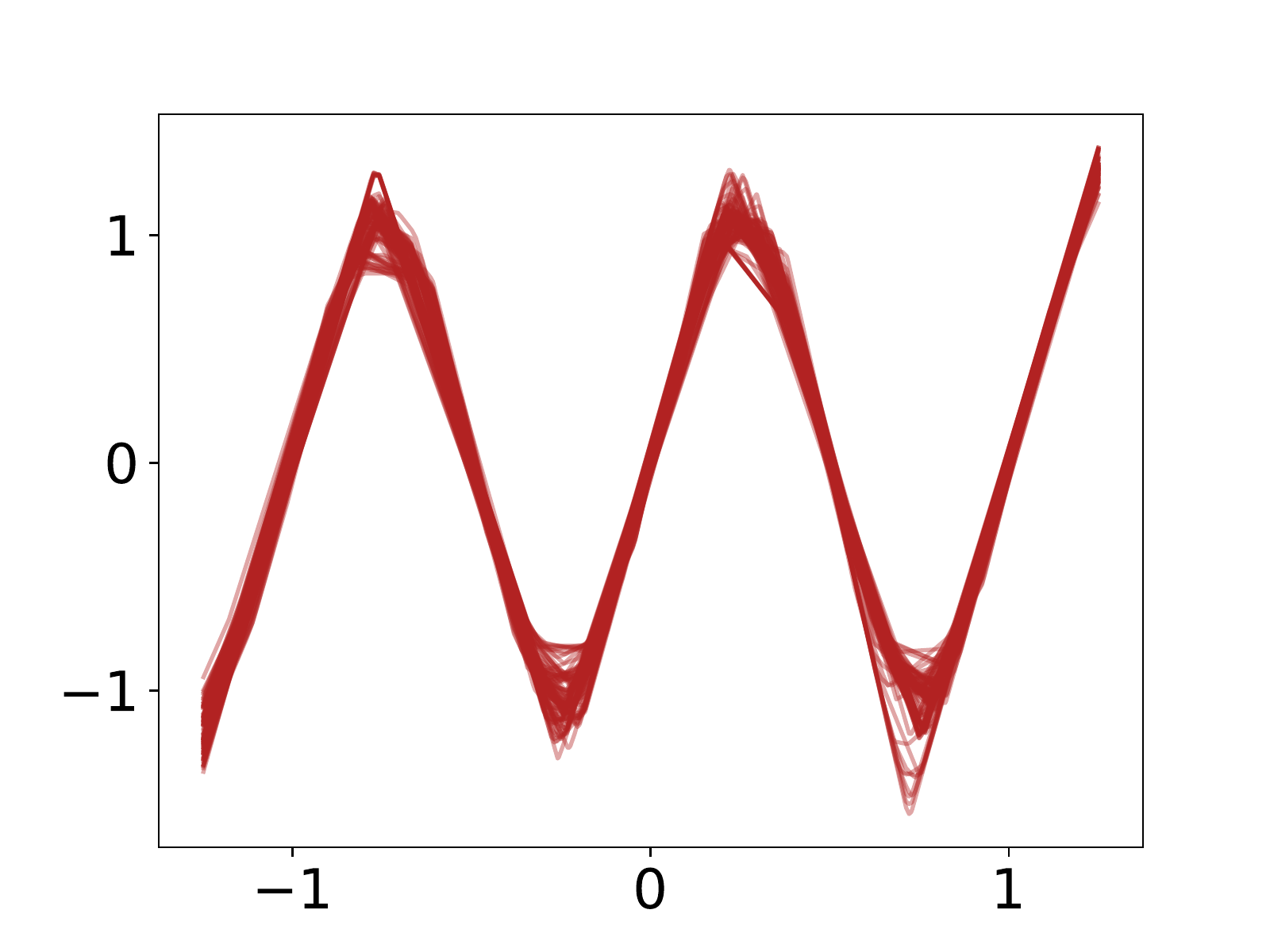}
    }
    
    \subfigure{
        \includegraphics[width=.31\textwidth]{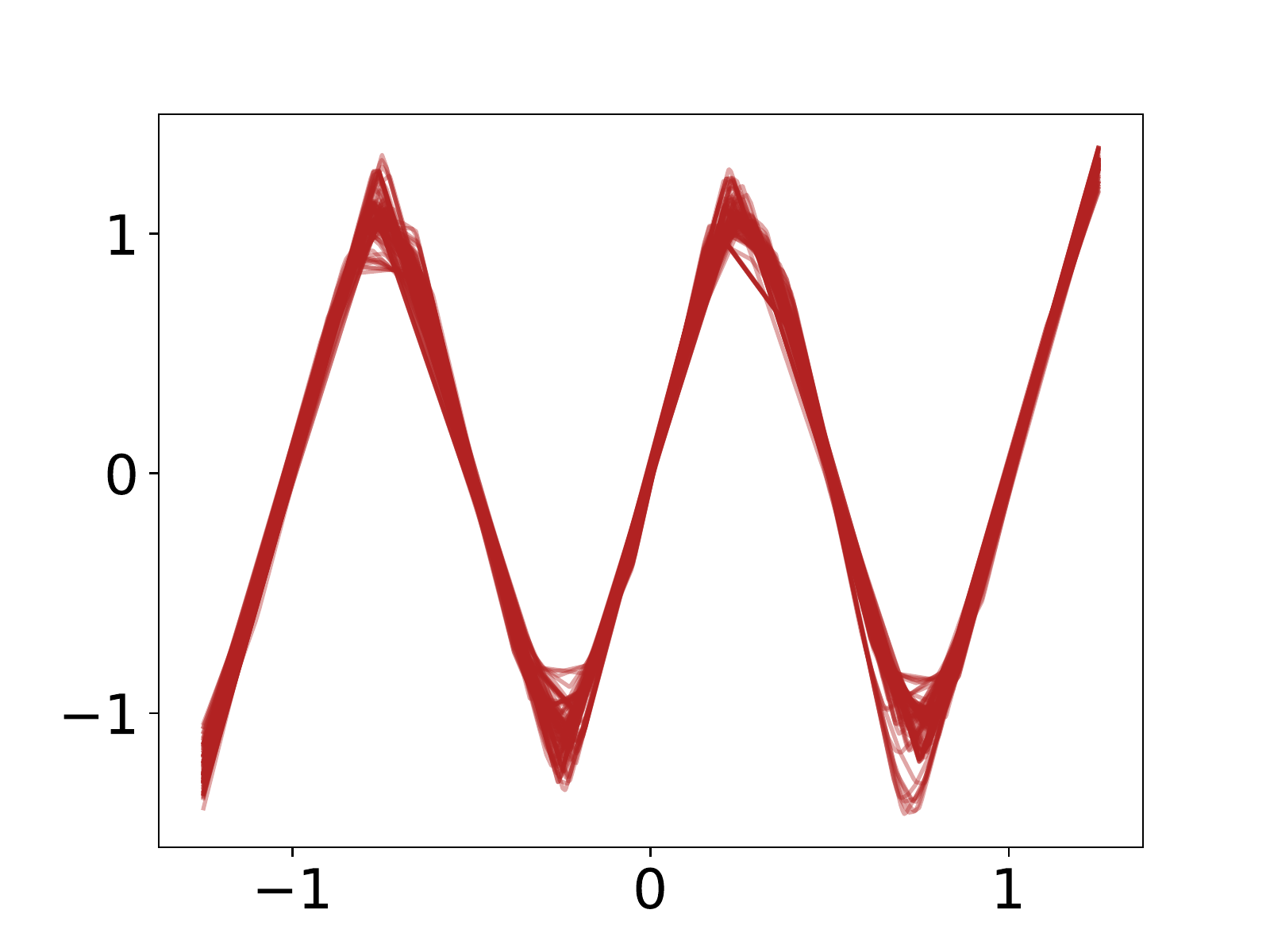}
    }
    \hfill
    \subfigure{
        \includegraphics[width=.31\textwidth]{figures/sinusoids/all_functions_vis/width1000}
    }
    \hfill
    \subfigure{
        \includegraphics[width=.31\textwidth]{figures/sinusoids/all_functions_vis/width10000}
    }
    \caption[Visualization of 100 different functions learned by different width neural networks]{Visualization of 100 different functions learned by the different width neural networks. Darker color indicates higher density of different functions. Widths in increasing order from left to right and top to bottom: 5, 10, 15, 17, 20, 22, 25, 35, 75, 100, 1000, 10000. We do \textit{not} observe the caricature from \cref{fig:high_var_caricature} as width is increased.}
    \label{fig:app_all_learned_sinusoids}
\end{figure}

\begin{figure}[H]
    \centering
    \subfigure{
        \includegraphics[width=.31\textwidth]{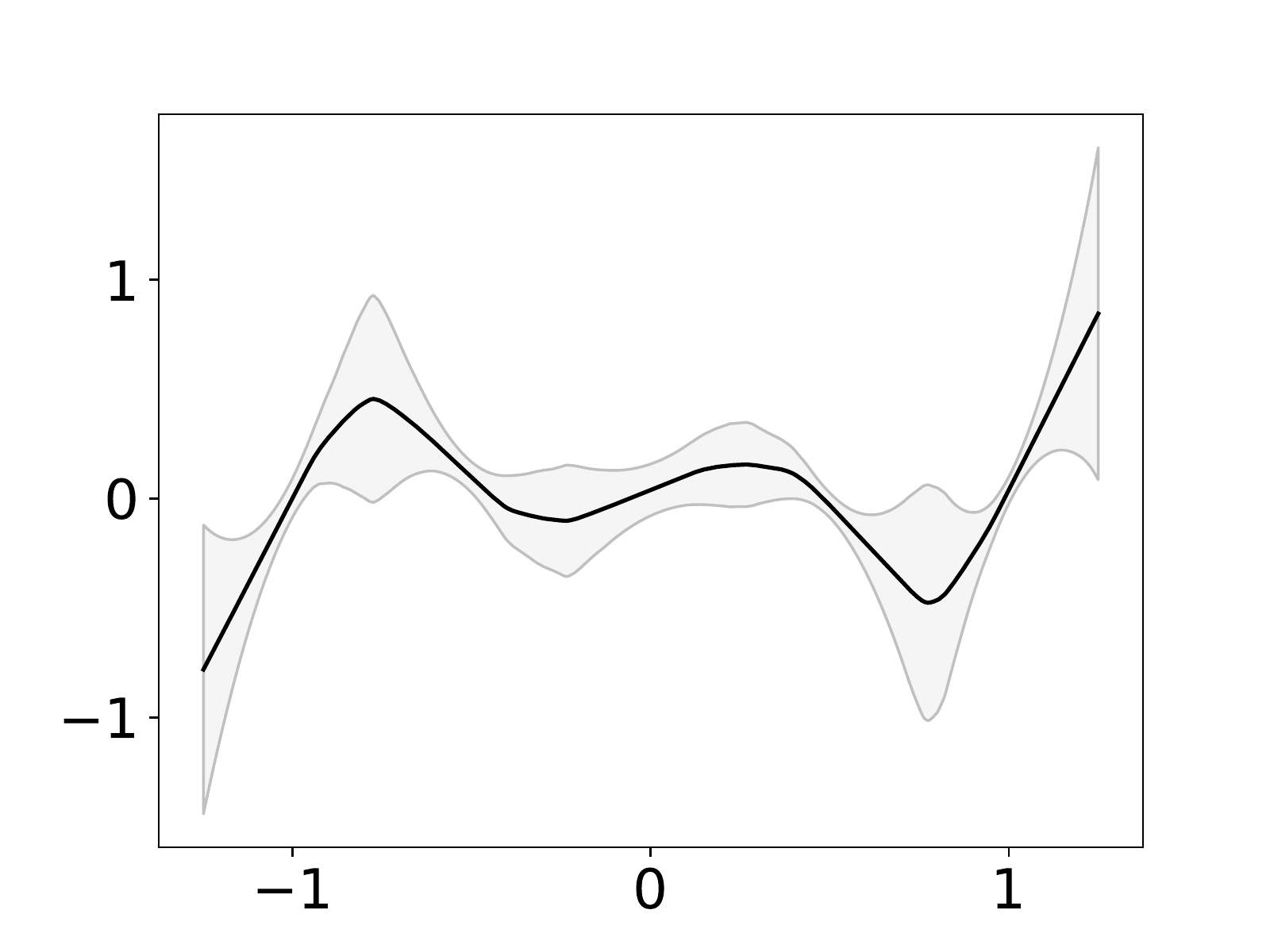}
    }
    \hfill
    \subfigure{
        \includegraphics[width=.31\textwidth]{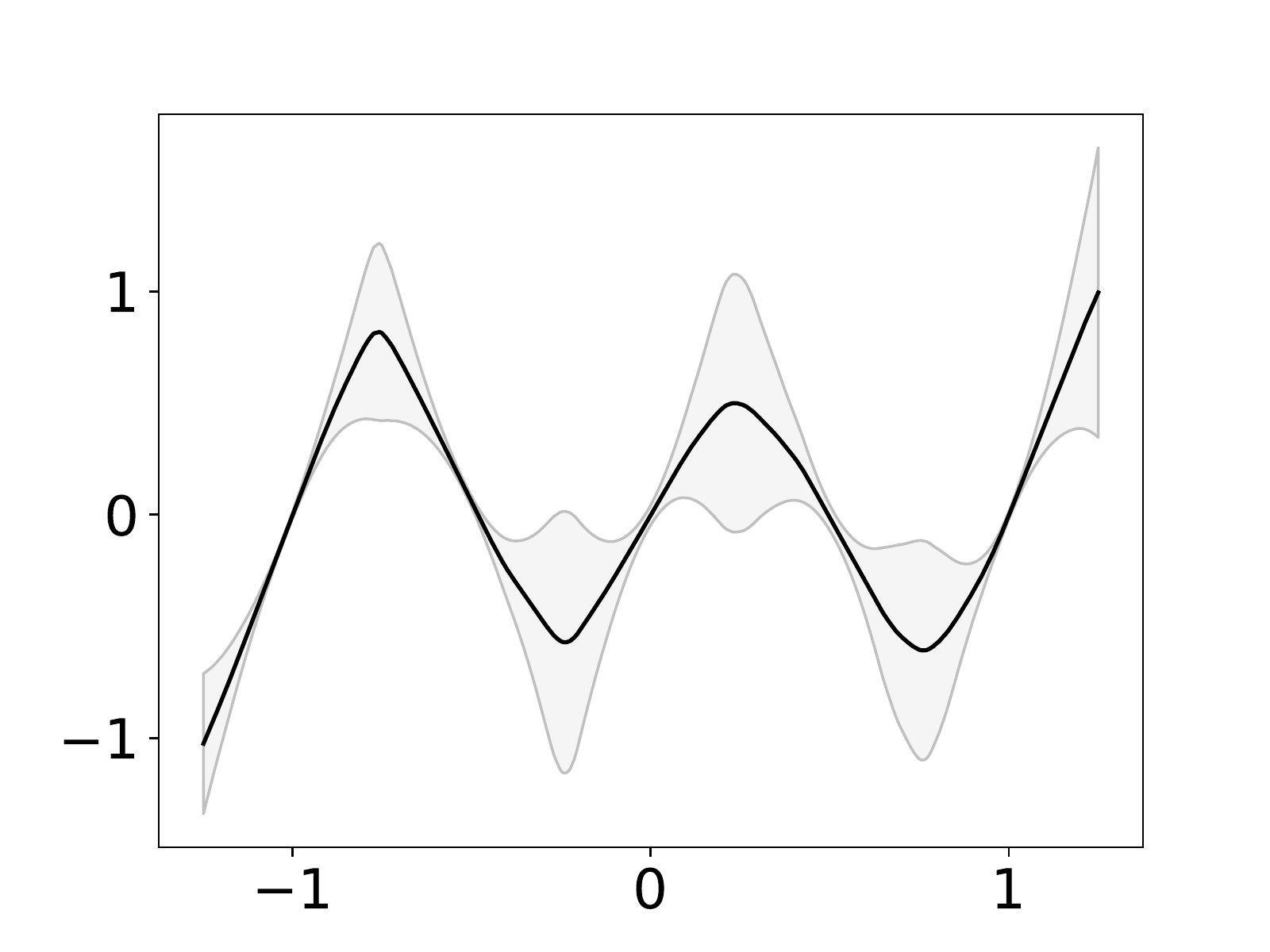}
    }
    \hfill
    \subfigure{
        \includegraphics[width=.31\textwidth]{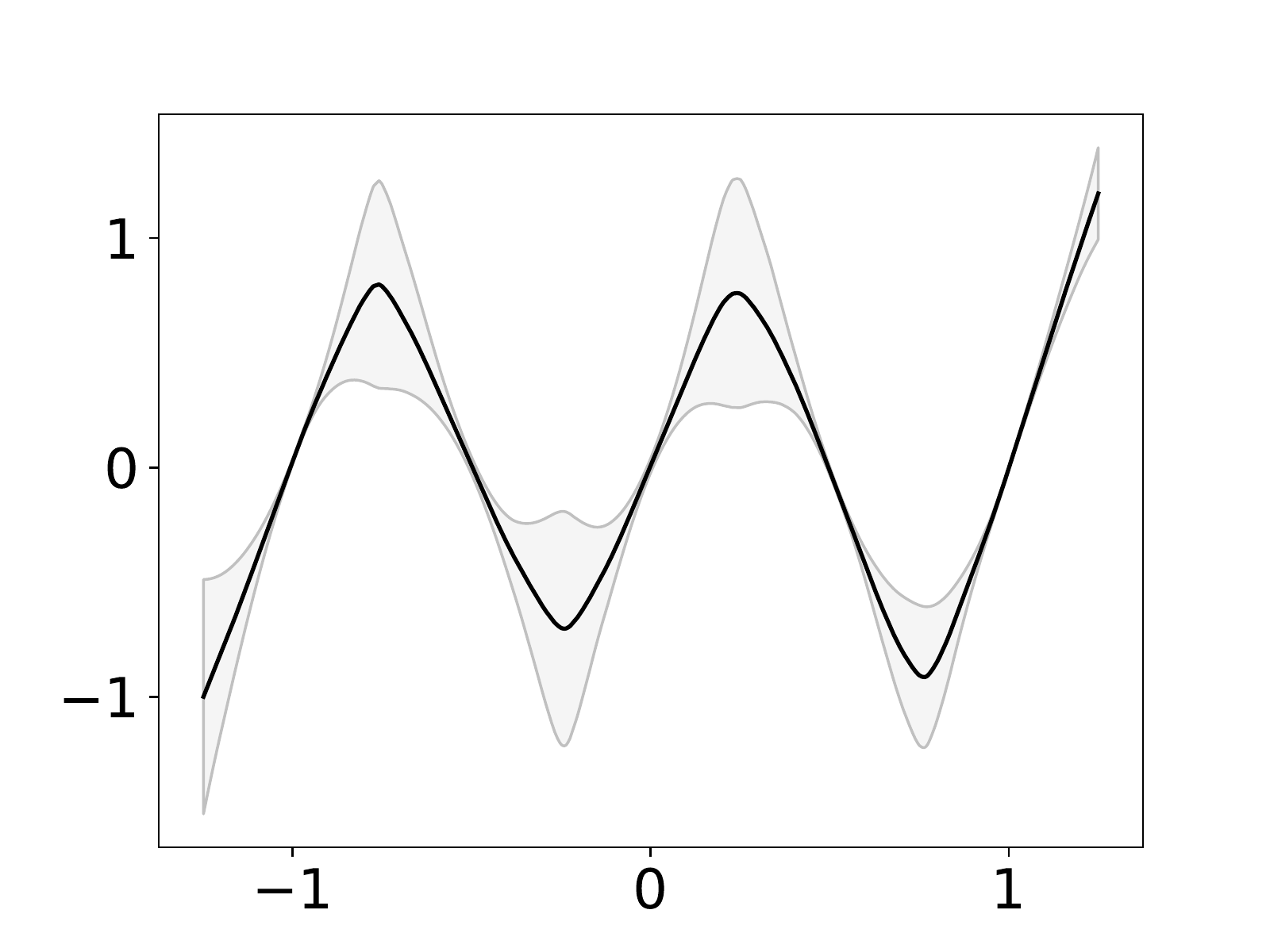}
    }
    
    \subfigure{
        \includegraphics[width=.31\textwidth]{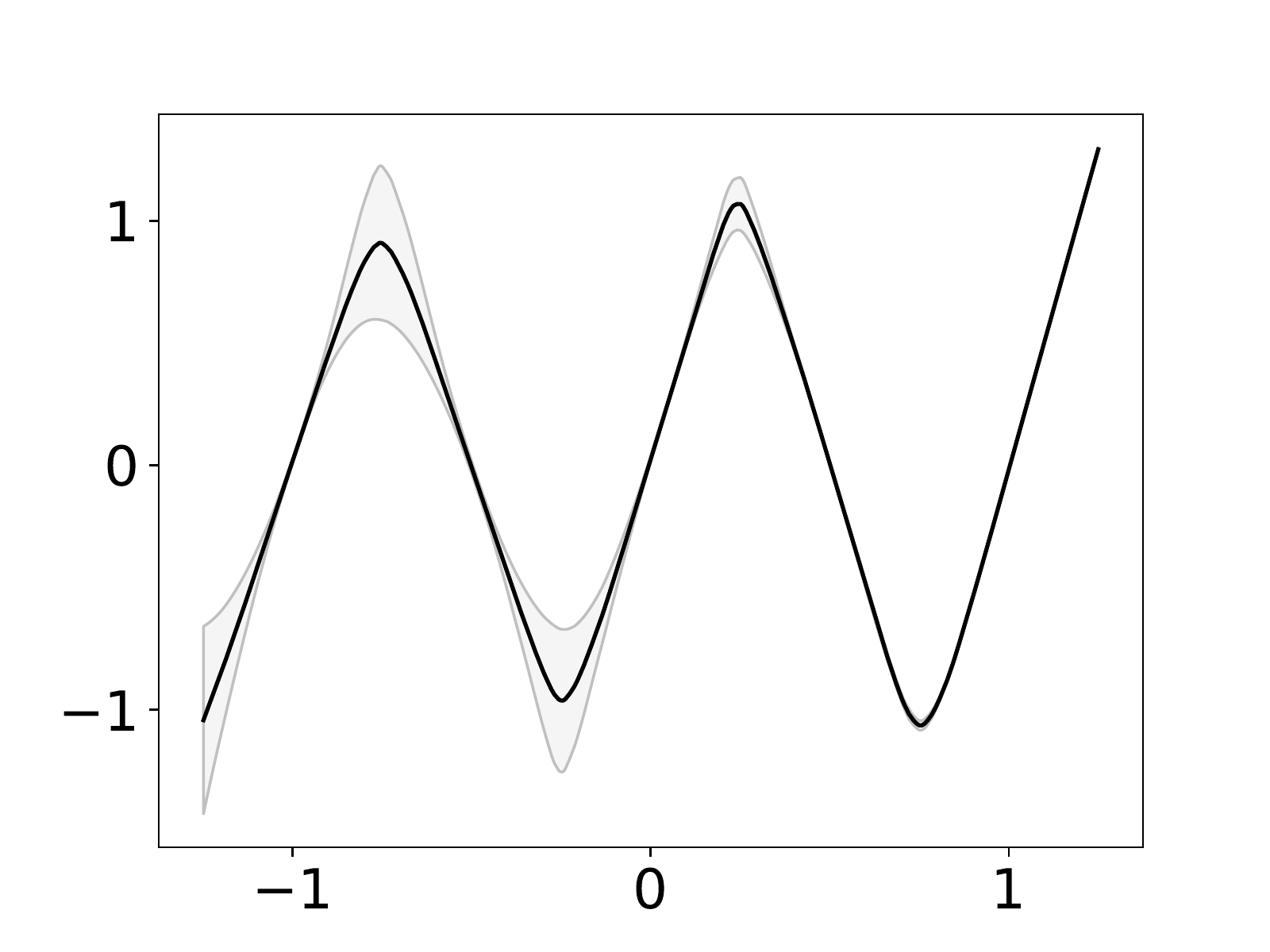}
    }
    \hfill
    \subfigure{
        \includegraphics[width=.31\textwidth]{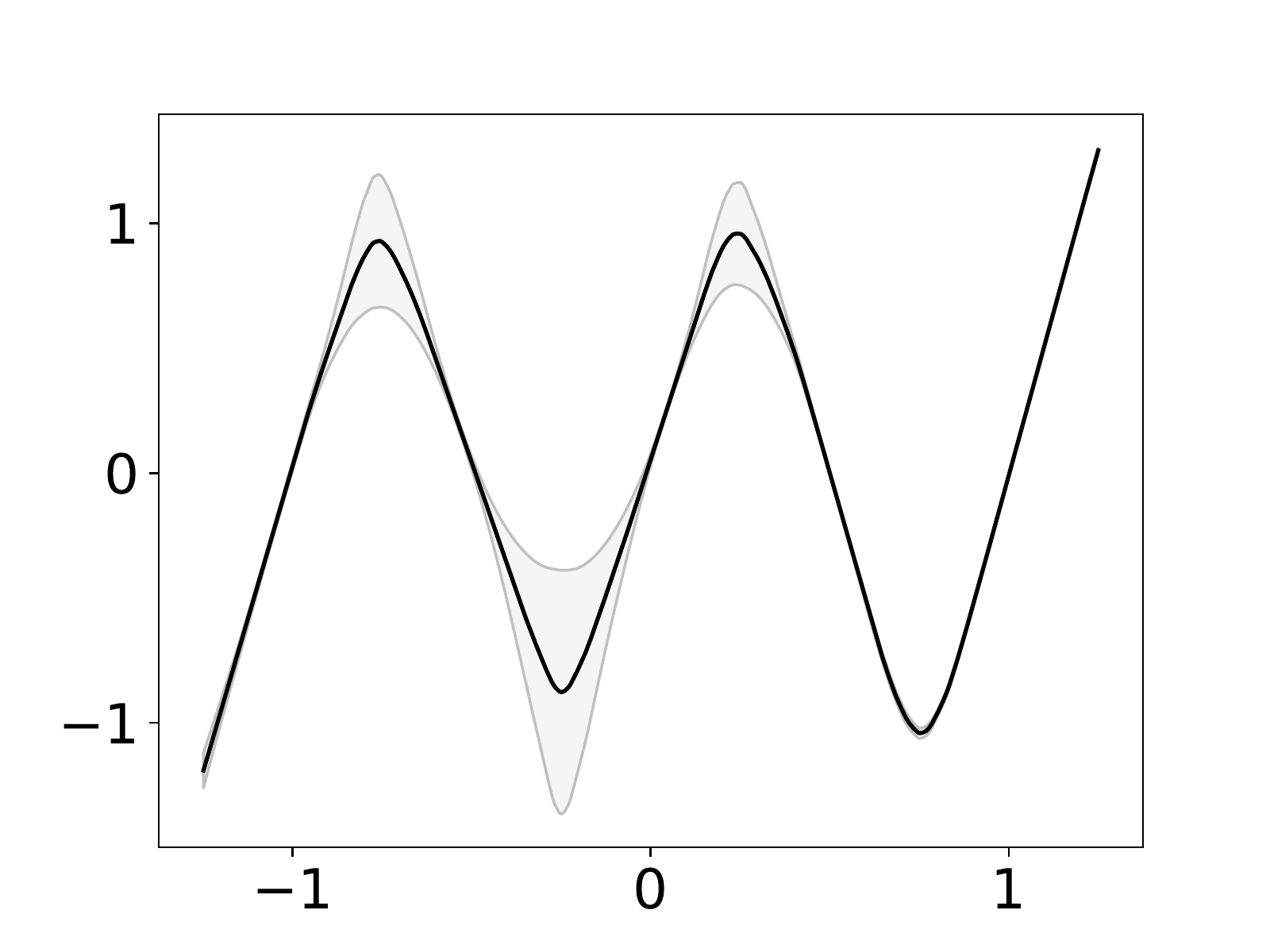}
    }
    \hfill
    \subfigure{
        \includegraphics[width=.31\textwidth]{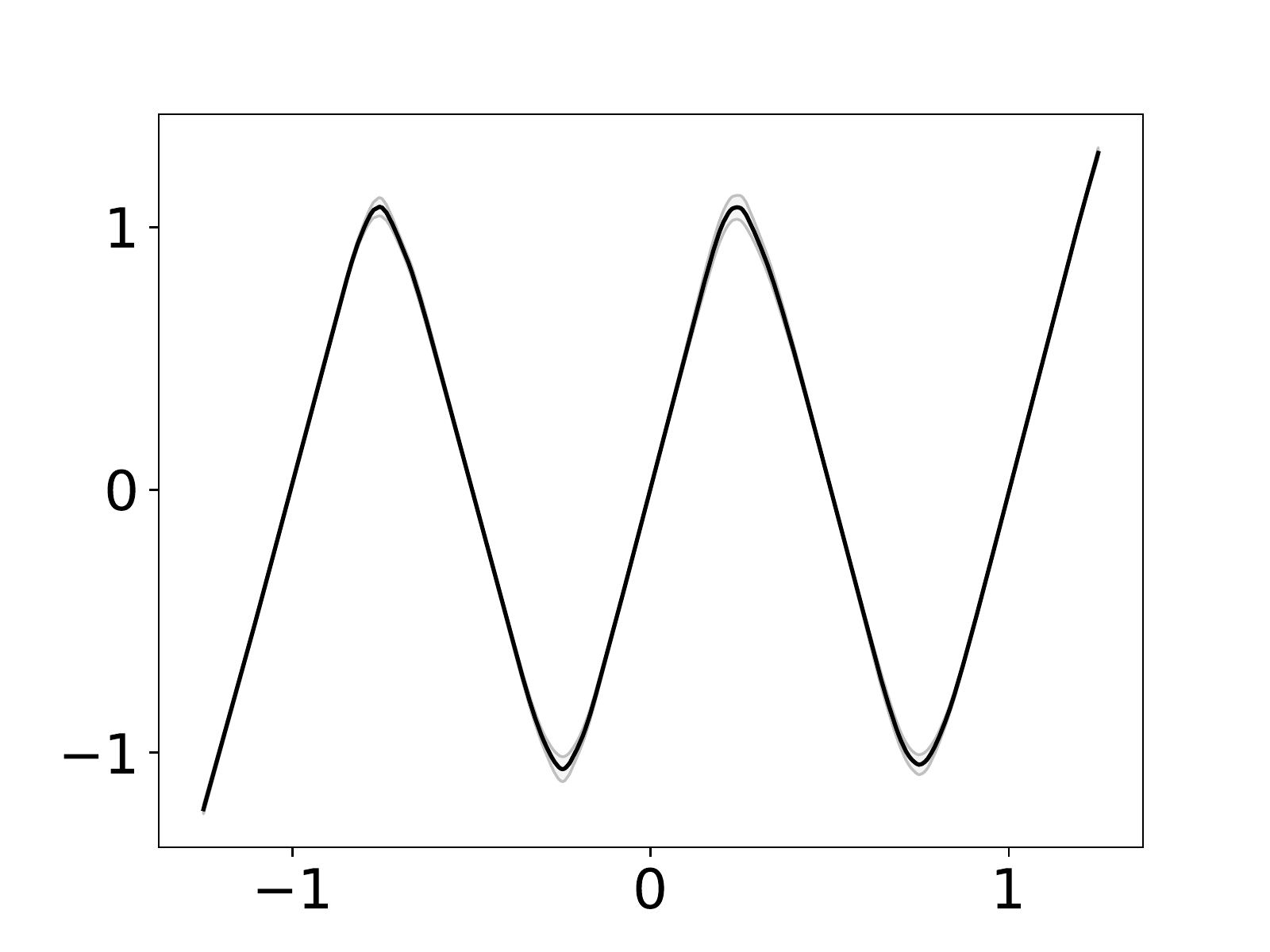}
    }
    
    \subfigure{
        \includegraphics[width=.31\textwidth]{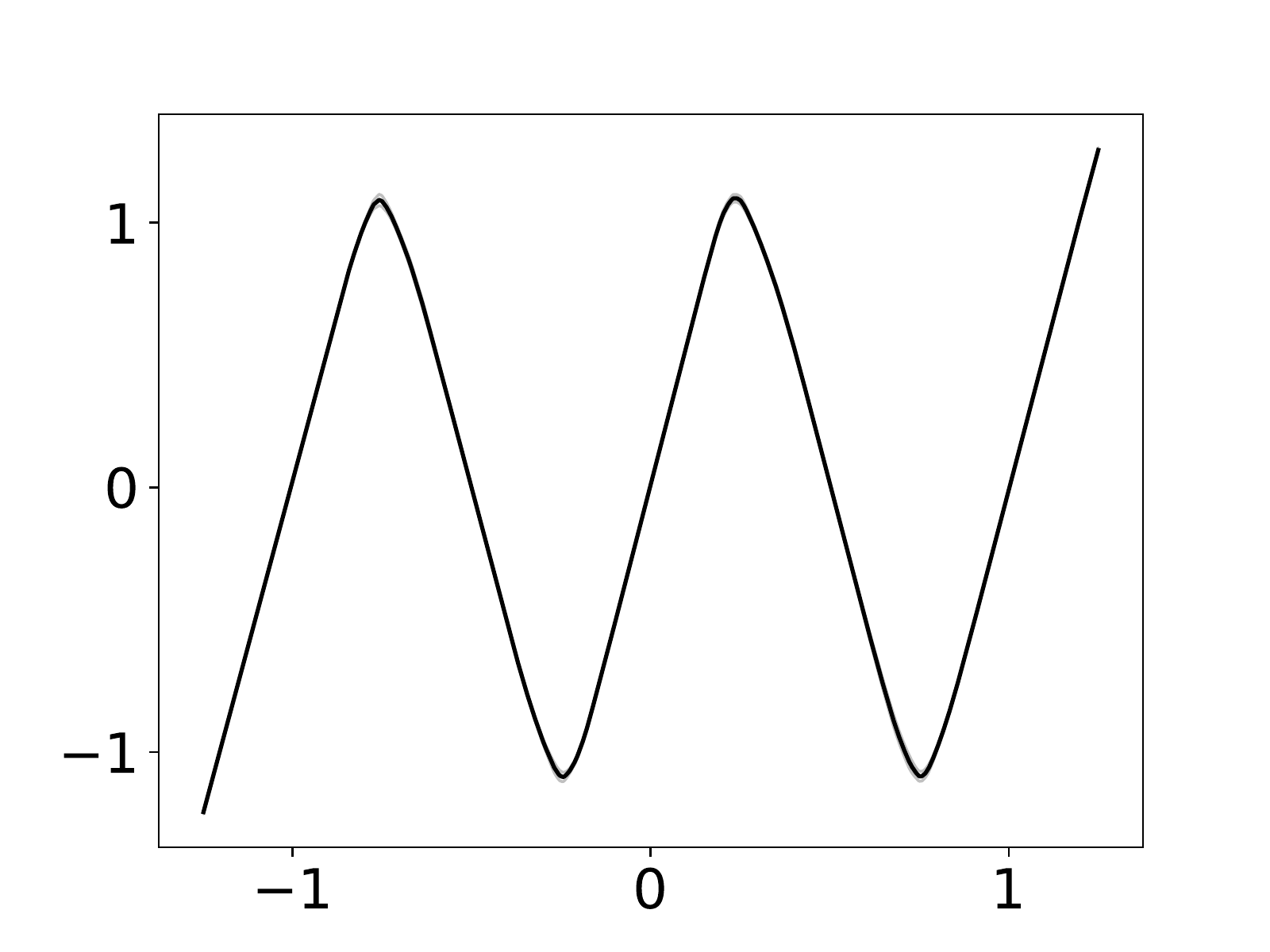}
    }
    \hfill
    \subfigure{
        \includegraphics[width=.31\textwidth]{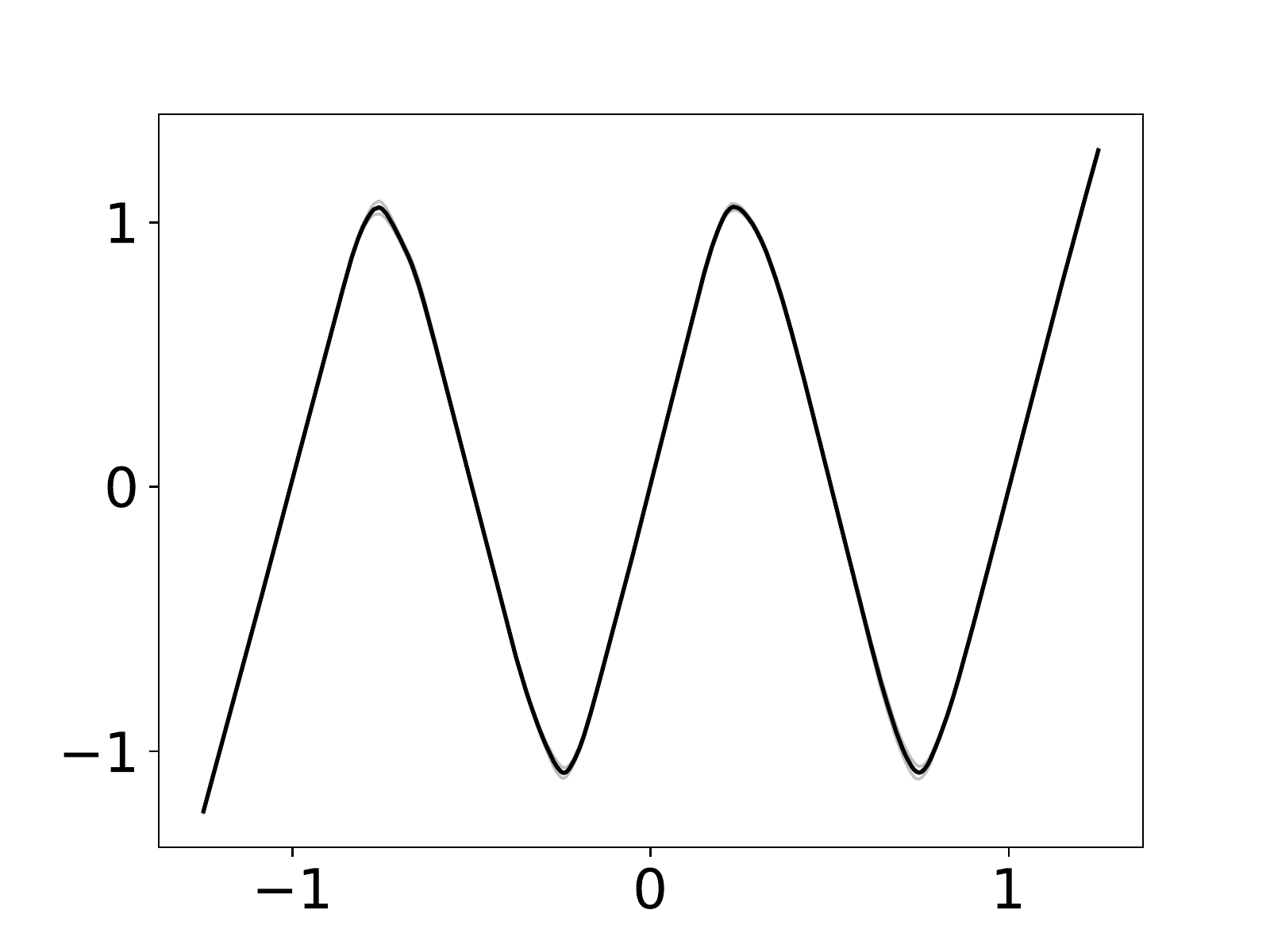}
    }
    \hfill
    \subfigure{
        \includegraphics[width=.31\textwidth]{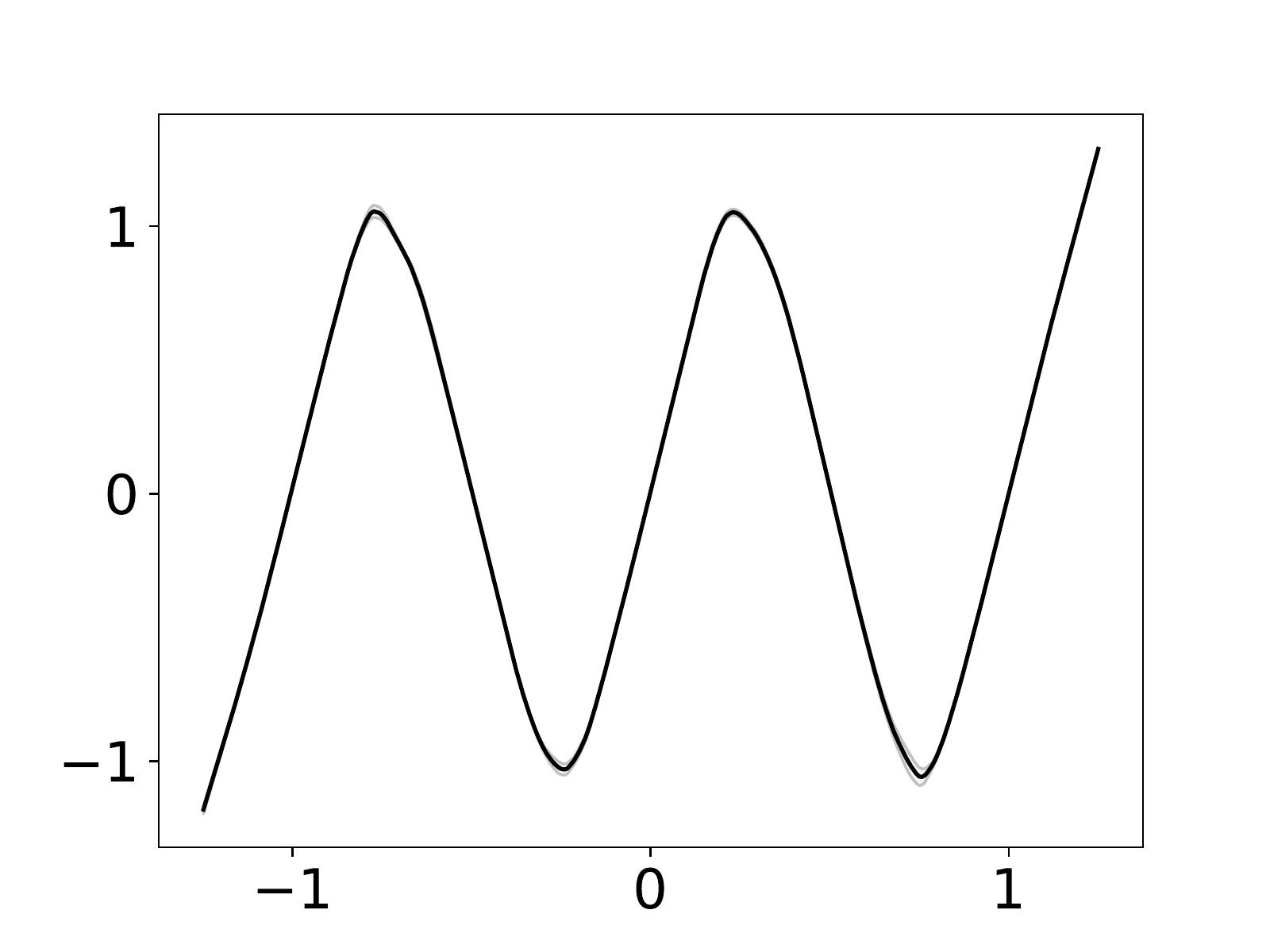}
    }
    
    \subfigure{
        \includegraphics[width=.31\textwidth]{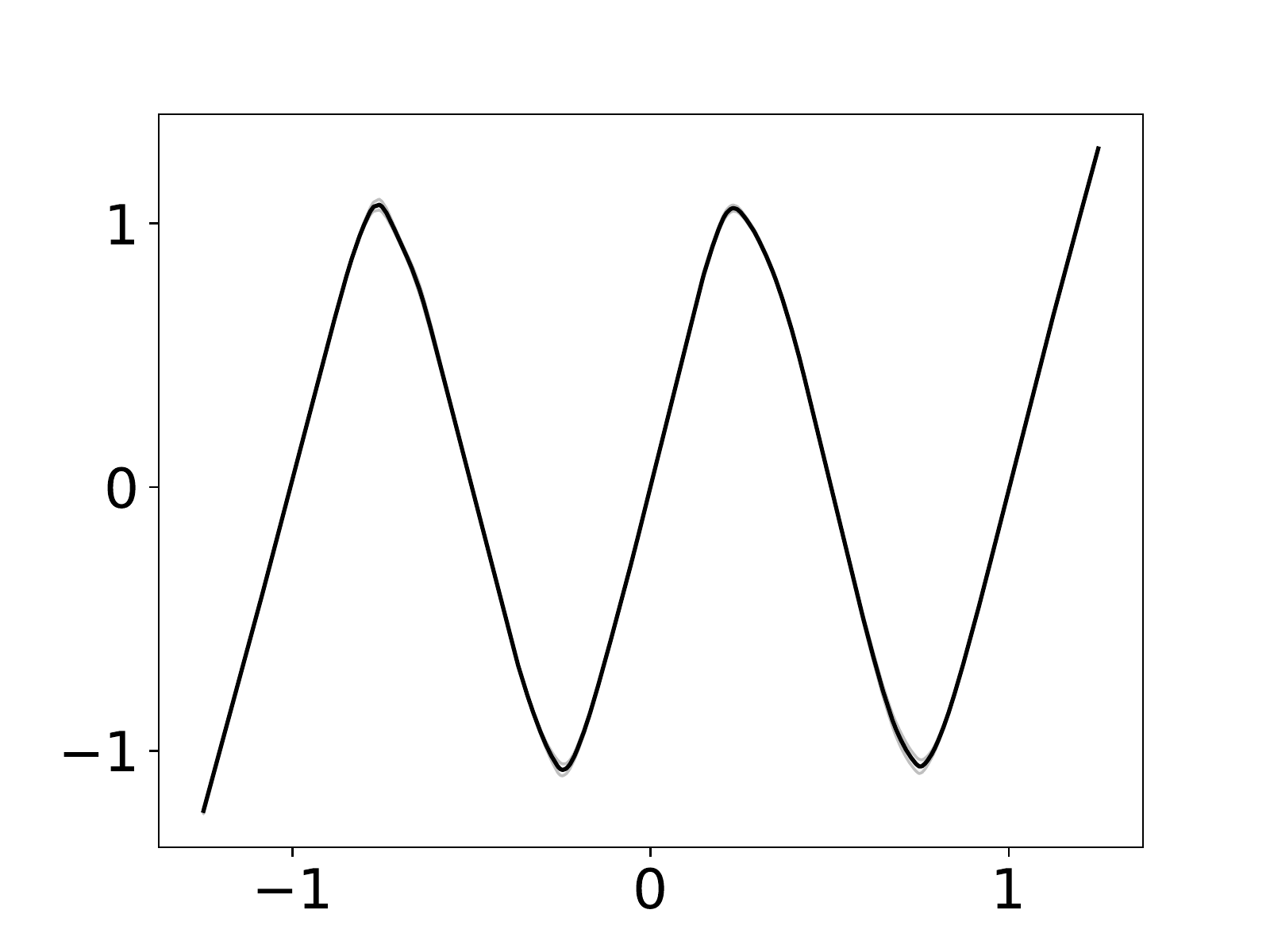}
    }
    \hfill
    \subfigure{
        \includegraphics[width=.31\textwidth]{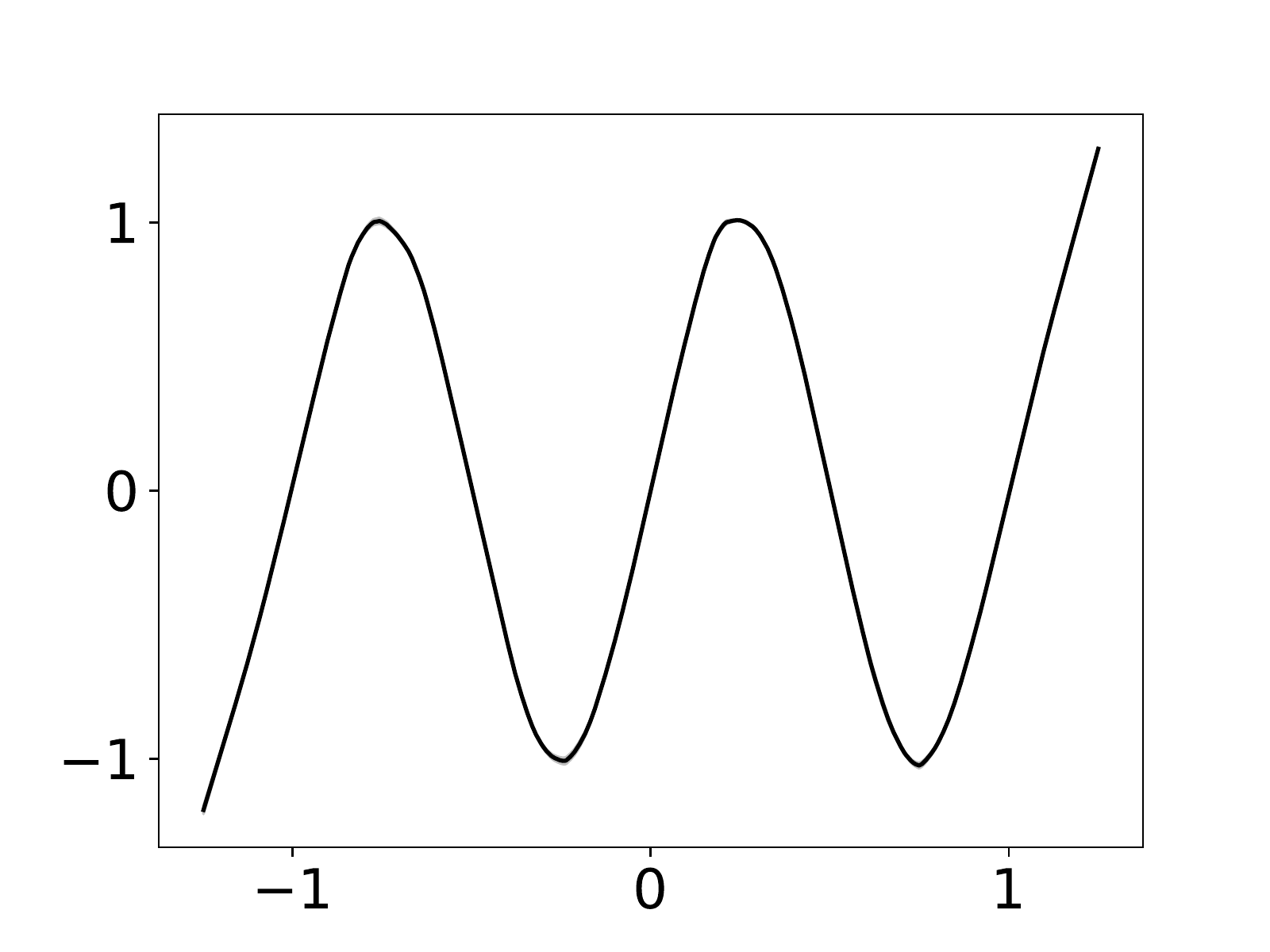}
    }
    \hfill
    \subfigure{
        \includegraphics[width=.31\textwidth]{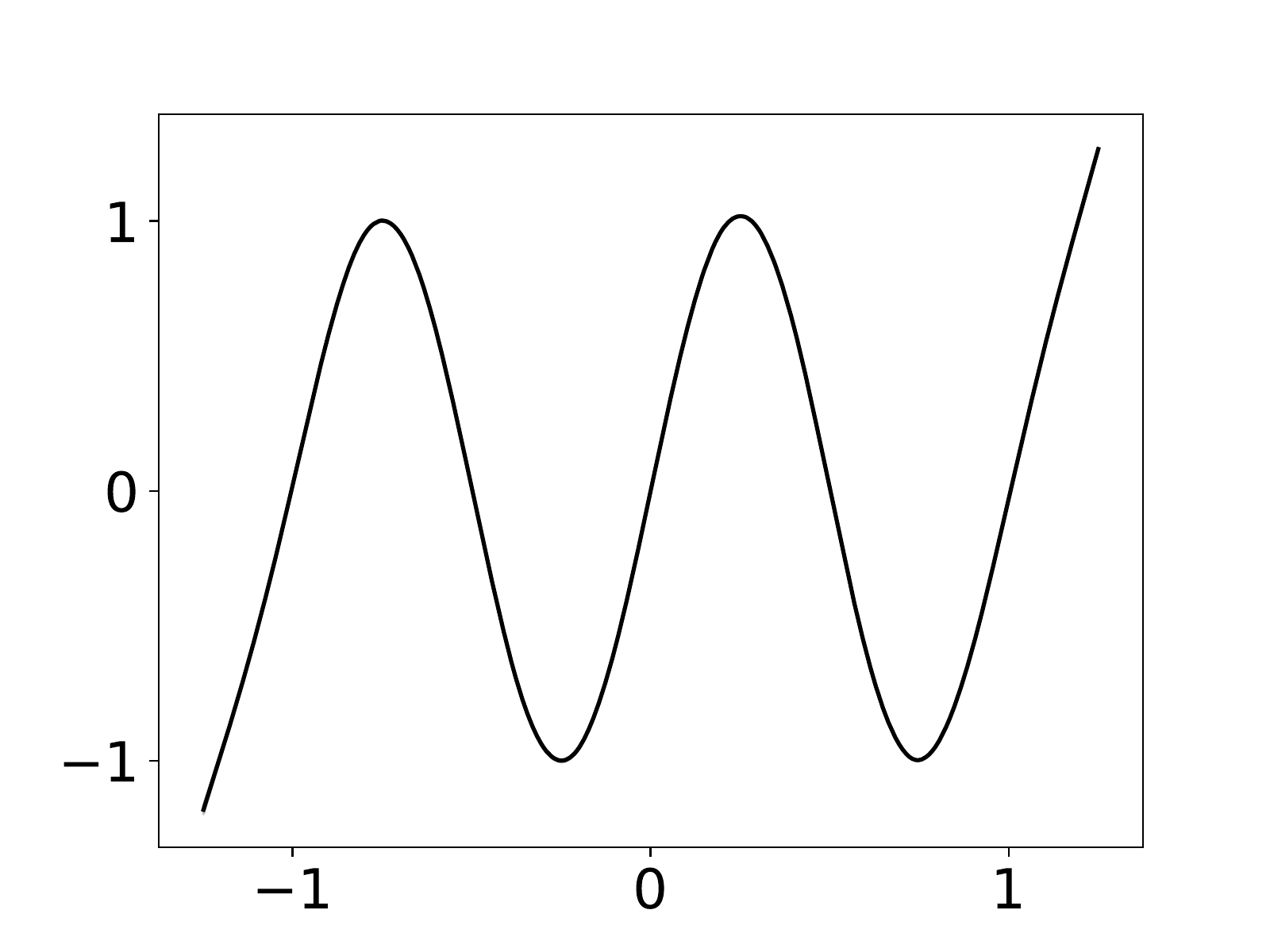}
    }
    \caption[Visualization of the mean prediction and variance of different width networks]{Visualization of the mean prediction and variance of the different width neural networks. Widths in increasing order from left to right and top to bottom: 5, 10, 15, 17, 20, 22, 25, 35, 75, 100, 1000, 10000.}
    \label{fig:app_mean_and_var_vis}
\end{figure}

\begin{figure}[H]
    \centering
    \subfigure{
        \includegraphics[width=.475\textwidth]{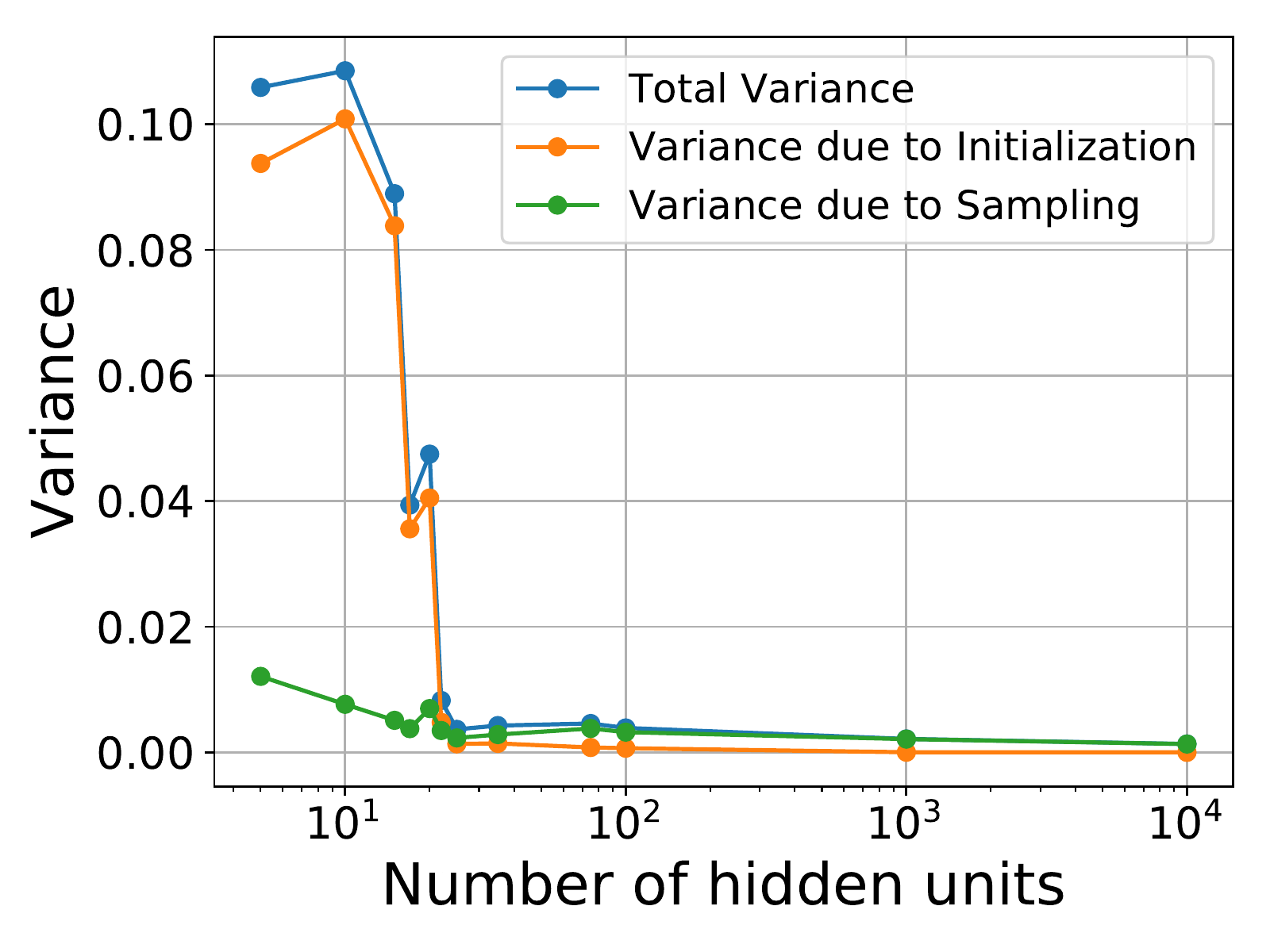}
    }
    \hfill
    \subfigure{
        \includegraphics[width=.475\textwidth]{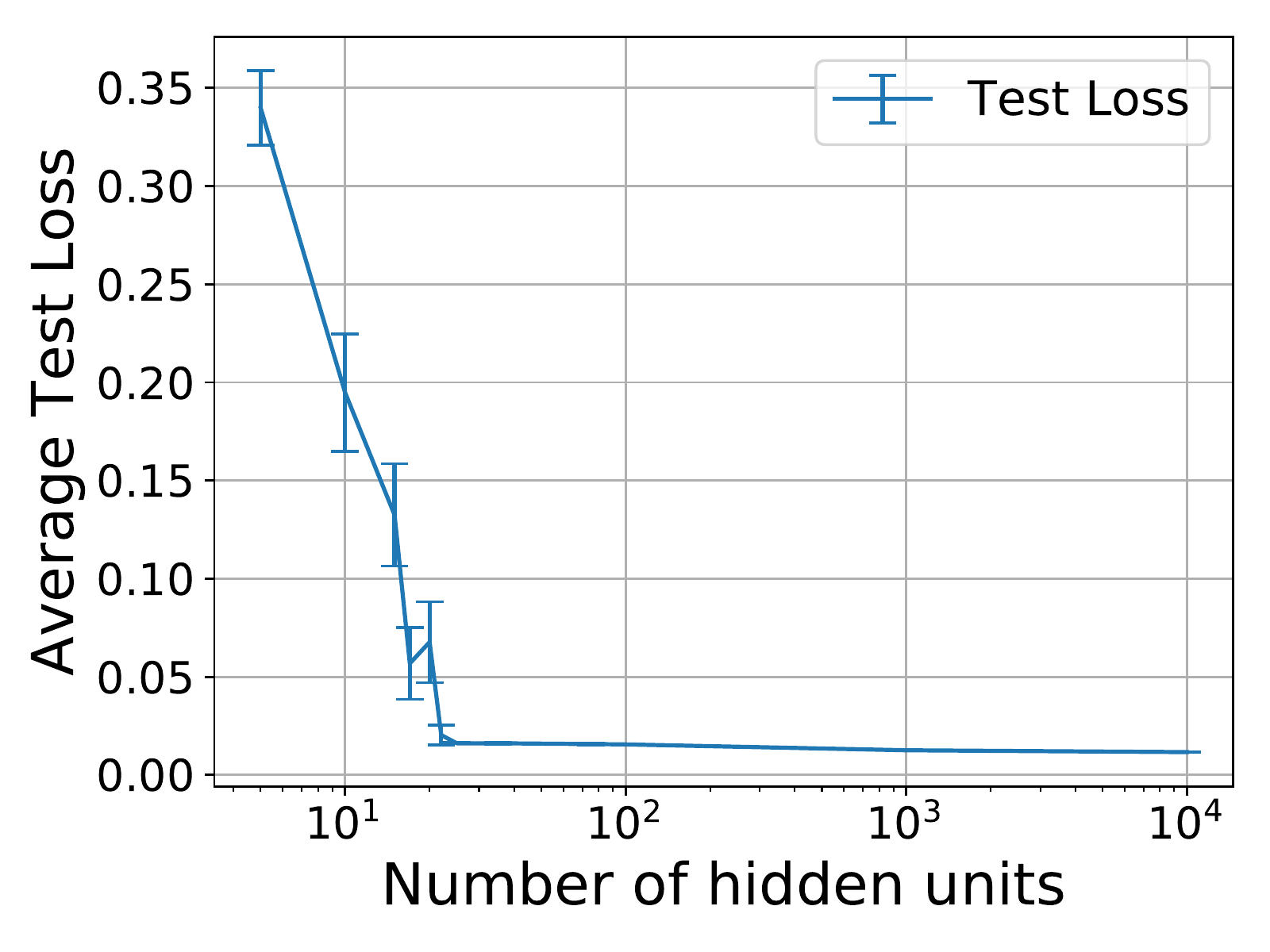}
    }
    \caption[Sinusoid regression task: decomposed variance and test error]{We observe the same trends of decomposed variance (left) and test error (right) in the sinusoid regression setting.}
    \label{fig:sinusoid_curves_app}
\end{figure}

\chapter{Depth and variance}
\label{app:depth}

\section{Main graphs}

\begin{figure}[H]
    \centering
    \subfigure[Bias and variance trends with depth, using dynamical isometry]{
        \includegraphics[width=.475\textwidth]{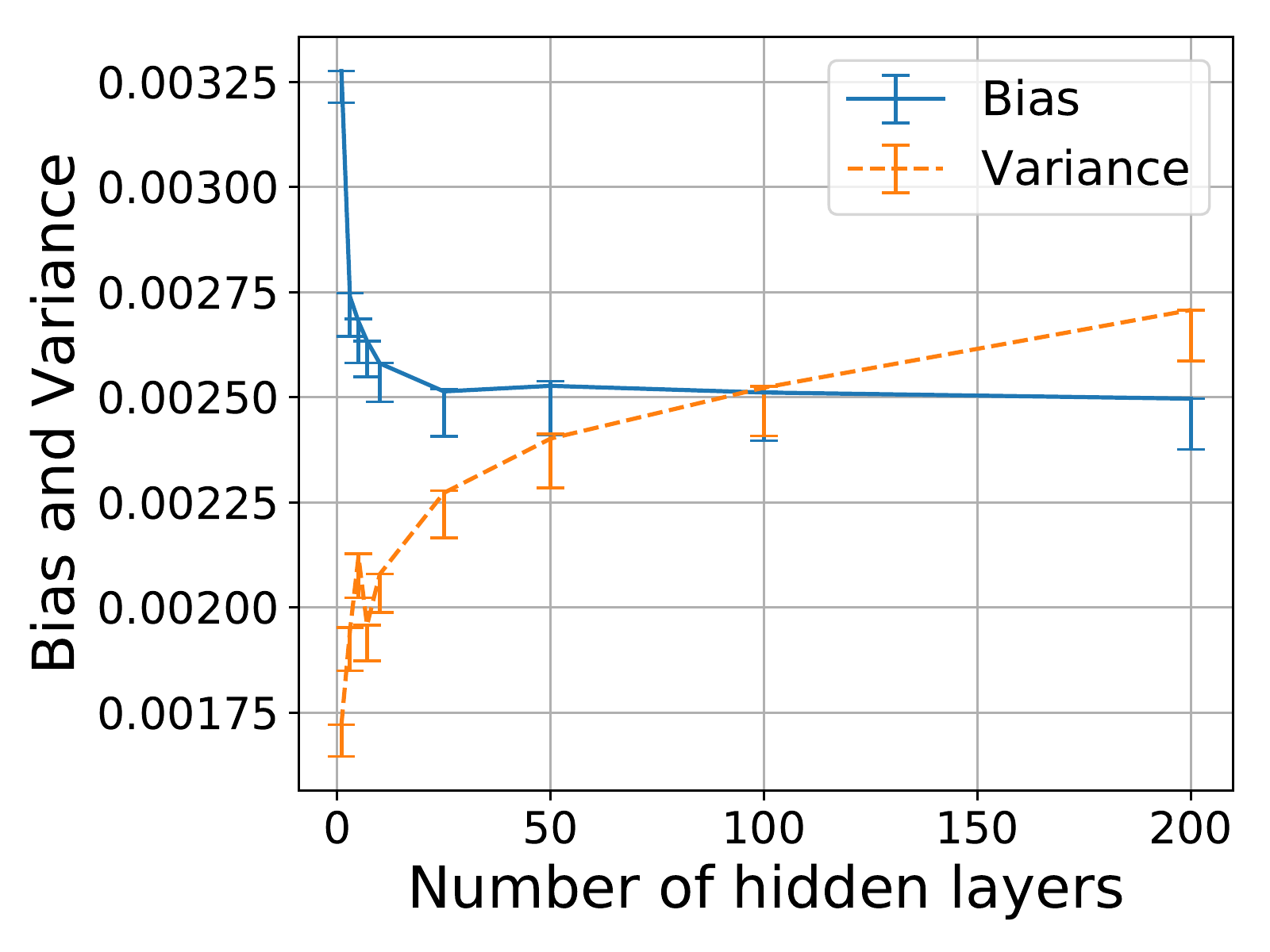}
        \label{fig:dyn_iso_bv}
    }
    \hfill
    \subfigure[Test error trends, using dynamical isometry vs.\ skip connections]{
        \includegraphics[width=.475\textwidth]{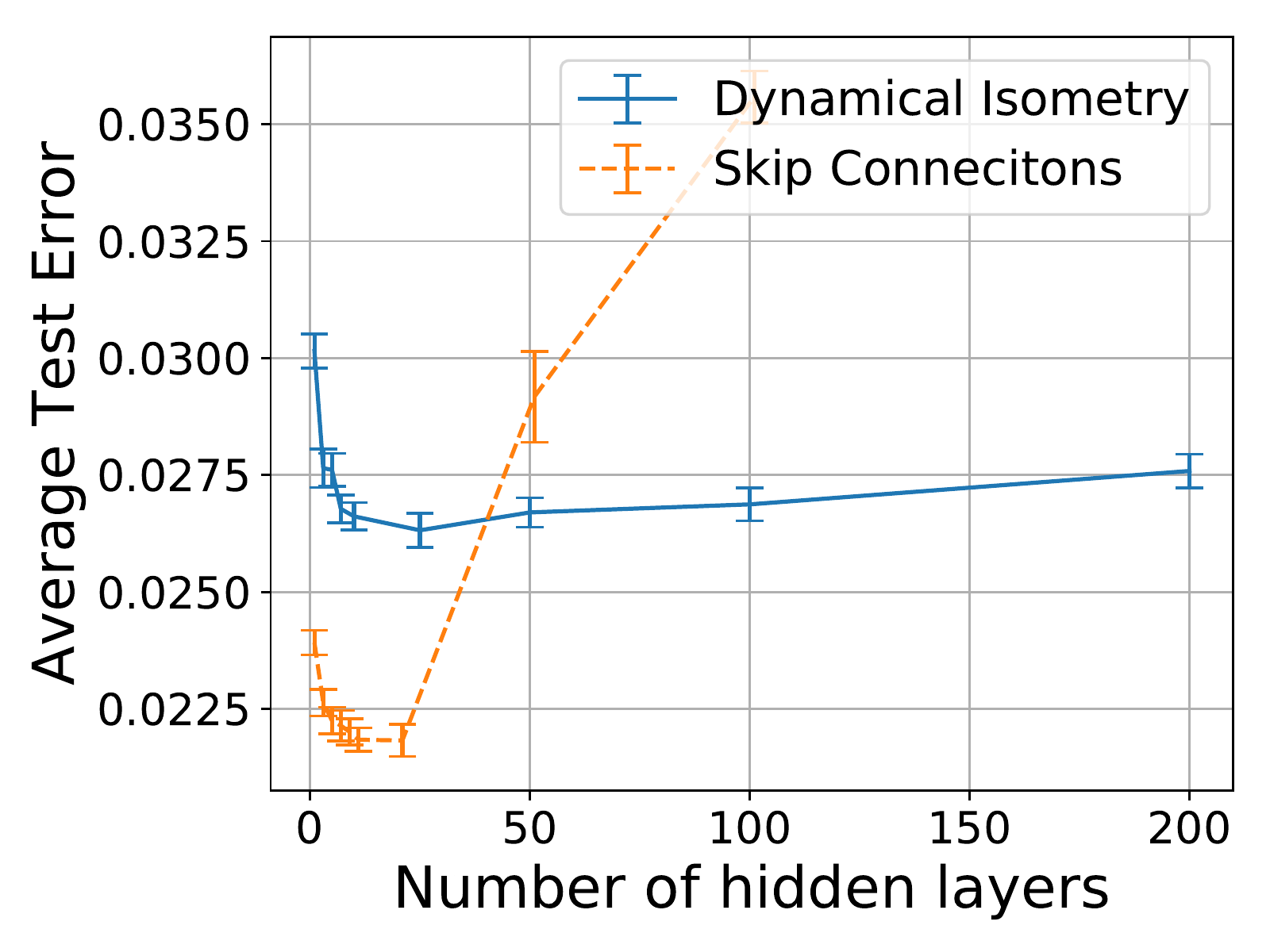}
        \label{fig:depth_test_error}
    }
    \caption[Main bias-variance with increasing network depth experiment]{We can see that, when using dynamical isometry, bias decreases with depth and variance slowly increases with depth (left). This increase in variance is so small that it only translates to a an increase in test error of about 0.1\% for depth 25 to depth 200 (right).}
\end{figure}

\section{Discussion on need for careful experimental design}
\label{app:depth_exp_protocol_discussion}

Depth is an important component of deep learning. We study its effect on bias and variance by fixing width and varying depth. However, there are pathological problems associated with training very deep networks such as vanishing/exploding gradient \citep{Hochreiter:91, long-term_dependencies, xavier2010}, signal not being able to propagate through the network \citep{deep_info_prop}, and gradients resembling white noise \citep{shattered_gradient}. \citet{resnet1} pointed out that very deep networks experience high test set error and argued it was due to high training set loss. However, while skip connections \citep{resnet1}, better initialization \citep{xavier2010}, and batch normalization \citep{batchnorm} have largely served to facilitate low training loss in very deep networks, the problem of high \textit{test set} error still remains.

The current best practices for achieving low test error in very deep networks arose out of trying to solve the above problems in training. An initial step was to ensure the mean squared singular value of the input-output Jacobian, at initialization, is close to 1 \citep{xavier2010}. More recently, there has been work on a stronger condition known as \textit{dynamical isometry}, where \textit{all} singular values remain close to 1 \citep{Saxe14exactsolutions, resurrecting_sigmoid}. \citet{resurrecting_sigmoid} also empirically found that dynamical isometry helped achieve low test set error. Furthermore, \citet[Figure 1]{xiao18} found evidence that test set performance did not degrade with depth when they lifted dynamical isometry to CNNs. This why we settled on dynamical isometry as the best known practice to control for as many confounding factors as possible.

We first ran experiments with vanilla full connected networks (\cref{fig:depth_vanilla}). These have clear training issues where networks of depth more than 20 take very long to train to the target training loss of \mbox{5e-5}. The bias curve is not even monotonically decreasing. Clearly, there are important confounding factors not controlled for in this simple setting. Still, note that variance increases roughly linearly with depth.

We then study fully connected networks with skip connections between every 2 layers (\cref{fig:depth_skip}). While this allows us to train deeper networks than without skip connections, many of the same issues persist (e.g. bias still not monotonically decreasing). The bias, variance, and test error curves are all checkmark-shaped.

\section{Vanilla fully connected depth experiments}
\label{app:depth_vanilla}

\begin{figure}[H]
    \centering
    \subfigure{
        \includegraphics[width=.475\textwidth]{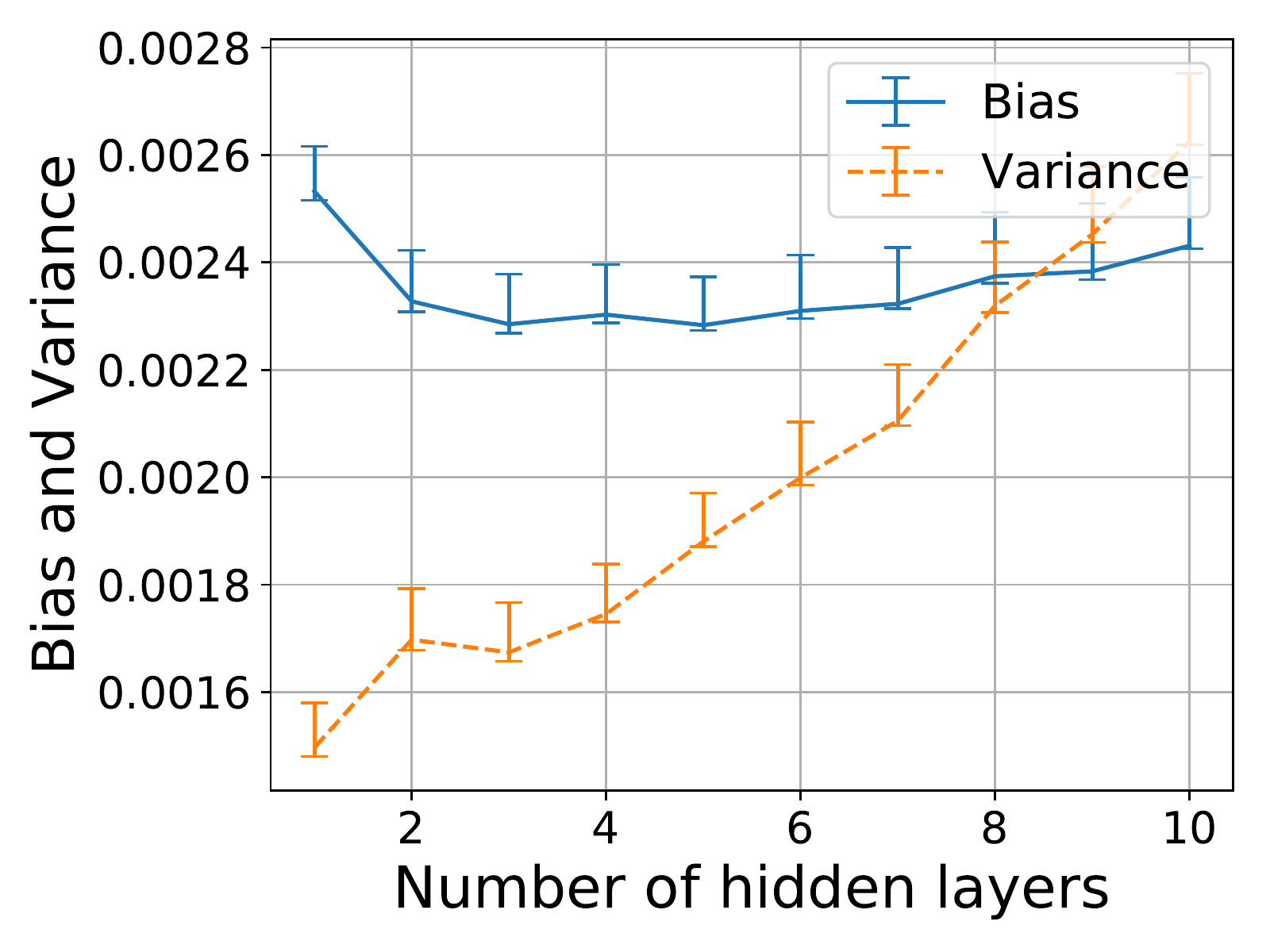}
    }
    \hfill
    \subfigure{
        \includegraphics[width=.475\textwidth]{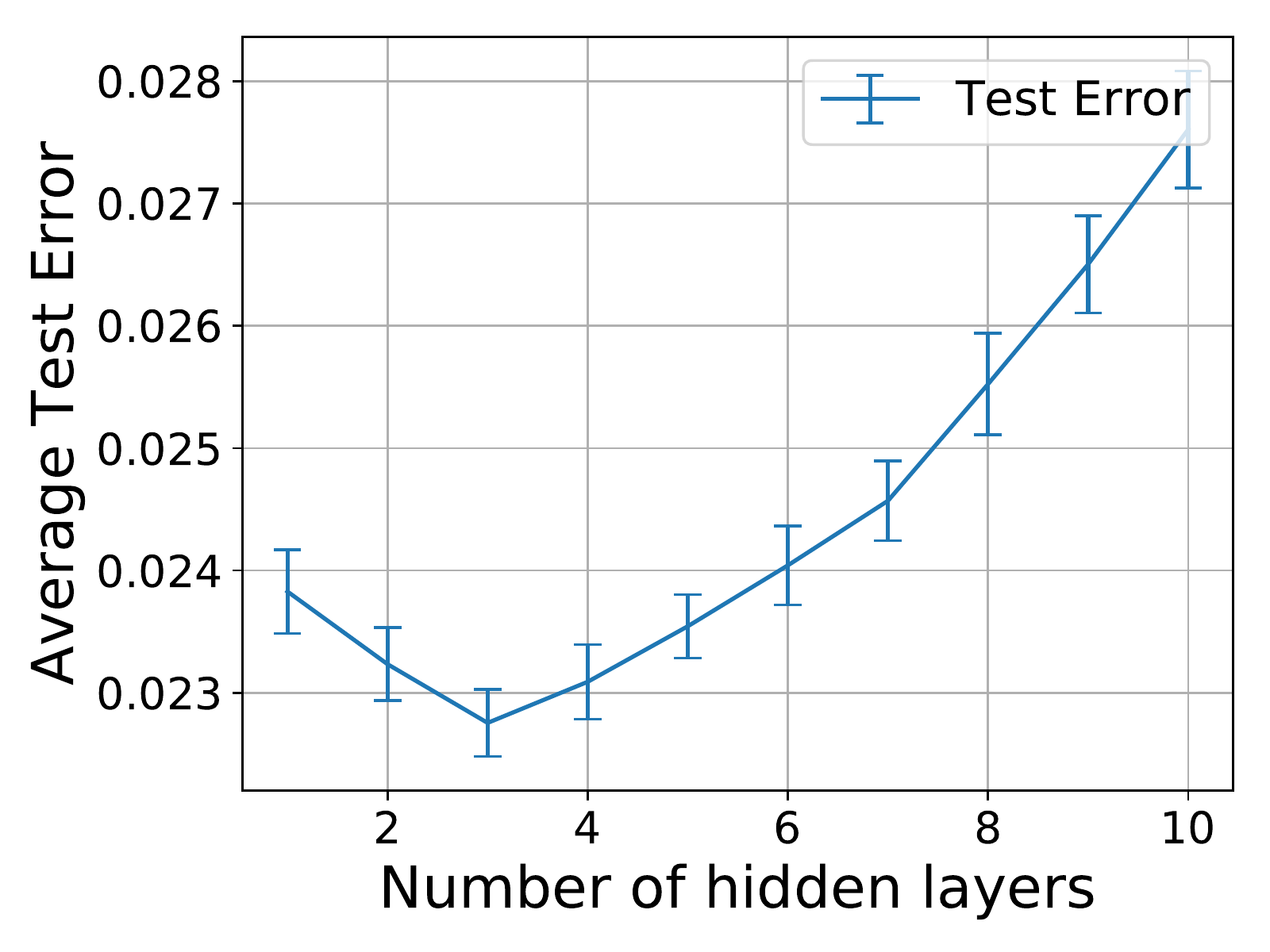}
    }
    \caption[Increasing depth in vanilla fully connected networks]{Test error quickly degrades in fairly shallow fully connected networks, and bias does not even monotonically decrease with depth. However, this is the first indication that variance might \textit{increase} with depth. All networks have training error 0 and are trained to the same training loss of 5e-5.}
    \label{fig:depth_vanilla}
\end{figure}

\section{Skip connections depth experiments}
\label{app:depth_skip}

\begin{figure}[H]
    \centering
    \subfigure{
        \includegraphics[width=.475\textwidth]{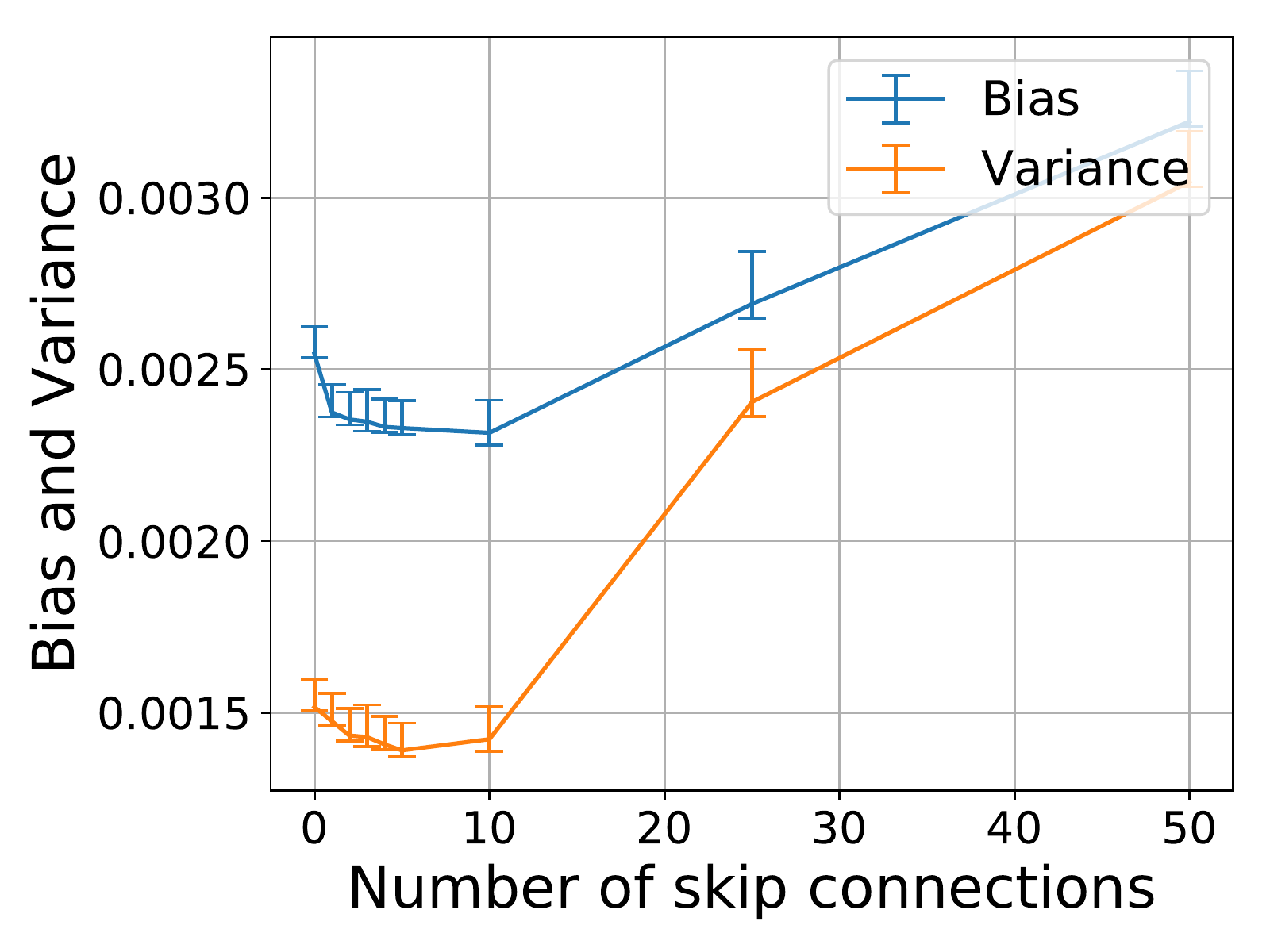}
    }
    \hfill
    \subfigure{\includegraphics[width=.475\textwidth]{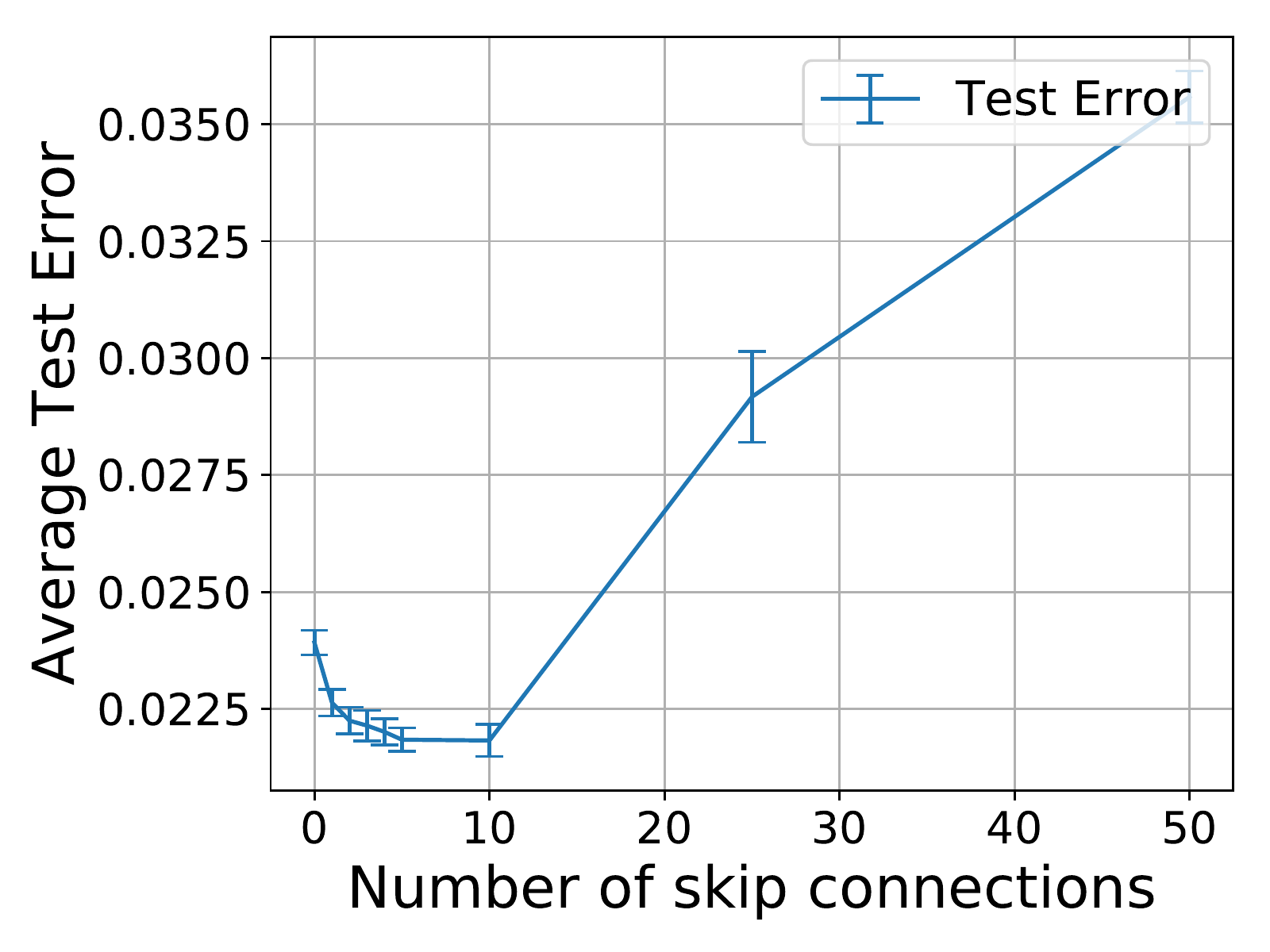}}
    \caption[Increasing depth in fully connected networks with skip connections]{While the addition of skip connections (between every other layer) might push the bottom of the U curve in test error out to 10 skip connections (21 layers), which is further than where the bottom was observed without skip connections (3 layers), test error still degrades noticeably in greater depths. Additionally, bias still does not even monotonically decrease with depth. While skip connections appear to have helped control for the factors we want to control, they were not completely satisfying. All networks have training error 0 and are trained to the same training loss of 5e-5.}
    \label{fig:depth_skip}
\end{figure}

\section{Dynamical isometry depth experiments}
\label{app:depth_dyn_iso}

The figures in this section are included in the main paper, but they are included here for comparison to the above and for completeness.

\begin{figure}[H]
    \centering
    \subfigure{
        \includegraphics[width=.475\textwidth]{figures/depth_dyn_iso/bias-variance}
    }
    \hfill
    \subfigure{
        \includegraphics[width=.475\textwidth]{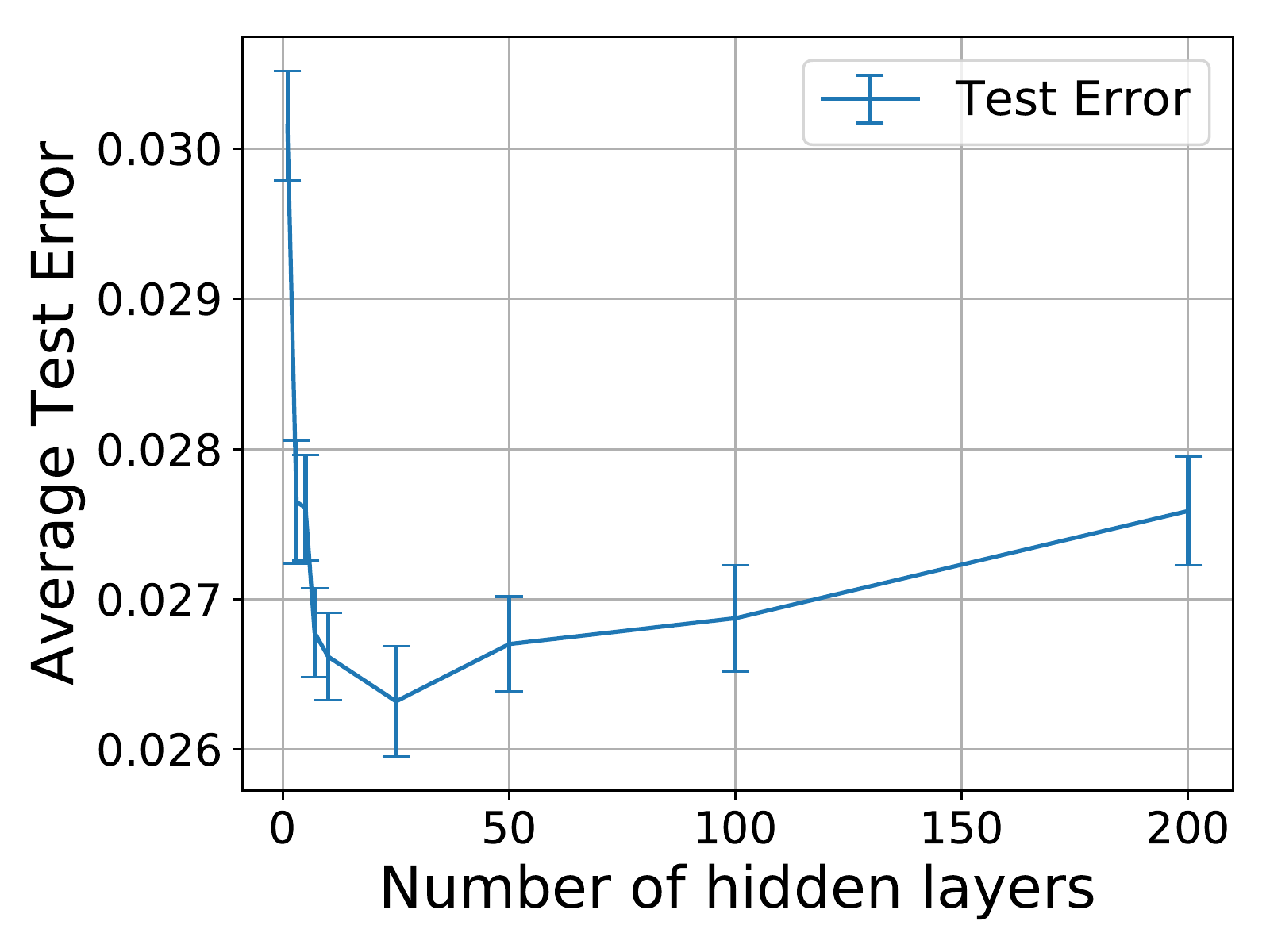}
    }
    \caption[Increasing depth in fully connected networks with dynamical isometry]{Additionally, dynamical isometry seems to cause bias to decrease monotonically with depth. While skip connections appear to have helped control for the factors we want to control, they were not completely satisfying. All networks have training error 0 and are trained to the same training loss of 5e-5.}
\end{figure}

\chapter{Some Proofs} 
\label{app:proofs}

\section{Proof of Classic Result for Variance of Linear Model}
\label{app:linear_underparam}

Here, we reproduce the classic result that variance grows with the number of parameters in a linear model. This result can be found in \citet{hastie_09}'s book, and a similar proof can be found in \citet{linear_var_lecture}'s lecture slides.
\begin{proof} 
For a fixed $x$, we have $h(x) = x^T \hat{\theta}$. 
Taking $\hat{\theta} = \Sigma^{-1} X^T Y$ to be the gradient descent solution, and using $Y= X\theta + \epsilon$, we obtain:
$$
h(x) = x^T \Sigma^{-1} X^T (X\theta + \epsilon) = 
x^T \theta + x^T \Sigma^{-1} X^T \epsilon
$$
Hence $\E_\epsilon[h(x)] = x^T \theta$, and the variance is,  
\begin{align*}
    \Var_\epsilon(h(x)) &= \E_\epsilon[(h(x) - \E_\epsilon[h(x)])^2] \\
    &= \E_\epsilon[(x^T \theta + x^T \Sigma^{-1} X^T \epsilon - x^T \theta)^2] \\
    &= \E_\epsilon[(x^T \Sigma^{-1} X^T \epsilon)^2] \\
    &= \E_\epsilon[(x^T \Sigma^{-1} X^T \epsilon)(x^T \Sigma^{-1} X^T \epsilon)^T] \\
    &= \E_\epsilon[x^T \Sigma^{-1} X^T \epsilon \epsilon^T (x^T \Sigma^{-1} X^T)^T] \\
    &= \sigma_\epsilon^2 x^T \Sigma^{-1}  \Sigma \Sigma^{-1} x\\
    &= \sigma_\epsilon^2 x^T \Sigma^{-1}  \Sigma \Sigma^{-1} x\\
    &= \sigma_\epsilon^2 x^T \Sigma^{-1} x \\
    &= \sigma_\epsilon^2 \Tr(x^T \Sigma^{-1} x) \\
    &= \sigma_\epsilon^2 \Tr(x x^T \Sigma^{-1}) 
\end{align*}

Taking the expected value over the empirical distribution, $\hat{p}$, of the sample, we find an explicit increasing dependence on $N$:
\begin{align*}
    \E_{x \sim \hat{p}} [\Var_\epsilon(h(x))] &= \E_{x \sim \hat{p}} [\sigma_\epsilon^2 \Tr(x x^T \Sigma^{-1})] \\
    &= \sigma_\epsilon^2 \Tr(\E_{x \sim \hat{p}} [x x^T] \Sigma^{-1}) \\
    &= \sigma_\epsilon^2 \Tr \left(\frac{1}{m} \Sigma \Sigma^{-1} \right) \\
    &= \sigma_\epsilon^2 \frac{1}{m} \Tr(I_N) \\
    &= \sigma_\epsilon^2 \frac{N}{m}
\end{align*}

\end{proof}

\section{Proof of Result for Variance of Over-parameterized Linear Models}
\label{app:linear_overparam}

Here, we produce a variation on what was done in \cref{app:linear_underparam} to show that variance does not grow with the number of parameters in over-parameterized linear models. Recall that we are considering the setting where $N > m$, where $N$ is the number of parameters and $m$ is the number of training examples.

\begin{proof}
By the law of total variance, 
$$ 
\Var(h(x)) = \E_\epsilon\Var_{\theta_0}(h(x)) + \Var_{\epsilon}(\E_{\theta_0}[h(x)])
$$
Here have $h(x) = x^T \hat{\theta}$, where $\hat{\theta}$ the gradient descent solution   $\hat{\theta} = P_\perp(\theta_0) + \Sigma^+ X^T Y$, and $\theta_0 \sim \mathcal{N}(0, \frac{1}{N} I)$. Then, 
\begin{align*}
\Var_{\theta_0}(h(x)) &= \E_{\theta_0}[(h(x) - \E_{\theta_0}[h(x)])^2] \\
    &= \E_{\theta_0}[x^T(P_\perp(\theta_0) - \E_{\theta_0}[P_\perp(\theta_0)])^2]\\
    &=\Var_{\theta_0}(x^T P_\perp(\theta_0))\\
    &=\Var_{\theta_0}(P_\perp(x)^T P_\perp(\theta_0))\\
    &= \frac{1}{N} \|P_\perp(x)\|^2
\end{align*}
Since $\E_{\theta_0}(h(x)) = x^T \Sigma^+ X^T Y$, the calculation of $\Var_\epsilon(\E_{\theta_0}) h(x))$ is similar as in  \ref{app:linear_underparam}, where $\Sigma^{-1}$ is replaced by $\Sigma^+$.
Thus, 
$$
    \Var_\epsilon(\E_{\theta_0} h(x))   
    = \sigma_\epsilon^2 \Tr(x x^T \Sigma^{+}) 
$$

Taking the expected value over the empirical distribution, $\hat{p}$, of the sample, we find an explicit dependence on $r = \rank(X)$, not $N$:
\begin{align*}
    \E_{x \sim \hat{p}} [\Var(h(x))] &= 0 +  \E_{x \sim \hat{p}} [\sigma_\epsilon^2 \Tr(x x^T \Sigma^{+})] \\
    &= \sigma_\epsilon^2 \Tr(\E_{x \sim \hat{p}} [x x^T] \Sigma^{+}) \\
    &= \sigma_\epsilon^2 \Tr \left(\frac{1}{m} \Sigma \Sigma^{+} \right) \\
    &= \sigma_\epsilon^2 \frac{1}{m} \Tr(I_r^+) \\
    &= \sigma_\epsilon^2 \frac{r}{m}
\end{align*}
where $I_r^+$ denotes the diagonal matrix with 1 for the first $r$ diagonal elements and $0$ for the remaining $N - r$ elements.
\end{proof}

\section{Proof of \cref{thm:init-var-decay}}
\label{app:more_general_setting}

First we state some known concentration results \citep{ledoux2001concentration} that we will use in the proof.

\begin{lemma}[Levy] \label{Levy} Let $h: S^{n}_R \to \reals$ be a function on the $n$-dimensional Euclidean sphere of radius $R$, with Lipschitz constant $L$; and $\theta \in S^n_R$ chosen uniformly at random for the normalized measure.  Then 
\beq
\Prob(|h(\theta)-\E[h]| >\epsilon) \leq 2 \exp\left( - C \frac{n \epsilon^2}{ L^2 R^2 }\right)
\eeq
for some universal constant $C >0$.
\end{lemma} 

Uniform measures on high dimensional spheres approximate Gaussian distributions \citep{ledoux2001concentration}. Using this, Levy's lemma yields an analogous concentration inequality for functions of Gaussian variables:

\begin{lemma}[Gaussian concentration] \label{Levy-Gauss} Let $h: \R^n \to \reals$ be a function on the Euclidean space $\R^n$, with Lipschitz constant $L$;  and $\theta \sim \mathcal{N}(0, \sigma \mathbb{I}_n)$  sampled from an isotropic $n$-dimensional Gaussian. Then: 
\beq
\Prob(|h(\theta)-\E[h]| >\epsilon) \leq 2 \exp\left( - C \frac{\epsilon^2}{ L^2 \sigma^2 }\right)
\eeq
for some universal constant $C >0$.
\end{lemma}
Note that in the Gaussian case, the bound is dimension free. 

In turn, concentration inequalities  give variance bounds for functions of random variables.

\begin{cor} \label{Levy-variance}
 Let $h$ be a function satisfying the conditions of Theorem \ref{Levy-Gauss},   and $\mbox{Var}(h) = \E[(h - \E[h])^2]$. Then 
 \beq 
  \mbox{Var}(h) \leq \frac{2 L^2 \sigma^2 }{C}
 \eeq
\end{cor}

\begin{proof} Let $g = h - \E[h]$. Then  $\mbox{Var}(h) = \mbox{Var}(g)$ and 
\beq
\mbox{Var}(g) = \E[|g|^2] 
= 2 \E\int_0^{|g|} t dt 
 =2 \E\int_0^\infty t \mathbbm{1}_{|g|>t} \, dt
\eeq
Now swapping expectation and integral (by Fubini theorem), and by using  the identity $\E \mathbbm{1}_{|g|>t} = \Prob(|g| > t)$, we obtain
\begin{align*}
\mbox{Var}(g) & = 2 \int_0^\infty t \, \Prob_R(|g| > t) \, d t \\
& \leq 2 \int_0^\infty 2 t  \exp\left( - C \frac{t^2}{ L^2\sigma^2}\right) d t \\
 & = 2 \left[-\frac{L^2 \sigma^2 }{C} \exp\left( - C \frac{t^2}{L^2\sigma^2}\right)\right]_0^\infty  = \frac{2 L^2\sigma^2}{C}
 \end{align*} 
\end{proof}

We are now ready to prove Theorem $\ref{thm:init-var-decay}$.
We first recall our assumptions:

\begin{assumption}
\label{assum:invariant-space}
The optimization of the loss function is invariant with respect to $\theta_{\mathcal{M}\perp}$.
\end{assumption}

\begin{assumption}
\label{assum:deterministic-solution}  
Along $\mathcal{M}$, optimization yields solutions independently of the initialization $\theta_0$.
\end{assumption}

We add the following assumptions.

\begin{assumption}
\label{assum:lipschitz}
The prediction $h_{\theta}(x)$ is 
$L$-Lipschitz with respect to $\theta_{\mathcal{M}\perp}$.
\end{assumption}

\begin{assumption}
\label{assum:init}
The network parameters are initialized as 
\beq
    \theta_0 \sim 
    \mathcal{N}(0, \frac{1}{N}\cdot I_{N\times N}).
\eeq
\end{assumption}

We first prove that the Gaussian concentration theorem  translates into concentration of predictions in the setting of \cref{sec:variance-from-optimization}.
\begin{theorem}[Concentration of predictions]
\label{thm:concentration-predictions}
Consider the setting of \cref{sec:back-to-nn} and Assumptions \ref{assum:invariant-space} and \ref{assum:init}.  
Let $\theta$ denote the parameters at the end of the learning process.
Then, for a fixed data set, $S$ we get concentration of the prediction, under initialization randomness, 
\begin{equation}
    \Prob(|h_{\theta}(x)-\E[h_{\theta}(x)]| >\epsilon) \leq 2 \exp\left( - C \frac{N \epsilon^2}{ L^2}\right)
\end{equation}
for some universal constant $C >0$.
\end{theorem}

\begin{proof}
In our setting,  
the parameters at the end of learning can be expressed as
\begin{equation}
    \theta = \theta_\mathcal{M}
^* + \theta_{\mathcal{M}^\perp}
\end{equation}
where $\theta_\mathcal{M}^*$ is independent of the initialization $\theta_0$.   To simplify notation, we will assume that, at least locally around $\theta_\mathcal{M}^*$, $\mathcal{M}$ is spanned by the first $d(N)$ standard basis vectors, and $\mathcal{M}^\perp$ by the remaining $N-d(N)$.
This will allow us, from now on, to use the same variable names for $\theta_\mathcal{M}$ and $\thmp$ to denote their lower-dimensional representations of dimension $d(N)$ and $N-d(N)$ respectively.
More generally, we can assume that there is a mapping from $\theta_\mathcal{M}$ and $\thmp$ to those lower-dimensional representations. 

From Assumptions~\ref{assum:invariant-space} and \ref{assum:init} we get
\begin{equation}
    \theta_{\mathcal{M}^\perp}
    \sim \mathcal{N}\left(0, \frac{1}{N} I_{(N-d(N))\times (N-d(N))}\right).
\end{equation}

Let $g(\thmp) 
\triangleq h_{\theta_\mathcal{M}
^* + \thmp}(x)$.
By Assumption~\ref{assum:lipschitz}, 
$g(\cdot)$ is $L$-Lipschitz.
Then, by the Gaussian concentration theorem we get,
\begin{equation}
    \Prob(|g(\thmp)-\E[g(\thmp)]| >\epsilon) \leq 2 \exp\left( - C \frac{N \epsilon^2}{ L^2}\right).
\end{equation}
\end{proof}
The result of Theorem~\ref{thm:init-var-decay} immediately follows from Theorem~\ref{thm:concentration-predictions} and Corollary~\ref{Levy-variance}, with $\sigma^2 = 1/N$:
\beq 
\Var_{\theta_0}(h_\theta(x)) \leq C \frac{2L^2}{N}
\eeq
Provided the Lipschitz constant $L$ of the prediction  grows more slowly than the square of dimension, $L=o(\sqrt{N})$, we conclude that the variance vanishes to zero as $N$ grows.

\section{Bound on classification error in terms of regression error}
\label{app:classification_regression_relation}

\newcommand{\Rcl}{\cR_{\mbox{\tiny classif}}}
\newcommand{\Rreg}{\cR_{\mbox{\tiny reg}}}
In this section we give a bound on  classification risk $\Rcl$ in terms of the regression risk $\Rreg$. 

{\bf Notation.} Our classifier defines a  map $h: \mathcal{X} \to  \mathbb{R}^k$, which outputs probability vectors $h(x) \in \mathbb{R}^k$, with $\sum_{y=1}^k h(x)_y = 1$. The classification loss is defined by 
\begin{align} 
L(h) &= \mbox{Prob}_{x,y} \{h(x)_y < \max_{y'} h(x)_{y'}\} \nonumber \\
&= \mathbb{E}_{(x,y)} I(h(x)_y < \max_{y'} h(x)_{y'})
\end{align}
where $I(a) = 1$ if predicate $a$ is true and 0 otherwise.  Given trained predictors $h_S$ indexed by  training dataset $S$, the classification and regression risks are given by,
\beq 
\Rcl = \mathbb{E}_S L(h_S), \qquad \Rreg = \mathbb{E}_S \mathbb{E}_{(x,y)} ||h_S(x) - Y||^2_2
\eeq
where $Y$ denotes the one-hot vector representation of the class $y$. 

\begin{prop}
The classification risk is bounded by four times the regression risk,  
$\Rcl \leq 4 \Rreg$.  
\end{prop}
\begin{proof}  First note that, if $h(x) \in \R^k$ is a probability vector, then 
\[
h(x)_y < \max_{y'} h(x)_{y'}\,  \Longrightarrow \, h(x)_y < \frac12
\]
By taking the expectation over $x, y$,  we obtain the inequality $L(h) \leq {\widetilde L}(h)$ where
\beq {\widetilde L} (h) = \mbox{Prob}_{x,y} \{ h(x)_y < \frac12\}\eeq  

We then have,
\begin{align*} \Rcl :=  \mathbb{E}_S L(h_S) 
&\leq \mathbb{E}_S \tilde{L}(h_S) \\
& = \mbox{Prob}_{S; \, x,y} \{ {h_S(x)}_y < \frac12 \} \\
&= \mbox{Prob}_{S; \, x,y} \{|h_S(x)_y - Y_y| >\frac12\} \\
&\leq \mbox{Prob}_{S; \, x,y} \{||h_S(x) - Y||_2 > \frac12 \}  \\
& = \mbox{Prob}_{S; \, x,y} \{||h_S(x) - Y||^2_2 > \frac14 \}
 \leq 4 \Rreg
\end{align*}

where the last inequality follows from  Markov's inequality.  

\end{proof}
\chapter{Common intuitions from impactful works} \label{app:intuitions}

``Neural Networks and the Bias/Variance Dilemma'' from \citep{geman}: ``How big a network should we employ? A small network, with say one hidden unit, is likely to be biased, since the repertoire of available functions spanned by $f(x; w)$ over allowable weights will in this case be quite limited. If the true regression is poorly approximated within this class, there will necessarily be a substantial bias. On the other hand, if we overparameterize, via a large number of hidden units and associated weights, then the bias will be reduced (indeed, with enough weights and hidden units, the network will interpolate the data), but there is then the danger of a significant variance contribution to the mean-squared error. (This may actually be mitigated by incomplete convergence of the minimization algorithm, as we shall see in Section 3.5.5.)''

``An Overview of Statistical Learning Theory'' from \citep{Vapnik:1999}: ``To avoid over fitting (to get a small confidence interval) one has to construct networks with small VC-dimension.''

``Stability and Generalization'' from \citet{Bousquet2002}: ``It has long been known that when trying to estimate an unknown function from data, one needs to find a tradeoff between bias and variance. Indeed, on one hand, it is natural to use the largest model in order to be able to approximate any function, while on the other hand, if the model is too large, then the estimation of the best function in the model will be harder given a restricted amount of data." Footnote: ``We deliberately do not provide a precise definition of bias and variance and resort to common intuition about these notions."

Pattern Recognition and Machine Learning from \citet{Bishop:2006}: ``Our goal is to minimize the expected loss, which we have decomposed into the
sum of a (squared) bias, a variance, and a constant noise term. As we shall see, there
is a trade-off between bias and variance, with very flexible models having low bias
and high variance, and relatively rigid models having high bias and low variance.''

``Understanding the Bias-Variance Tradeoff'' from \citet{fortmann-roe_2012}: ``At its root, dealing with bias and variance is really about dealing with over- and under-fitting. Bias is reduced and variance is increased in relation to model complexity. As more and more parameters are added to a model, the complexity of the model rises and variance becomes our primary concern while bias steadily falls. For example, as more polynomial terms are added to a linear regression, the greater the resulting model's complexity will be.''
\begin{figure}[h]
 \centering
 \includegraphics[width=.7\textwidth]{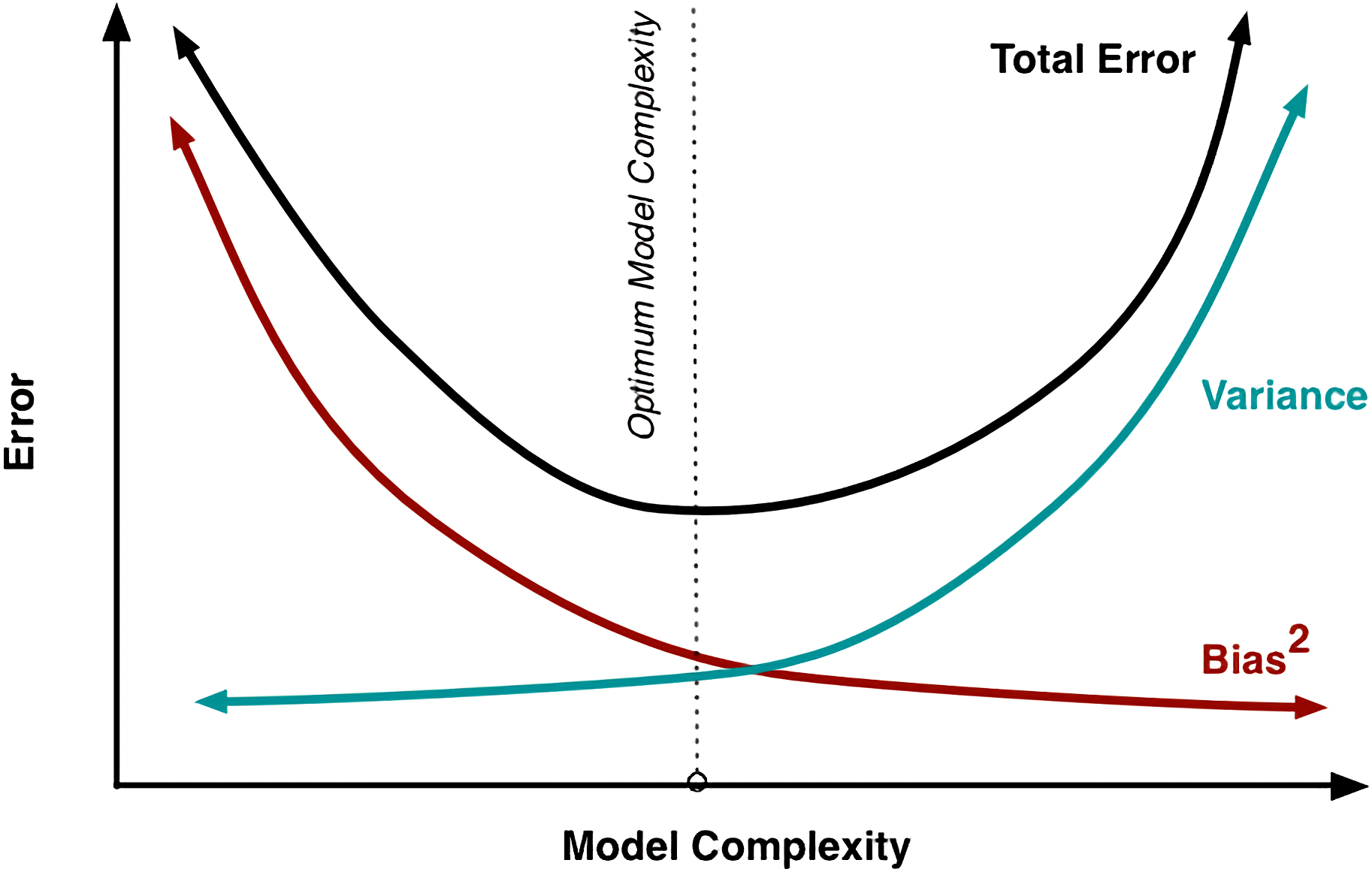}
 \caption{Illustration of common intuition for bias-variance tradeoff \citep{fortmann-roe_2012}}
\end{figure}

\end{document}